\documentclass{article}

\usepackage{microtype}
\usepackage{graphicx}
\usepackage{subcaption}
\usepackage{booktabs} 

\usepackage[utf8]{inputenc} 
\usepackage[T1]{fontenc}    
\usepackage{hyperref}       
\usepackage{url}            
\usepackage{booktabs}       
\usepackage{amsfonts}       
\usepackage{nicefrac}       
\usepackage{microtype}      
\usepackage{xcolor}         
\usepackage{mathtools}
\usepackage{amssymb,amsthm,amsmath}
\usepackage{enumitem}
\usepackage{bbm}

\usepackage[capitalize]{cleveref}  
\Crefname{figure}{Fig.}{Figs.}
\Crefname{tabular}{Tab.}{Tabs.}
\Crefname{section}{Sec.}{Secs.}

\theoremstyle{definition}
\newtheorem{theorem}{Theorem} 
\newtheorem{proposition}{Proposition} 
\newtheorem{lemma}{Lemma} 

\theoremstyle{remark}
\newtheorem{remark}[theorem]{Remark} 

\newcommand{\squishlisttwo}{
 \begin{list}{$\bullet$}
  { \setlength{\itemsep}{1pt}
     \setlength{\parsep}{0pt}
    \setlength{\topsep}{0pt}
    \setlength{\partopsep}{0pt}
    \setlength{\leftmargin}{1em}
    \setlength{\labelwidth}{1.5em}
    \setlength{\labelsep}{0.5em} } }

\newcommand{\squishend}{
  \end{list}  }

\DeclareMathOperator*{\argmin}{argmin}

\usepackage[accepted]{icml2023}

\icmltitlerunning{Fair yet Asymptotically Equal Collaborative Learning}

\begin{document}

\twocolumn[
\icmltitle{Fair yet Asymptotically Equal Collaborative Learning
}



\icmlsetsymbol{equal}{*}

\begin{icmlauthorlist}
\icmlauthor{Xiaoqiang Lin}{nus,equal}
\icmlauthor{Xinyi Xu}{nus,astar,equal}
\icmlauthor{See-Kiong Ng}{ids}
\icmlauthor{Chuan-Sheng Foo}{astar}
\icmlauthor{Bryan Kian Hsiang Low}{nus}
\end{icmlauthorlist}

\icmlaffiliation{nus}{Department of Computer Science, National University of Singapore, Singapore.}
\icmlaffiliation{astar}{Institute for Infocomm Research, A*STAR, Singapore.}
\icmlaffiliation{ids}{Institute of Data Science, National University of Singapore, Singapore}

\icmlcorrespondingauthor{Xinyi Xu}{xinyi.xu@u.nus.edu}

\icmlkeywords{Machine Learning, ICML}

\vskip 0.3in
]



\printAffiliationsAndNotice{\icmlEqualContribution} 

\begin{abstract}
In collaborative learning with streaming data, \textit{nodes} (e.g., organizations) jointly and \textit{continuously} learn a machine learning (ML) model by sharing the \textit{latest model updates} computed from their latest streaming data. For the more resourceful nodes to be willing to share their model updates, they need to be \textit{fairly} incentivized. This paper explores an incentive design that guarantees fairness so that nodes receive rewards commensurate to their contributions. Our approach leverages an explore-then-exploit formulation to estimate the nodes' contributions (i.e., exploration) for realizing our theoretically guaranteed fair incentives (i.e., exploitation).
However, we observe a ``rich get richer'' phenomenon arising from the existing approaches to guarantee fairness and it discourages the participation of the less resourceful nodes. To remedy this, we \textit{additionally} preserve asymptotic \emph{equality}, i.e., less resourceful nodes achieve equal performance eventually to the more resourceful/``rich'' nodes.
We empirically demonstrate in two settings with real-world streaming data: federated online incremental learning and federated reinforcement learning, that our proposed approach outperforms existing baselines in fairness and learning performance while remaining competitive in preserving equality. 
\end{abstract}

\section{Introduction} \label{sec:introduction}
The problem of collaborative learning with streaming data involves having multiple nodes collecting data incrementally \citep{chen2020asynchronous,Le2021_continuous_stream_fl} to jointly and \emph{continuously} learn an ML model by sharing the \emph{latest model updates} computed from their latest streaming data \cite{Jin2021-streaming-data,onlineml2018}, for a possibly long-term collaboration \citep{Xu2021_FL_repeated_game_long_term,Yuan2021_incentive_FL_long_term,Zhang2022_FL_long_term}.\footnote{Refer to \cref{sec:foil} for a detailed example.}
The setting of streaming data is motivated by scenarios where data collection is time-consuming and takes place over a long term, e.g., hospitals collect $100$ thousand medical scans collaboratively over months \citep{Flores2020_nvidia_FL_medical, Miller1996_finance_monthly, Leersum2013_medicine_weekly}. Moreover, due to data scarcity in the beginning of data collection, collaboration through data sharing can benefit the nodes by providing them with an ML model trained on the shared data, e.g., repeated interactions among $17$ partners/nodes \citep{Hale2019_mellody_FL}. As the ML model is continuously trained, it also finds application in real-time decision-making, e.g., provide real-time customized services \citep{Zhao2020timeliness}, or make predictions in the stock market \citep{onlineml2018,Bifet2009DATASM, hft2018}. In contrast, it is undesirable or infeasible to wait till the completion of the data collection and model training, which can take over weeks, months, or an indeterminate length \citep{Leersum2013_medicine_weekly,Miller1996_finance_monthly, Wang2021_FL_field_guide}.

To ensure the effectiveness of such collaboration of data sharing through continuous learning,\footnote{We highlight it is different from continual learning \citep{yoon21-federated-continual-learning} where a model is incrementally trained w.r.t.~\emph{different} tasks, i.e., the distribution of data changes/shifts over time.} it is important to incentivize the nodes to share their latest data (indirectly) via the \emph{latest} model updates. In particular, a \emph{fair} incentive mechanism is shown to be effective, by giving higher incentives to nodes with higher contributions~\cite{profit-allocation-FL-Song2019,Sim2020,Seb2022-incentivizing}. 
One specific approach is via \emph{ex post} fair incentives \cite{Chen2020-FL-incentive,Richardson2020-FL-incentive} based on the nodes' \textit{true contributions} from all of their shared/uploaded model updates over the entire training \cite{profit-allocation-FL-Song2019,Wang2020-SV-in-FL}. This faces two practical obstacles: 1.~it requires an additional external resource (e.g., money) for such incentives, but it is unclear who should provide this resource or what the denomination is, e.g., the exact monetary value of model updates \cite{Sim2020}; 2.~the incentives are realized only after training completes which can take a long and indeterminate time. It means the nodes do \emph{not} know up-front when or what they will receive as incentives, which makes it difficult to convince them to join the collaboration.

In contrast, the approach of training-time incentive design can avoid these obstacles by realizing some carefully designed incentives \emph{during} training \citep{Xu2021-fair-CML,agussurja2022_bayesian_para}.
However, this implies that at the time of incentive realization during training, the true contributions (i.e., defined over the entire training) of the nodes are not completely known.
This presents a challenge because the incentives are ideally designed w.r.t.~the true contributions to guarantee fairness so that the nodes with higher true contributions receive better incentives.
(1) \textbf{How to obtain accurate estimates of the true contributions so that the incentives which are fair w.r.t.~the estimates, are also fair w.r.t.~the true contributions?}

Regarding incentive realization, an existing approach of using model updates \cite{Xu2021-fair-CML} can lead to non-convergence behavior of the model (\cref{sec:experiments}) and has a ``localized'' fairness guarantee w.r.t. a particular iteration $t$ instead of overall.
Another existing approach uses the estimates of some latent variables in an asymptotic approach to ensure fairness \citep{agussurja2022_bayesian_para}, but its theoretical result is limited to only $2$ nodes.
(2) \textbf{What is a suitable realization of incentives during training, with a fairness guarantee w.r.t.~the overall training and applies to more than $2$ nodes?}

Lastly, an undesirable ``rich get richer'' outcome can arise from the fairness guarantee and worse inequality.
Specifically, as nodes with higher contributions (e.g., the more resourceful organizations/companies with better local data) receive better incentives (e.g., better models), it can widen the ``gap'' between the more resourceful and less resourceful nodes.
This can discourage the participation of the less resourceful nodes, which is undesirable as their participation can improve the overall utility (i.e., model performance) as shown in \cref{sec:experiments}.
While \cite{li2019-qFFL} aims to address equality, unfortunately, it cannot also guarantee fairness.
(3) \textbf{How to address the dilemma between fairness to incentivize the more resourceful nodes to join and equality to encourage the less resourceful nodes to join?}

We propose a collaborative incremental learning framework to answer these questions under the streaming data setting:
For (1), as the nodes have stationary data distributions, their true contributions can be estimated cumulatively, and importantly, with more accuracy over more iterations.
We identify the classic \emph{explore-exploit} paradigm: improving the accuracy of contribution estimates (exploration) vs.~using the current estimates to design incentives (exploitation), and adopt an explore-then-exploit formulation by developing a stopping criterion for exploration.
For (2), we utilize the \emph{latest} model during training as a valuable resource for realizing incentives: the nodes with higher contributions synchronize more frequently with the latest model in expectation. We show this design guarantees fairness in terms of the asymptotic convergence behaviors of the models of the nodes instead of localized to certain iteration $t$'s.
For (3), we introduce an equality-preserving perspective so that \textit{all} nodes can receive the equally optimal model asymptotically via a single equalizing coefficient to trade-off (empirical) fairness with asymptotic equality.

Our specific contributions are summarized as follows:
\squishlisttwo
    \item Leveraging the Hotelling's one-sample test to design an adjustable stopping criterion for contribution evaluation as the exploration in an \textit{explore-then-exploit} formulation;

    \item Proposing a novel node sampling distribution to \textit{guarantee fairness} (as in the Shapley value) so that the nodes with higher contributions receive the latest model more frequently, and thus they observe better convergence;
    
    \item Introducing an equalizing coefficient so that all nodes receive the optimal model with \textit{equal asymptotic convergence} while guaranteeing fairness; and
    
    \item Empirically demonstrating on federated online incremental learning and federated reinforcement learning that our proposed approach outperforms existing baselines in terms of fairness and learning performance while remaining competitive in preserving equality.
\squishend

\section{Preliminaries and Setting}
\label{sec:setting}
\paragraph{Federated learning (FL) setting and notations.}
A set $[N]\coloneqq \{i\}_{i=1,\ldots,N}$ of $N$ \textit{honest} compute nodes\footnote{Honest nodes do not deviate from the proposed algorithm and we provide a result from relaxing this assumption in \cref{sec:appendix-proofs}.} collaborate by communicating via a \textit{trusted} coordinating node, \textit{coordinator} to jointly learn an ML model parameterized by $\theta \in \Theta$ w.r.t.~a minimizing objective $\mathbf{J}(\theta) \coloneqq \sum_i p_i\ \mathbf{L}(\theta;\mathcal{D}_i)$ where $\mathcal{D}_i$ is the local data of node $i$, $\mathbf{L}(\theta;\mathcal{D}_i)$ is a loss function of $\theta$ on $\mathcal{D}_i$~(e.g., cross-entropy loss for classification), and $p_i\coloneqq  |\mathcal{D}_i| / \sum_{i'} |\mathcal{D}_{i'}|$ is the data size weighted coefficient~\cite{BrendanMcMahan2017-FedAvg}. Let $\theta_{i,t}$~($\theta_t$) denote the model at node $i$~(coordinator) in iteration $t$. 
At $t=0$, all nodes receive the same initialized model $\theta_{i,0}\coloneqq\theta_0$ from the coordinator. To begin iteration $t>0$, the coordinator selects a subset of size $k\leq N$ nodes which first synchronize with the latest model $\theta_{i, t-1} \gets \theta_{t-1}$ and then compute the derivative of the loss 
$\Delta \theta_{i,t-1} \coloneqq \nabla \mathbf{L}(\theta_{i, t-1};s_{i,t})$ where $s_{i,t}$ is a randomly selected subset/batch of $\mathcal{D}_i$ (or from distribution of $\mathcal{D}_i$ for streaming data). For technical reasons (e.g., proving \cref{proposition:fairness}), we consider the data distribution of each node to be \emph{stationary} over time and defer the non-stationary setting to future work. These selected nodes then upload $\Delta \theta_{i,t-1}$ to the coordinator to be aggregated as $\Delta \theta_{t}\coloneqq \sum_{\text{selected } i} p_i\  \Delta \theta_{i,t-1}$ and updates the latest model as $\theta_t \gets \theta_{t-1} - \Delta \theta_{t}\ .$ All the notations are tabulated in \cref{appendix:algorithm}.

\paragraph{Shapley value for contribution and incentive.}
Node $i$'s contribution till iteration $t$ is defined as 
\begin{equation}
\label{equ:cumulative-sv-psi}
\boldsymbol{\psi}_t \coloneqq \{\psi_{i,t}\}_{i\in [N]}\ , \quad 
\psi_{i,t} \coloneqq t^{-1} \textstyle \sum_{l=0}^{t} \phi_{i,l}
\end{equation}
where $ \textstyle \phi_{i,t} \coloneqq N^{-1} \sum_{\mathcal{S} \subseteq{[N]}\setminus \{i\}} \binom{N-1}{|\mathcal{S}|}^{-1} \mathbf{U}( \mathcal{S} \cup \{i\}) - \mathbf{U}( \mathcal{S})$ is the Shapley value (SV) \citep{shapley1953value} in iteration $t$. As $\phi_{i,t}$ has an exponential time complexity in $N$, our implementation adopts a linear approximation \citep{Fatima2008-linear-SV-approximation} (details in \cref{sec:appendix-proofs}).
The utility function $\mathbf{U}:2^N \mapsto \mathbb{R}$ is defined as the inner product $\langle \Delta \theta_{\mathcal{S},t},\Delta \theta_{[N],t} \rangle$
between the aggregated model update from a subset/coalition of nodes, $\Delta \theta_{\mathcal{S},t} = \sum_{i \in \mathcal{S}} p_i\ \Delta \theta_{i,t}$ and that from all the nodes,  $\Delta \theta_{[N],t} = \sum_{i \in [N]} p_i\ \Delta \theta_{i,t}$.
Intuitively, node $i$'s contribution in iteration $t$ is determined by how closely $\Delta \theta_{i,t}$ aligns with other model updates $\Delta \theta_{i',t}$ \citep{Xu2021-fair-CML}.
The implementation details are in \cref{sec:appendix-proofs}.
We denote $i$'s true contribution with $\psi_{i}^* \coloneqq \lim_{t\to \infty} \psi_{i,t}$,\footnote{As $\psi_{i,t}$ is empirically observed to converge as $\theta_t$ converges, the limit is assumed to exist.} 
and define $\boldsymbol{\psi}^*\coloneqq \{\psi_{i}^*\}_{i\in[N]}$. 
Subsequently, we refer to $\psi_{i,t}$ as the \textit{contribution estimate} and the accuracy is w.r.t.~$\psi^*_{i}$.
Incentives (e.g., monetary \citep{profit-allocation-FL-Song2019}, model updates \citep{Xu2021-fair-CML}) designed in proportion to $\boldsymbol{\psi}^*$ or $\boldsymbol{\psi}$ to guarantee fairness are often called Shapley-fair \citep{Sim2020,agussurja2022_bayesian_para,zhou2023}.

\paragraph{Settings for the running example in \cref{sec:proposed-framework}.}
We use the MNIST~\cite{mnist} dataset with $N=30$ nodes, each with uniformly randomly sampled $600$ images and a standard \textit{convolutional neural network} with two convolution and two fully-connected layers. We reassign $20\%$ of the labels (to a randomly selected incorrect one) for a designated subset of $30\%$ of the nodes to have lower $\psi_i^*$ \cite{profit-allocation-FL-Song2019} to simulate some nodes have noisy observation of data due to different level of resourcefulness (e.g., hospitals with different levels of budgets for data collection equipment with different precision \cite{medicalerror2004}).\footnote{Refer to \cref{appendix:setings-experiments} for more explanation on resourcefulness.}
The accuracy of $\psi_{i,t}$ is evaluated via the \textit{recall fraction} $\coloneqq \frac{\text{$\#$ nodes with designated low-quality data and lowest $30\%$}\  \boldsymbol{\psi}_t}{ \text{$\#$ nodes with designed low-quality data}}$ \cite{Wang2020-SV-in-FL}.
We provide additional results under three more types of low-quality/less valuable data in \cref{appendix:experiments}.

\section{Explore-Exploit in Contribution Evaluation and Incentive Realization}\label{sec:proposed-framework}
Our proposed framework first ``explores'' sufficiently by evaluating the contributions of the nodes (\cref{sec:motivation}); and then ``exploits'' the converged and approximately accurate contribution estimates to realize the fair incentives (\cref{sec:stage2}), by sampling the nodes to receive the latest model according to their contribution estimates (i.e., a higher contribution estimate leads to a higher sampling probability).

\subsection{The Explore-Exploit Perspective and Contribution Evaluation} \label{sec:motivation} 
We view the expected number of iterations for $\theta_t$ to converge to optimal, $T^*$ (an indeterminate and possibly large number) as a fixed budget,\footnote{\cite{Li2019-fedavg-noniid} provides a big $\mathcal{O}$ notation for $T^*$ under stationary data setting, used in \cref{sec:fair-complexities}.} and describe a stopping criterion for contribution evaluation.
From the definition of $\psi_i^*$, increasing exploration (i.e., extending contribution evaluation over more iterations) generally improves the accuracy in $\psi_{i,t}$, as shown in \cref{fig:motivation-accurate-sv}. 
Then, as it takes $T^*$ iterations for $\theta_t$ to converge to optimal, effectively we have a fixed budget of $T^*$ iterations to allocate between contribution evaluation (exploration) and incentive realization (exploitation). \cref{fig:spectrum} shows two extreme cases when choosing the stopping iteration for contribution evaluation. Specifically, allocating all the iterations to contribution evaluation reduces the entire framework to without incentive realization: All the nodes receive the same model throughout \citep{profit-allocation-FL-Song2019,Wang2020-SV-in-FL}, which is unfair \citep{Xu2021-fair-CML}.
In contrast, allocating all iterations to incentive realization while performing contribution evaluation concurrently \cite{Nagalapatti2021-game-of-gradients-fl} is shown to have poor empirical performance in guaranteeing fairness (\cref{fig:error-propagation} in \cref{appendix:experiments}) because the contribution estimates are inaccurate due to insufficient exploration.

\begin{figure}[!ht]
    \centering
    \includegraphics[width=0.9\linewidth]{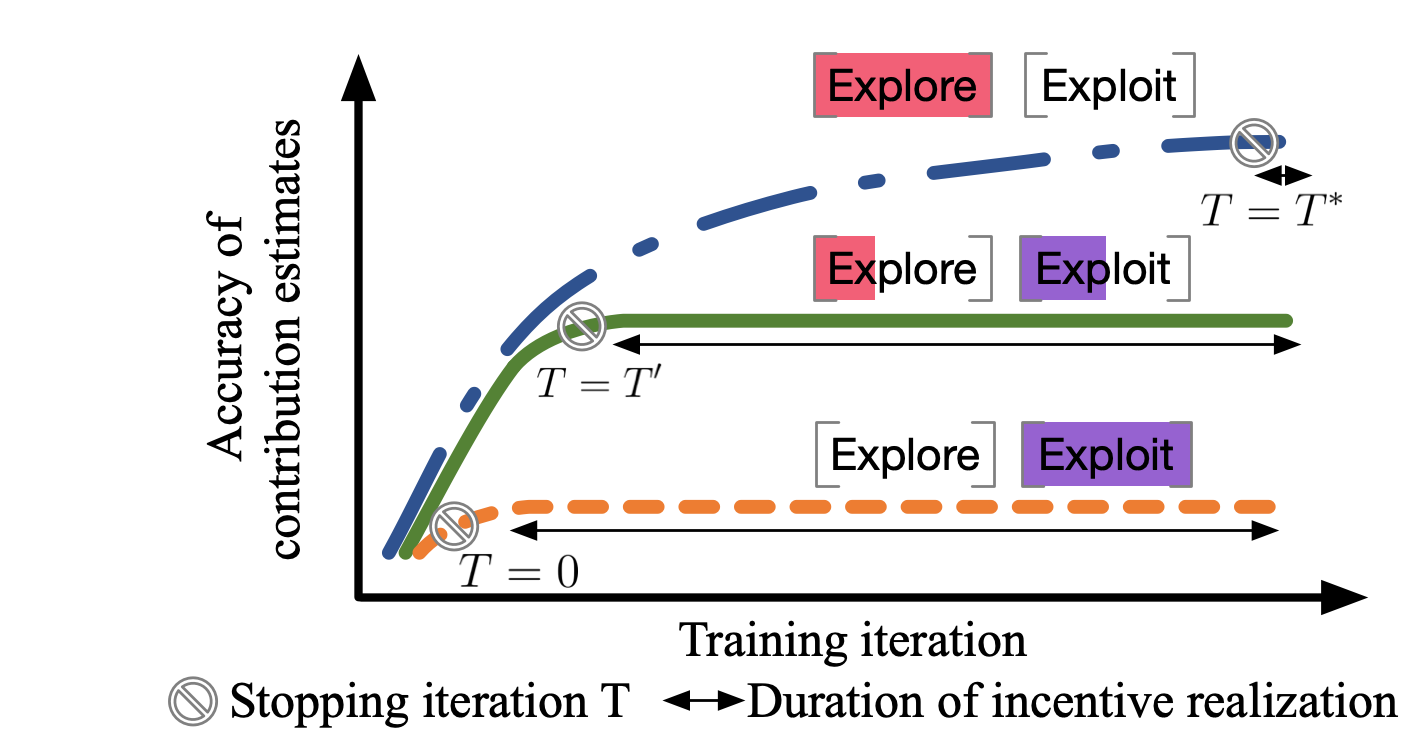
    }
\caption{Illustration of the explore-exploit perspective. The vertical axis is denotes the accuracy of the current contribution estimates (e.g., recall fraction in \cref{sec:setting}). If the stopping iteration $T=0$ (i.e., no-explore-all-exploit), fairness is not guaranteed; if $T=T^*$ (i.e., no-exploit-all-explore), then no iteration for incentives. A carefully set $T'$ avoids these two problematic situations.}
    \label{fig:spectrum}
\end{figure}

\textbf{Hypothesis testing-based stopping criterion.}
Though the recall fraction can reflect the accuracy of $\boldsymbol{\psi}_t$ (\cref{fig:motivation-accurate-sv}), it cannot be used in practice, e.g., because the true noise in the node's data is unknown. Fortunately, the convergence of the observed past values of $\boldsymbol{\psi}_t$ (\cref{fig:motivation-accurate-sv} right) somewhat reflects the convergence of the recall fraction (\cref{fig:motivation-accurate-sv} left). This inspires using Hotelling's one-sample test \cite{hotelling} to monitor the convergence of $\boldsymbol{\psi}_t$ to gauge its accuracy.

Precisely, at current iteration $t_s$, we determine whether $\boldsymbol{\psi}_{t_s} \coloneqq t^{-1}_s \sum^{t_s}_{t=1} \boldsymbol{\phi}_{t}$ (where $\boldsymbol{\phi}_{t} \coloneqq \{\phi_{i,t}\}_{i=1,\ldots,N}$) has converged w.r.t.~$\boldsymbol{\psi}_{t_s-\tau}$ by testing whether $\{\boldsymbol{\phi}_{t}\}_{t =1,\ldots, t_s - \tau}$ and $\{\boldsymbol{\phi}_{t}\}_{t= 1,\ldots, t_s}$ have the same mean.
Informally, $\tau$ is the size of test window (i.e., number of past iterations).
Define $\hat{\boldsymbol{\mu}}_{0, t_s} \coloneqq (t_s-\tau)^{-1}\sum_{t=1}^{t_s-\tau} \boldsymbol{\phi}_t$ and
assume there exist $\boldsymbol{\mu}_{t_s} \in \mathbb{R}^N$ and $\Sigma_{t_s}\in \mathbb{R}^{N\times N}$ s.t.~$\{ \boldsymbol{\phi}_{t}\}_{t=1,\ldots,t_s}$ 
are $t_s$ i.i.d.~samples from the Gaussian $\mathcal{N}(\boldsymbol{\mu}_{t_s}, \Sigma_{t_s})$.\footnote{We theoretically discuss and empirically verify this assumption in \cref{sec:appendix-proofs}.}
As such, rejecting the null hypothesis $h_{0, t_s}: \boldsymbol{\mu}_{t_s}=\hat{\boldsymbol{\mu}}_{0, t_s}$ means there is statistical evidence that $\boldsymbol{\phi}_{t}$ is fluctuating between $t_s-\tau$ and $ t_{s}$, so $\boldsymbol{\psi}_{t_s}$ has not converged.

\begin{proposition}
\label{prop:hypothesis-testing}
For iteration $t_s$, define
$
\texttt{T2} \coloneqq t_s(\boldsymbol{\psi}_{t_s}- \hat{\boldsymbol{\mu}}_{0,t_s})^\top S^{-1}(\boldsymbol{\psi}_{t_s}-\hat{\boldsymbol{\mu}}_{0,t_s})
$
where $S$ is the estimated sample covariance matrix from 
 $\{\boldsymbol{\phi}_{t}\}_{t= 1,\ldots, t_s}$.
The following \textit{stopping criterion} guarantees at most $\alpha$ type-$1$ error:
$
p\text{-value} \coloneqq \text{Pr}(\texttt{T2} \geq T^2_{1-\alpha, \mathcal{N}, 2(t_s-1)} ) \geq \alpha.\footnote{$T^2_{1-\alpha, \mathcal{N}, 2\tau-2}$ denotes the $1-\alpha$ quantile for the Hotelling's $T$-squared distribution $T^2_{\mathcal{N}, 2\tau-2}$.}
$
\end{proposition}

\begin{figure}[t]
    \centering 
    \includegraphics[width=0.93\linewidth]{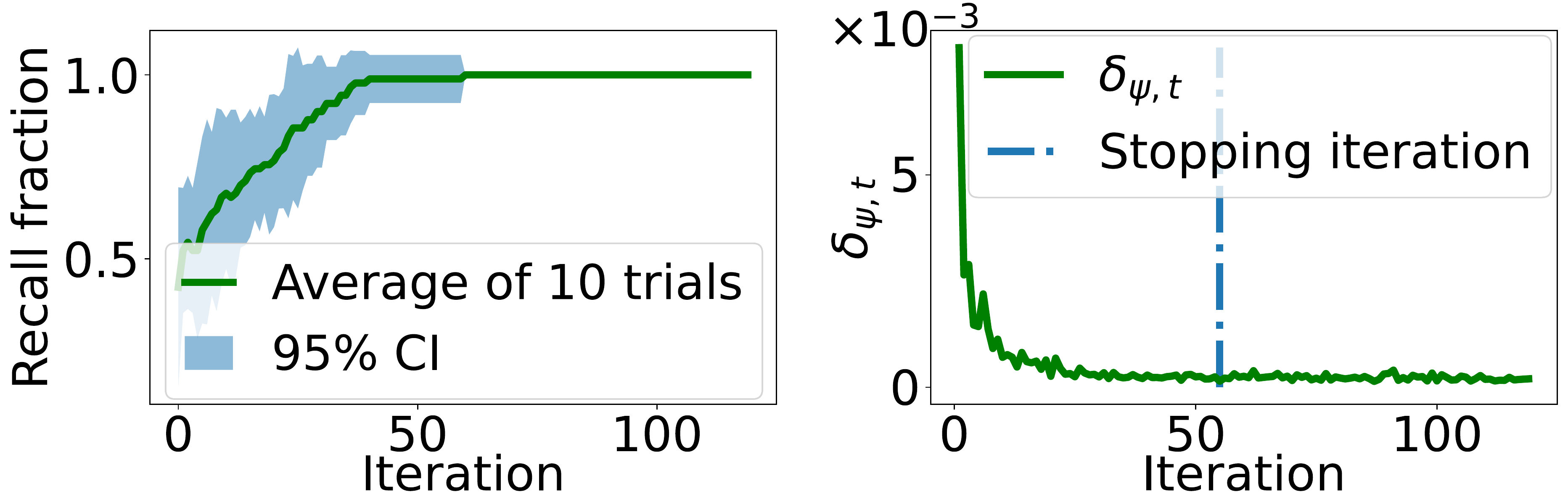}

    \caption{Left: Recall fraction increases to $100\%$ (i.e., optimal) over iterations with
    shaded area denoting the $95\%$ confidence interval.
    Right: Fluctuations $\delta_{\psi,t}\coloneqq |\boldsymbol{\psi}_t - \boldsymbol{\psi}_{t-1}|_\infty$.
    }
    \label{fig:motivation-accurate-sv}
\end{figure}

\cref{prop:hypothesis-testing} presents a stopping criterion based on a test for the lack of statistical evidence $\boldsymbol{\psi}_t$ is fluctuating between $t_s -\tau$ and $t_s$ w.r.t.~a pre-set significance level $\alpha$. We denote the stopping iteration as $T_{\alpha}$.
The significance level $\alpha$ reflects the strictness of the stopping criterion: a larger $\alpha$ rejects $h_{0,t_s}$ more easily and is stricter about the convergence and accuracy of $\boldsymbol{\psi}_t$.
With more iterations, we expect $h_{0,t_s}$ to hold, i.e., $\boldsymbol{\psi}_t$ converges and the $p$-value to be relatively large, so we choose a large $\alpha$~(e.g., $0.5$) and obtain $T_{0.5}=55$ in \cref{fig:motivation-accurate-sv}~(right) with relatively accurate contribution estimates $|\boldsymbol{\psi}_{T_{0.5}} - \boldsymbol{\psi}^* |_{\infty} = 4.4\text{e-}3$ where $\boldsymbol{\psi}^*$ is empirically approximated with $\boldsymbol{\psi}_t$ for the last iteration.

A technical caveat of Hotelling's $T^2$ distribution is it requires $\tau > N$, which means to ensure the accuracy in $\boldsymbol{\psi}_t$ for a larger number of nodes requires a longer contribution evaluation. This implies the criterion is less feasible with a larger $N$. To mitigate this limitation on the feasibility with large $N$, we describe a sub-sampling technique: select a small random subset of $M < N$ nodes for the stopping criterion w.r.t.~their $\psi_{i,t}$. Therefore, only $\tau > M$ is required. Note that the training and update of $\psi_{i,t}$ is as usual for all $N$ nodes. \cref{fig:p-value-and-convergence-complexity} (left) shows the original $p$-value (black) and that with sub-sampling (orange) align well, and validates this mitigation. Further verification is provided in \cref{appendix:experiments}. We apply sub-sampling in our experiments in \cref{sec:experiments}.
Separately, as we focus on the \emph{cross-silo} setting \cite{Wang2021_FL_field_guide,Zhang2022_FL_long_term} where the nodes have reliable connections, we implicitly assume they all can participate in each $t$ during contribution evaluation. If this assumption is not satisfied in practice, our framework can still be applied with a simple modification ($\phi_{i,t}=0$ if $i$ does not participate in iteration $t$ \citep{Wang2020-SV-in-FL}), but it will take longer for $\boldsymbol{\psi}_t$ to converge as shown in \cref{appendix:experiments}.

Importantly, \cref{fig:p-value-and-convergence-complexity} (left) shows $\boldsymbol{\psi}_t$ converges~(red dashed line) earlier than $\theta_t$, as the loss continues to decrease, meaning that \emph{the incentive realization starts when $\theta_t$ is somewhat close to but not quite at convergence}.
In practice, we expect the contribution evaluation to be relatively short, and is observed to be about $1/5$ in length to incentive realization (i.e., train $\theta_t$ to convergence) in our experiments. 
The implication is that fairness is guaranteed when $\theta_t$ starts to have competitive performance instead of at very early stage of training.
In other words, during contribution evaluation, all nodes are willing to share because the performance of $\theta_t$ is not very competitive. When it comes to incentive realization where $\theta_t$ will be trained to convergence, it is important to guarantee fairness for the nodes.

\begin{figure}[t]
    \centering 
    \includegraphics[width=0.98\linewidth]{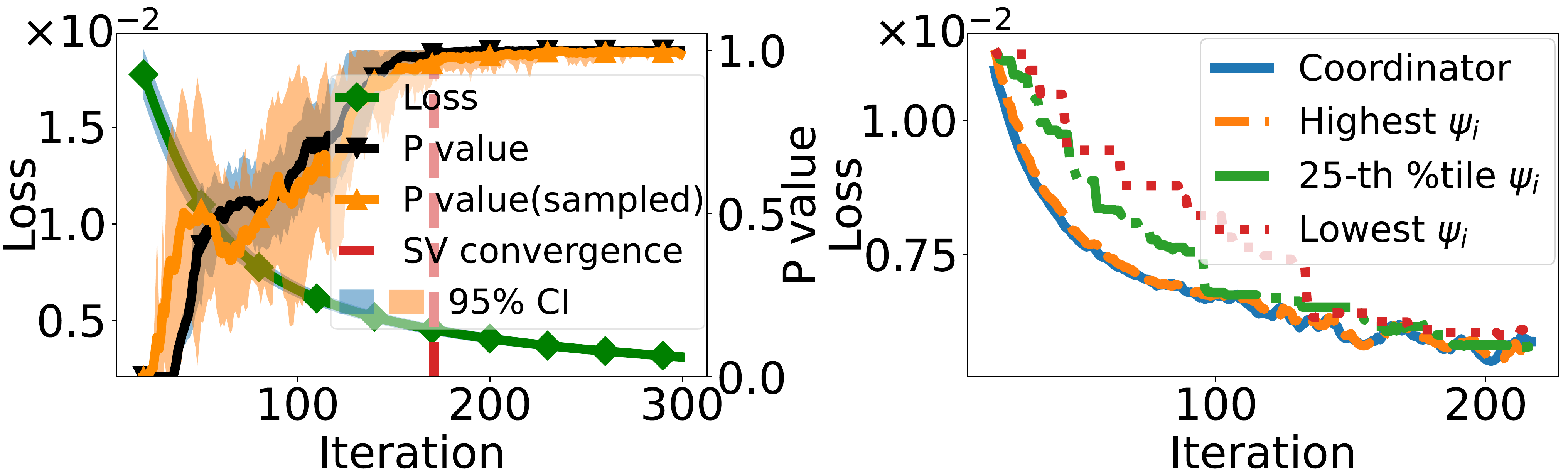}

    \caption{Left: The $p$-value with~(orange) and without~(black) sub-sampling using $\tau = 20$~($35$). Red dashed line marks when $p$-value reaches $0.99$ (i.e., $\boldsymbol{\psi}_t$ has converged).
    Right: Average (over 10 trials) validation loss for the nodes with highest, 25-th percentile and lowest $\psi_{i, T_{\alpha}}$ 
    with $\beta=0.01$.
    }
    \label{fig:p-value-and-convergence-complexity}
\end{figure}
\subsection{Convergence-based Incentive Realization}
\label{sec:stage2}
We design a convergence-based incentive realization by carefully managing the convergence behaviors of the nodes' models via the node sampling distribution, used in FL to select $k<N$ nodes to synchronize with the latest model in each iteration $t$~(and conduct local training to upload their model updates).
We adopt the sampling with replacement scheme~\citep{Li2019-fedavg-noniid,li2020-fl-sampling-with-replacement} as it admits closed-form expressions for analysis. The sampling probability $\varrho_i$ for $i$ is given by the \texttt{softmax} of $\psi_{i,T_{\alpha}}$ and an equalizing coefficient $\beta > 0$~(a.k.a. the temperature parameter) as follows:
\begin{equation} 
\label{equ:sampling-probs}
    \varrho_i \coloneqq \exp\left(  \psi_{i,T_{\alpha}}/\beta\right)  / \textstyle \sum_{i' \in [N]}\exp\left( \psi_{i', T_{\alpha}}/\beta\right).
\end{equation}
For subsequent discussion,
define $q_i\coloneqq 1 - (1 - \varrho_i)^k$ as the probability $i$ is selected in the subset of size $k$ following a counting-based argument w.r.t.~$\varrho_i$. 

\paragraph{Fair convergence complexities.} \label{sec:fair-complexities}
We design the incentives so that the nodes with higher contributions receive the latest model more frequently. This implies these nodes have less expected staleness in their models and lower/better convergence complexities.
The \textit{staleness} $\gamma_{i,t}$ for node $i$ in iteration $t$ is defined as the difference between the current iteration $t$ and the most recent iteration $t'$ in which node $i$ was selected to synchronize $\theta_{i,t'}$ with $\theta_{t'}$. E.g., if node $i$ is selected in iteration $t$ for the update, then $t'=t$ and $\gamma_{i,t}=0$ (staleness resets to $0$ every time $i$ is selected).
Subsequently, we define \textit{expected staleness} for node $i$ from any $t$ as
$\Gamma_i \coloneqq \textstyle \sum_{\gamma =0}^\infty  \mathbb{P}[\text{stale for } \gamma \text{ iterations}] \times \gamma $, and show that $\Gamma_i =(1-q_i) / q_i^2 $ (\cref{sec:appendix-proofs}).
Utilizing a convergence $\mathcal{O}(1/\epsilon)$ for $\theta_t$ \citep[Theorem 2]{Li2019-fedavg-noniid},\footnote{The expected number of iterations for $\mathbb{E}[\mathbf{J}(\theta_t)] - \min_{\theta}\mathbf{J}(\theta) \leq \epsilon$ (assuming stationary data distribution).} we derive an expected \emph{convergence complexity} for node $i$ as $C_i \coloneqq \mathcal{O}(1/\epsilon) + \Gamma_i$, as the sum of the expected number of iterations for $\theta_t$ to converge, and the expected number of iterations for $i$'s model to catch up to $\theta_t$, with the following fairness guarantee:
\begin{proposition}[\bf Fair Expected Convergence] \label{proposition:fairness}
The expected convergence complexity $C_i$ satisfies:
\begin{itemize}[noitemsep,nolistsep]
    \item \textit{symmetry}: $\psi_{i, T_\alpha} = \psi_{i', T_\alpha} \implies C_i = C _{i'}\ ;$
    \item \textit{strict desirability}: $\psi_{i, T_\alpha} > \psi_{i', T_\alpha}\implies C_i < C_{i'}\ ;$
    \item \textit{strict monotonicity}: Supposing $\forall i' \neq i\ \ \psi_{i',T_\alpha}$ is fixed, $\psi'_{i, T_{\alpha} } > \psi_{i, T_{\alpha}} \implies C'_{i} < C_i\ .$
\end{itemize}
\end{proposition}
Its proof with a formal discussion on the respective sufficient conditions are in \cref{sec:appendix-proofs}.
Symmetry implies that nodes with equal contributions have equal convergence complexities.
Strict desirability guarantees a node with a higher contribution has a lower convergence complexity. 
Strict monotonicity incentivizes the nodes to make high contributions because if a node makes a higher contribution $\psi'_{i, T_{\alpha} } > \psi_{i, T_{\alpha}}$, \emph{ceteris paribus}, its corresponding convergence complexity is lower.
\cref{fig:p-value-and-convergence-complexity} (right) empirically illustrates this result (using MNIST with label noise to vary the quality of data from different nodes) that the nodes with higher $\psi_{i,T_\alpha}$ converge faster (i.e., loss of local model decreases faster).

\paragraph{Asymptotic equality.}
The proposed \cref{equ:sampling-probs} mitigates the ``rich get richer'' by ensuring all nodes having controllable non-zero probabilities of being selected in each $t$ and resets it staleness to $0$. This results in the asymptotically equal convergence complexities. 
\begin{proposition}
\label{prop:equal-convergence}
Let $i_* \coloneqq \argmin_{i\in[N]} \psi_{i, T_{\alpha}}$. 
$(\Gamma_{i_*} = \mathcal{O}(1/\epsilon)) \implies (\forall i\in[N]\ \ C_i = \mathcal{O}(1/\epsilon))\ .$
\end{proposition}

\cref{prop:equal-convergence} (its proof is in \cref{sec:appendix-proofs}) states that given the sufficient condition, all nodes, regardless of the contributions, have asymptotically equal convergence complexities at $\mathcal{O}(1/\epsilon)$,\footnote{$\mathcal{O}(1/\epsilon)$ is an FL-dependent convergence complexity~\cite{Li2019-fedavg-noniid} and not specific to our framework.} 
which coincides with the complexity for the node with lowest contribution due to the fairness guarantee. 
The sufficient condition states in order to preserve equality, no node can be left stale for too long: no longer than it takes for $\theta_t$ to converge.

For a finite $\beta >0$, \cref{proposition:fairness} guarantees fairness in $C_i$ so we analyze the effect of $\beta$ on $\Gamma_i$ to find the suitable range for $\beta$. We use some synthetic $\boldsymbol{\psi}_{T_{\alpha}}$ to illustrate $\beta \to \infty$ has an equalizing effect (i.e., $\Gamma_i, \forall i \in [N]$ converges to the same value) while $\beta \to 0$ leads to the ``rich get richer'' inequality in \cref{fig:effects-of-beta-Gamma} where $i=1$ ($i=10$) has the lowest (highest) $\psi_{i,T_{\alpha}}$ with the correspondingly highest (lowest) $\Gamma_i$. Intuitively, if $\beta$ is too large and equalizes $\Gamma_i$ (thus also equalizes $C_i$), then it violates the fairness guarantee. Specifically, the strict desirability is violated, since node $i$ with strictly higher contribution than some node $i'$ does not receive strictly better incentive/lower complexity. Therefore, it reduces the effectiveness of the incentives. In contrast, if $\beta$ is too small, the poor/nodes with lower contributions start to `suffer'/have much larger expected staleness $\Gamma_i$ and be discouraged from collaborating. 
Ideally, $\beta$ should be set to achieve a proportionality $0 < r_1 \leq \psi_i/(1/\Gamma_i)  \leq r_2$ to guarantee fairness and preserve equality. 
Interestingly, $r_1=r_2$ coincides with Shapley fairness \citep[Definition 1]{Sim2020}.
In \cref{sec:appendix-proofs}, we formalize $\beta$'s selection via \cref{lemma:finding-beta}, further analyze the difficulties of preserving equality via $\beta$ and discuss how to prevent $i_*$ from having an arbitrarily bad $\Gamma_{i_*}$.

\begin{figure}[!ht]
    \centering
    \includegraphics[width=0.90\linewidth]{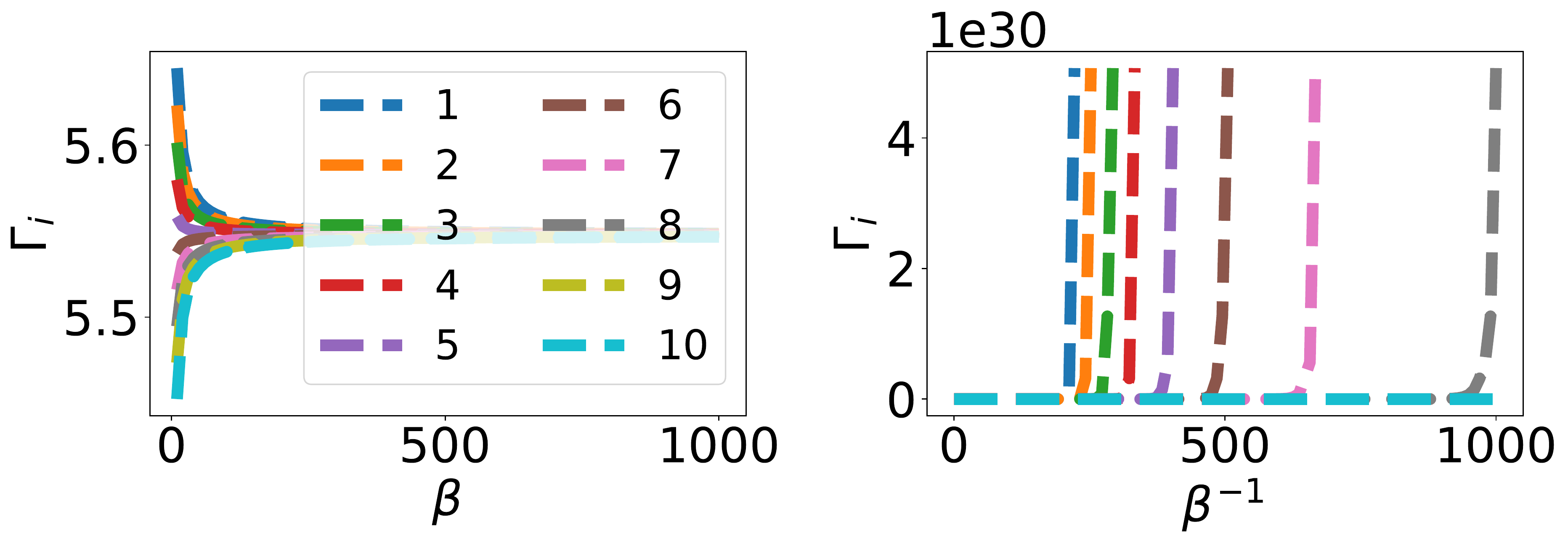}
    \caption{
    Effects of $\beta \to \infty $ (left) and $\beta \to 0$ (right) on $\Gamma_i$. 
    $N=10, k=4$, $\boldsymbol{\psi}_{i, T_{\alpha}} \propto i$ and $|\boldsymbol{\psi}_{T_{\alpha}}|_1 = 1$, for illustration.
    }
    \label{fig:effects-of-beta-Gamma}
\end{figure}

\section{Experiments}\label{sec:experiments}
Our implementation is publicly available at \url{ https://github.com/xqlin98/Fair-yet-Equal-CML}.
\subsection{Experimental Settings}
We investigate 
\textit{federated online incremental learning}~(FOIL)~\cite{LOSING20181261} and \textit{federated reinforcement learning}~(FRL)~\cite{qi2021federated} where new data become available and used during training \cite{Jin2021-streaming-data,Le2021_continuous_stream_fl}. 
\textbf{Data setting.}
At each $t$, each node randomly selects a batch $s_{i,t}$ from its local data $\mathcal{D}_i$ to simulate new data for training, and the aggregation of batches from all nodes at $t+1$ is used for testing the model updated at $t$.
\textbf{Online performance and evaluation metrics.}
As nodes are interested in the latest model $\theta_t$ during training instead of only the final performance, we adopt an online performance \citep{LOSING20181261}: 
$P_{i, \text{online}} \coloneqq t^{-1} \sum_{l=1}^{t} \mathbf{P}(\theta_{i,l-1}, s_l)$ where $s_{t} \coloneqq \bigcup_{i=1}^{N}s_{i,t}$ is the collection of all nodes' latest data and $\mathbf{P}(\theta, s)$ is a performance measure of $\theta$ on $s$~(e.g., test accuracy). 
To evaluate \textit{fairness}, we adopt the Pearson correlation coefficient: $\rho \coloneqq \text{Pearson}(\text{contributions}, \text{incentives})$ where a high value of $\rho$ suggests fairness~\cite{Xu2021-fair-CML}.
The contributions are specified for both learning tasks subsequently. For incentives, we use $P_{i, \text{online}}$~(higher is better incentive) or average staleness $\bar{\gamma}_i \coloneqq T^{-1} \textstyle \sum_{t=1}^T \gamma_{i,t}$~(lower is better incentive).
To evaluate \textit{equality}, we use standard deviation and worst-node performance \cite{li2019-qFFL}.
\textbf{Comparison baselines.} 
(a) FedAvg as the vanilla FL framework~\citep{BrendanMcMahan2017-FedAvg}, (b) $q$-fair FL~($q$FFL) \citep{li2019-qFFL} aiming at equalizing model performance across the nodes, (c) fair gradient rewards in FL~(FGFL) \citep{Xu2021-fair-CML} using model updates to ensure fairness, (d) game of gradients~(GoG)~\citep{Nagalapatti2021-game-of-gradients-fl} dynamically updating the sampling distribution according to SV to favor high-quality nodes, and (e) standalone training (Standalone) without collaboration. We also compare the communication costs and running times of these approaches in \cref{appendix:experiments}. Apart from these existing baselines, we compare a simple extension to FedAvg, namely Cons, which uses constraints to achieve equality in \cref{appendix:experiments}.

\subsection{Federated Online Incremental Learning}\label{sec:foil}

\textbf{Datasets and varying data quality to control true contributions.}
We perform experiments on the following datasets:
(a) image classification tasks: MNIST~\citep{mnist} and CIFAR-10~\citep{cifar10}, (b) medical image classification task: Pathmnist~(PATH)~\citep{medmnistv1} containing images of tissues related to colorectal cancer, (c) high frequency trading dataset~(HFT)~\citep{hft2018} as a time-series task to predict the increase/decrease of the bid price of the financial instrument, and (d) electricity load prediction 
task~(ELECTRICITY)~\citep{electricity2020} as a time-series task to predict the electricity load in German every $15$ minutes.
Recall in \cref{sec:introduction}, we described application scenarios which require real-time decision making (HFT and ELECTRICITY) or long-term medical data collection (PATH). 

To simulate the local data $\mathcal{D}_i$, each dataset is partitioned into $N$ subsets uniformly randomly and distributed to the nodes. 
In addition, we vary data qualities/contributions of different nodes by setting different levels of added noises (e.g., due to observational errors \citep{medicalerror2004,profit-allocation-FL-Song2019,Wang2020-SV-in-FL}), quantities of data, and levels of missing values in feature. Therefore, we can control the true contributions $\boldsymbol{\psi}^*$. So, we can verify whether the incentives are fair, i.e., the node with the best/least noisy data, receives the best incentive/online performance. For noise, we consider two types of noise, in features or labels. Specifically, for feature (label) noise, $\zeta_i$ proportion of data on node $i$ is added with independent zero-mean Gaussian noise with variance $1$~(has the label flipped to a random label) where the noise ratio $\zeta_i$ varies between $0\%$ and $90\%$ ($0\%$ and $20\%$) across nodes. For quantity, we follow a power law partition of data to have data unequally distributed to different nodes~\cite{BrendanMcMahan2017-FedAvg} and $\zeta_i$ is set to be the negative quantity of node $i$ (i.e., $-|D_i|$), so that a larger $\zeta_i$ indicates a lower contribution. For missing values, we select $\zeta_i$ (varies between $0\%$ to $90\%$) proportion of data on node $i$ to assign $0$ to $50\%$ of its features/pixels randomly. Due to space limitations, we present the comparative results for feature noise and defer the rest to \cref{appendix:experiments}.

\textbf{Hyper-parameter details.}
In $t$, each of the $N=30$ nodes trains on their own latest data $s_{i,t}$ for $E=1$ epoch. 
For \textit{contribution evaluation}, we use the stopping criterion by setting $\alpha=0.7, \tau=15$ on $10$ sub-sampled nodes.
The total number of iterations is the same for all baselines~(including ours): $150$ for CIFAR-10, HFT, ELECTRICITY and PATH and $130$ for MNIST.
For \textit{incentive realization}, $k=12$~(i.e., $40\%$ ratio) and $\beta=1/150$. 
For FedAvg, $q$FFL and FGFL, the selection ratio is $40\%$.
For MNIST, we use the same CNN model as in \cref{sec:setting}. The optimization algorithm is \textit{stochastic gradient descent} with a learning rate of $0.002$ on $s_{i,t}$ as a batch (for MNIST, $|s_{i,t}|=3$). Additional details on hyper-parameters are in \cref{appendix:setings-experiments}. 

\textbf{Results.}
We present the fairness results, averaged over $5$ random trials~(standard errors in brackets) in \cref{tab:by-class-vs-not-by-class-pearson} where a high correlation $\rho$ between $\boldsymbol{\zeta}$ and online loss/average staleness indicates fairness.
As both FedAvg and $q$FFL do not consider fairness as in SV, they do not perform well. 
Both FGFL and GoG achieve fairness inconsistently as they both begin realizing the incentives immediately from $t=1$ based on possibly inaccurate contribution estimates (further illustrated in \cref{appendix:experiments}), while our approach first ensures the accuracy in $\boldsymbol{\psi}_{T_\alpha}$.
To directly validate our theoretical results, \cref{table:pearson-staleness} in \cref{appendix:experiments} shows $\rho$ w.r.t.~the average staleness and our approach performs best overall.

\begin{table}[!ht]
    \caption{
    $\rho(\text{online loss}, \zeta)$
    in FOIL under the setting of feature noise. \textit{Higher} $\rho$ implies better fairness.
    }
    \centering
    \resizebox{\linewidth}{!}{
    \label{tab:by-class-vs-not-by-class-pearson}
    \begin{tabular}{c|ccccc}
    \toprule 
     & MNIST & CIFAR-10 & HFT & ELECTRICITY & PATH \\
     \midrule 
    FedAvg&-0.020(0.097)&0.137(0.049)&0.038(0.045)&-0.033(0.038)&0.135(0.026)\\
    qFFL&-0.022(0.114)&-0.109(0.140)&0.060(0.078)&0.036(0.126)&-0.236(0.030)\\
    FGFL&0.556(0.032)&0.313(0.081)&0.055(0.033)&0.476(0.057)&0.419(0.098)\\
    GoG&0.551(0.023)&0.130(0.067)&0.201(0.021)&0.512(0.027)&0.189(0.102)\\
    \midrule 
    Ours&\textbf{0.647(0.018)}&\textbf{0.400(0.069)}&\textbf{0.378(0.055)}&\textbf{0.676(0.018)}&\textbf{0.557(0.060)}\\
    \bottomrule 
    \end{tabular}
    }
\end{table}

\begin{table}[!ht]
\caption{Average/Minimum of online accuracy~(standard error) over all nodes under the setting of feature noise. For ELECTRICITY, we measure MAPE, so lower is better.}
\begin{subtable}{0.99\linewidth}
    \centering
    \caption{Average of online accuracy}
    \label{table:average-online-acc-feature-noise}
    \resizebox{\linewidth}{!}{
    \begin{tabular}{c|ccccc}
    \toprule 
     & MNIST & CIFAR-10 & HFT & ELECTRICITY & PATH \\
     \midrule 
    FedAvg&0.483(0.019)&0.166(0.011)&0.499(0.046)&1.408(0.081)&0.255(0.004)\\
    qFFL&0.101(0.011)&0.100(0.004)&0.281(0.079)&1.413(0.055)&0.101(0.011)\\
    FGFL&0.485(0.018)&0.169(0.010)&0.496(0.056)&1.571(0.090)&0.154(0.009)\\
    GoG&0.572(0.015)&0.193(0.006)&0.556(0.016)&1.394(0.032)&0.288(0.004)\\
    Standalone&0.481(0.013)&0.153(0.004)&0.540(0.014)&1.581(0.083)&0.202(0.003)\\
    \midrule
    Ours&\textbf{0.611(0.009)}&\textbf{0.195(0.007)}&\textbf{0.581(0.014)}&\textbf{0.139(0.002)}&\textbf{0.302(0.005)}\\
    \bottomrule 
    \end{tabular}
    }

    \caption{Minimum of online accuracy}
    \label{table:minimum-online-acc-feature-noise}
    \resizebox{\linewidth}{!}{
    \begin{tabular}{c|ccccc}
    \toprule 
     & MNIST & CIFAR-10 & HFT & ELECTRICITY & PATH \\
     \midrule 
    FedAvg&0.478(0.020)&0.165(0.011)&0.497(0.046)&1.405(0.081)&0.250(0.005)\\
    qFFL&0.101(0.011)&0.100(0.004)&0.281(0.079)&1.413(0.055)&0.101(0.011)\\
    FGFL&0.437(0.030)&0.167(0.010)&0.493(0.058)&1.497(0.071)&0.153(0.008)\\
    GoG&0.553(0.014)&0.189(0.006)&0.548(0.017)&1.383(0.032)&0.271(0.004)\\
    Standalone&0.279(0.024)&0.131(0.006)&0.515(0.017)&1.281(0.065)&0.131(0.006)\\
    \midrule 
    Ours&\textbf{0.603(0.010)}&\textbf{0.193(0.007)}&\textbf{0.581(0.014)}&\textbf{0.139(0.002)}&\textbf{0.298(0.005)}\\
    \bottomrule 
    \end{tabular}
    }
\end{subtable}
\end{table}

\cref{table:average-online-acc-feature-noise} shows our approach performs the best overall w.r.t.~$P_{i,\text{online}}$, as it carefully manages the staleness among local models so that no one is left behind for too long.
\cref{fig:cml_vs_ours_convergence_path} (right) under the setting in \cref{sec:foil} shows FGFL can lead to non-convergence (mentioned in introduction).
It implies some nodes may leave before the end to prevent their performance from deteriorating, and reduces the overall effectiveness of the collaboration. In contrast, our approach (\cref{fig:cml_vs_ours_convergence_path} left) ensures all nodes eventually have the optimal model (i.e., coordinate node) and we compare our theoretical fairness guarantee with theirs in \cref{sec:appendix-proofs}.
Moreover, the results for the poorest/best-performing nodes in 
\cref{table:minimum-online-acc-feature-noise,table:maximum-online-acc} (\cref{appendix:experiments}) show our approach gives a better worst/best-performance than the baselines. Moreover, even the best-performing nodes improve their performance (\cref{table:maximum-online-acc} in \cref{appendix:experiments}), which verifies that incentivizing the less resourceful nodes to join can improve the overall performance.

\begin{figure}[!ht]
    \centering
    \includegraphics[width=0.45\linewidth]{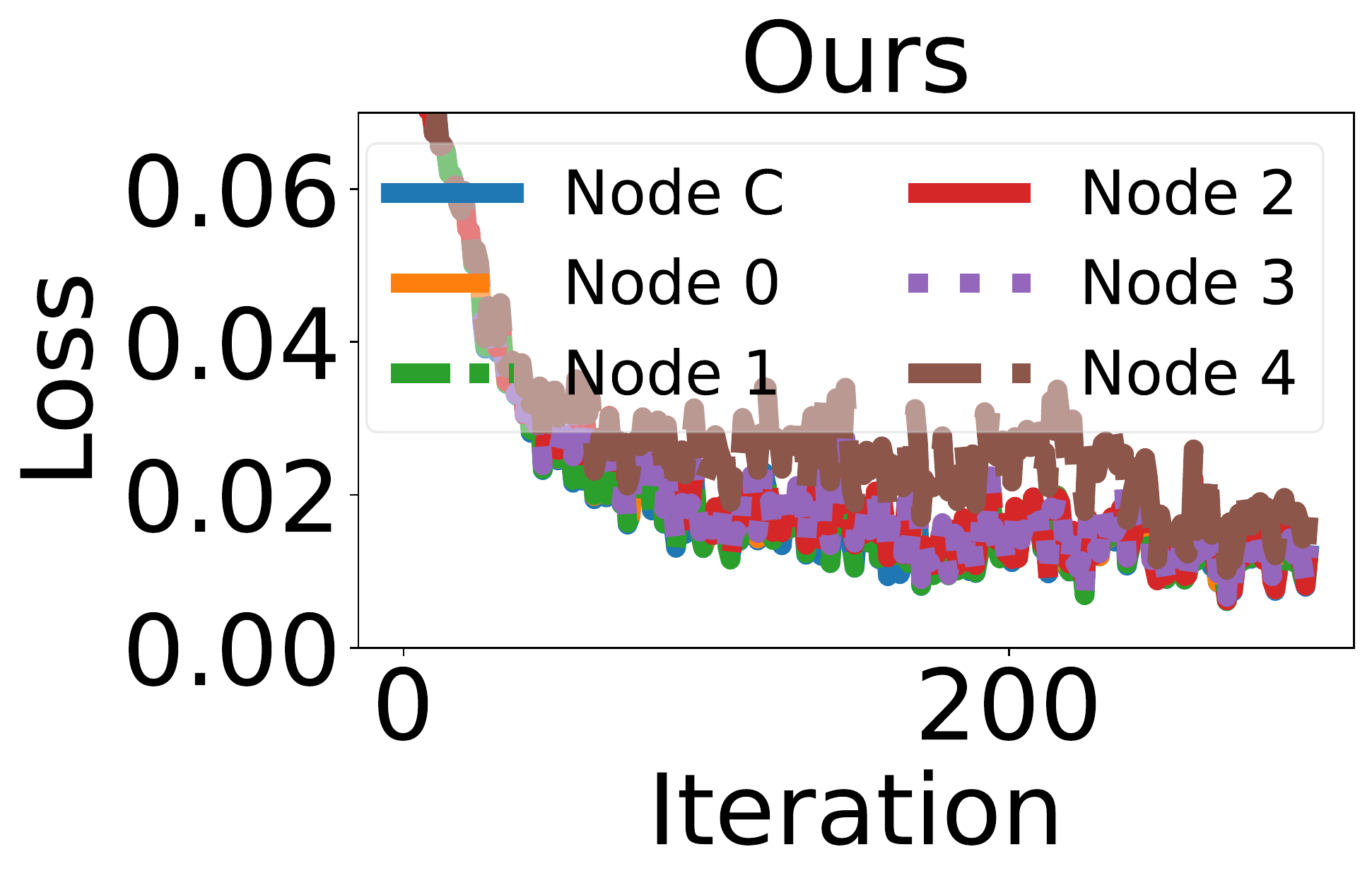}
    \includegraphics[width=0.45\linewidth]{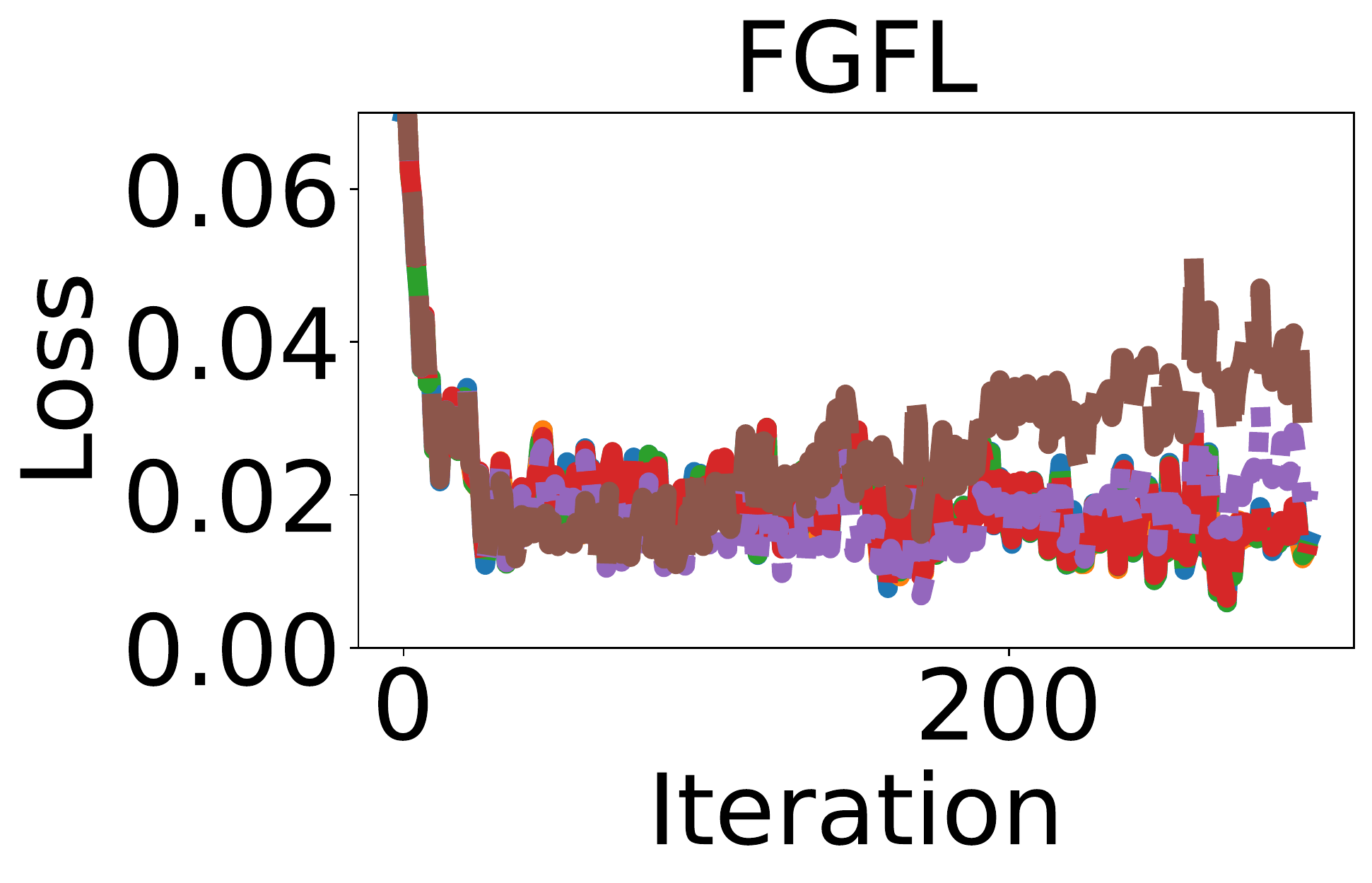}
    \caption{
    Validation loss (over $5$ random trials) of each node of Ours vs.~FGFL. Node C is the coordinator node. The setting is label noise, $N=5, \tau=10, \alpha=0.5$.
    }
    \label{fig:cml_vs_ours_convergence_path}
\end{figure}

\begin{figure}[!ht]
    \centering
    \includegraphics[width=0.45\linewidth]{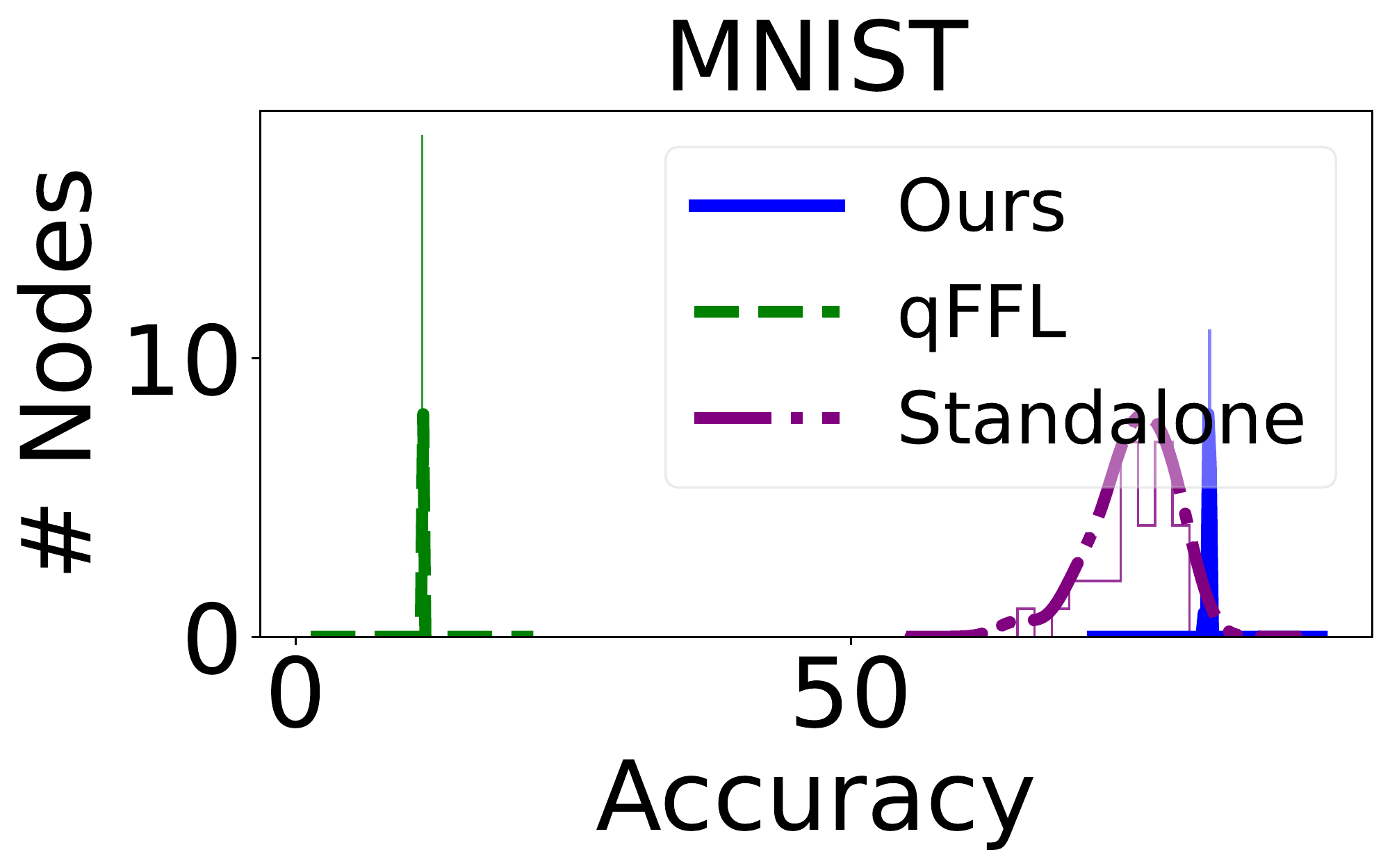}
    \includegraphics[width=0.45\linewidth]{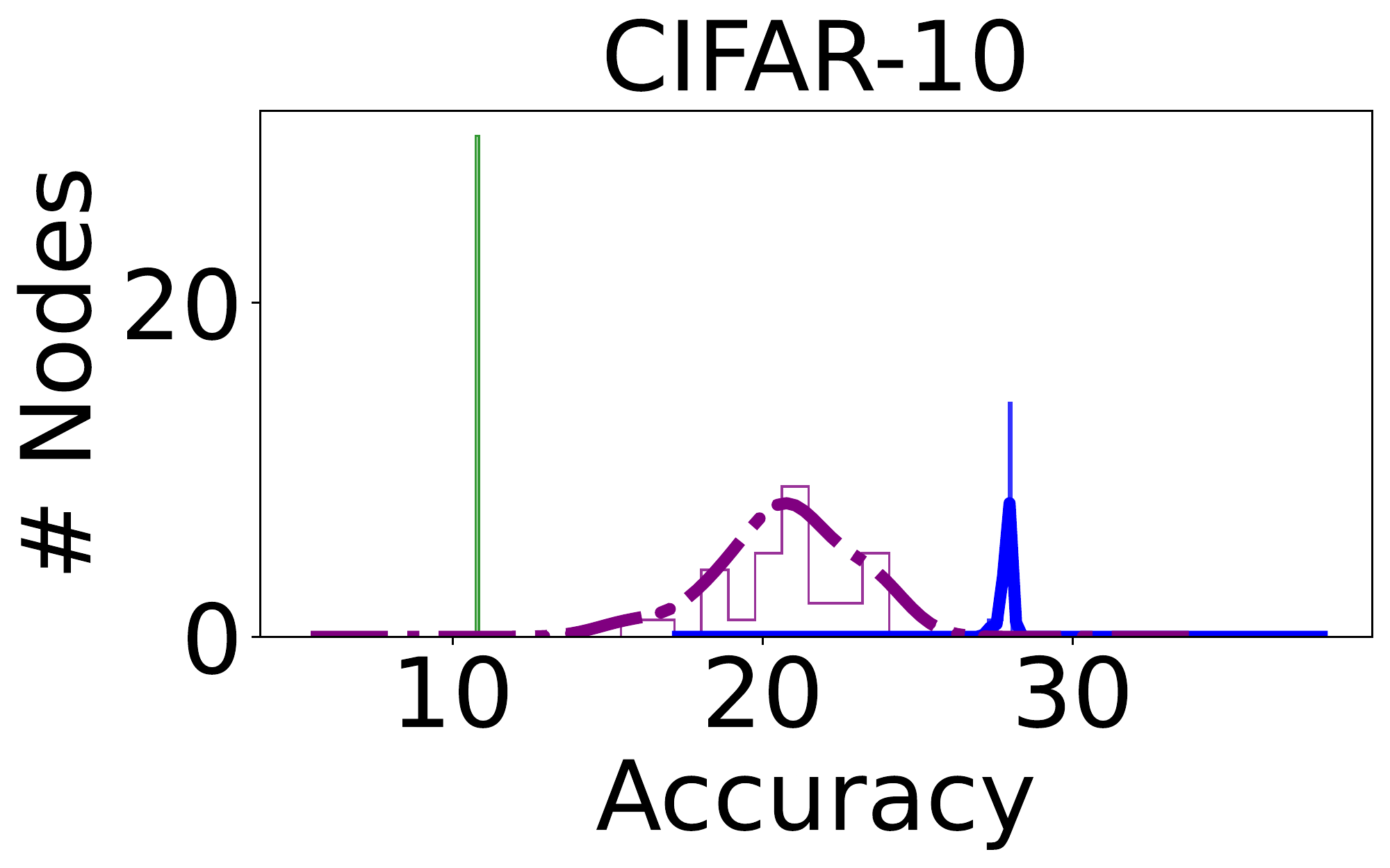}
    \caption{
    Distribution of average (validation) accuracy of final $10$ iterations (over $5$ random trials) under \textit{label} noise. 
    }
    \label{fig:equality-comparison}
\end{figure}

Lastly, \cref{table:minimum-online-acc-feature-noise} shows our approach performs the best overall in preserving the equality in terms of max-min \cite{li2019-qFFL} (i.e., ensuring the worst-performing node has a high online and final performance). 
Moreover, w.r.t.~performance variation \cite{li2019-qFFL}, \cref{fig:equality-comparison} (and quantitatively in \cref{table:standard-dev-online-accu} in \cref{appendix:experiments}) shows our method performs competitively to $q$FFL, as the accuracy is concentrated with small variation. However, $q$FFL performs poorly in terms of validation accuracy (in \cref{fig:equality-comparison} and \cref{table:average-online-acc-feature-noise}) because of its exponentiation of \textit{local objectives} which are biased due to noise in data (i.e., optimizing the local objective for the node having noisiest data with higher weight impairs learning performance for the others). We provide additional comparisons of the fairness and equality performance of these approaches under different degrees of heterogeneity among the local data of the nodes in \cref{appendix:experiments}.

\textbf{Empirical fairness vs.~equality trade-off via $\beta$.}
We empirically find a suitable range of $\beta$ to be $[1/150,1]$ as shown in \cref{table:fairness-varying-beta-feature-noise} and verify the results in \cref{fig:effects-of-beta-Gamma}. Specifically, under the setting of feature noise, we find that $\beta \in [1/150,1]$ produces the most balanced results for fairness (via the correlation coefficient $\rho (\text{online loss}, \boldsymbol{\zeta})$) and equality (via the standard deviation/std of online accuracy).\footnote{We adopt online accuracy to illustrate the rates at which the nodes converge are comparable (i.e., asymptotic equality) instead of final test accuracy which represents the (asymptotic) performance (\cref{table:fairness-varying-beta-final-acc} in \cref{appendix:experiments})}
We highlight \cref{proposition:fairness} guarantees fairness w.r.t.~the \textit{expected} convergence complexities. Hence, as $\beta$ increases, while the expectations $\Gamma_i$ and $C_i$ are guaranteed to be fair, the actual realized model performance (e.g., online loss) can observe lower fairness.

\begin{table}[!ht]
	\centering
    \setlength{\tabcolsep}{2pt}
    \caption{
    $\rho(\text{online loss}, \boldsymbol{\zeta})$ (std of online accuracy).
    Higher $\rho$ (lower std) means better fairness (equality).
    }
    \label{table:fairness-varying-beta-feature-noise}
    \resizebox{\linewidth}{!}{
    \begin{tabular}{c|ccccc}
    \toprule 
    $\beta$ & MNIST & CIFAR-10 & HFT & ELECTRICITY & PATH \\
     \midrule 
    1/350&0.642(9.52e-03)&0.490(1.86e-03)&0.448(6.04e-05)&0.581(1.63e-03)&0.516(3.18e-03)\\
    1/150&0.647(2.2e-03)&0.400(6.5e-04)&0.378(5.8e-05)&\textbf{0.676}(2.95e-04)&\textbf{0.557}(1.41e-03)\\
    1/100&\textbf{0.705}(1.13e-03)&\textbf{0.507}(3.80e-04)&0.415(1.05e-04)&0.572(1.63e-04)&0.312(1.15e-03)\\
    1/50&0.641(5.66e-04)&0.297(2.67e-04)&\textbf{0.476}(\textbf{8.88e-17})&0.286(8.17e-05)&0.282(8.32e-04)\\
    1/20&0.466(4.91e-04)&0.131(3.05e-04)&0.217(1.84e-06)&0.166(\textbf{5.15e-05})&0.127(7.91e-04)\\
    1/10&0.120(4.97e-04)&0.034(2.91e-04)&0.090(1.24e-05)&0.068(6.01e-05)&0.018(7.49e-04)\\
    1&0.171(\textbf{3.79e-04})&0.063(\textbf{2.43e-04})&-0.187(1.30e-05)&0.193(6.11e-05)&0.082(\textbf{6.41e-04})\\
    1000&-0.005(3.88e-04)&-0.185(2.77e-04)&-0.079(1.61e-05)&-0.034(5.18e-05)&0.157(7.19e-04)\\
    \bottomrule 
    \end{tabular}
    }
\end{table}

\subsection{Federated Reinforcement Learning}
We investigate FRL as a natural setting where the data stream in as the agent explores a complex environment and collects its observed states, actions and reward signals. Having a latest model is beneficial in FRL as it guides the agent to explore the environment more ``cleverly'' instead of randomly.
We investigate three Atari games: Breakout, Pong, and SpaceInvaders where all $N=5$ (and $k=2$) nodes explore copies of the same game in parallel for $T=450$ iterations.
We add different levels of noise to $\zeta_i$ proportion of the observed states (images with pixel values from $[0,255]$ are added with zero-mean Gaussian noise with variance $50$ to each pixel) or reward signals (discrete values from $\{-1,0,1\}$ are randomly reassigned) by each node to control their true contributions.
$\boldsymbol{\zeta} = [0.6, 0.4, 0.2, 0, 0]$ for Breakout and SpaceInvaders and $\boldsymbol{\zeta} = [0.2, 0.1, 0.05, 0, 0]$ for Pong (more noise sensitive).
We adopt the Deep Q-Networks~(DQN)~\cite{mnih2013playing} with minor modifications (smaller learning rate and memory size).
For our approach, $\alpha=0.95, \tau=20$, and $\beta=0.01$. Additional details on other hyper-parameters are in \cref{appendix:setings-experiments}.

\textbf{Online score evaluation and fairness results.} 
We evaluate $\theta_{i,t}$ via its average game score in $5$ episodes by following an $\epsilon$-greedy~(for $\epsilon=0.02$) policy as $\mathbf{P}(\theta_{i,t})$ for computing $P_{i,\text{online}}$.
\cref{table:fedrl-noise-scores-benchmark} (brackets indicate standard errors over $3$ independent trials) shows while FGFL and GoG perform competitively, our approach performs best overall.

\begin{table}[!ht]
    \centering
    \caption{
    $\rho(\text{online score}, \zeta)$ 
    in FRL under the settings of reward noise and state noise. 
    \textit{Lower} $\rho$ indicates better fairness. 
    }
    \label{table:fedrl-noise-scores-benchmark}
        \resizebox{\linewidth}{!}{
    \begin{tabular}{c|ccc|ccc}
    \multicolumn{1}{c}{}  & \multicolumn{3}{c}{\bf Reward Noise} & \multicolumn{3}{c}{\bf State Noise}\\
     \toprule
     & Breakout & Pong & SpaceInvaders &  Breakout & Pong & SpaceInvaders \\
     \midrule
    FedAvg&0.169(0.755)&0.329(0.298)&-0.036(0.371)
    &-0.160(0.343)&0.018(0.257)&0.164(0.262)\\
    qFFL&-0.229(0.251)&0.136(0.506)&0.407(0.361)
    &-0.049(0.308)&-0.486(0.360)&-0.173(0.322)\\
    FGFL&-0.884(0.018)&-0.861(0.028)&-0.377(0.223)&-0.858(0.049)&-0.760(0.002)&-0.054(0.043)\\
    GoG&-0.898(0.040)&-0.869(0.060)&-0.399(0.167)&-0.481(0.172)&0.328(0.352)&\textbf{-0.298(0.390)}\\
    \midrule
    Ours&\textbf{-0.900(0.032)}&\textbf{-0.946(0.010)}&\textbf{-0.646(0.291)}&\textbf{-0.978(0.012)}&\textbf{-0.817(0.032)}&0.283(0.261)\\
    \bottomrule
    \end{tabular}
    }
\end{table}

\textbf{Additional RL-specific experiments.} 
We investigate the effects of two RL parameters, the memory size and exploration ratio on the contributions in \cref{fig:RL-additional} for Breakout (more results in \cref{appendix:experiments}).
\cref{fig:RL-additional} (left) shows the agent with smallest memory size has the highest contribution as using a large memory size hides the critical transitions among the redundant trivial transitions and leads to inefficient learning \cite{zhang2018deeper}.
\cref{fig:RL-additional} (right) shows the agent with a moderated exploration ratio 
has the highest contribution. The agent with no exploration (blue) has gradually higher contribution because it starts contributing more by exploitation \textit{after} the environment has been sufficiently explored by others collectively.

\begin{figure}[!ht]
    \centering
    \includegraphics[width=0.49\linewidth]{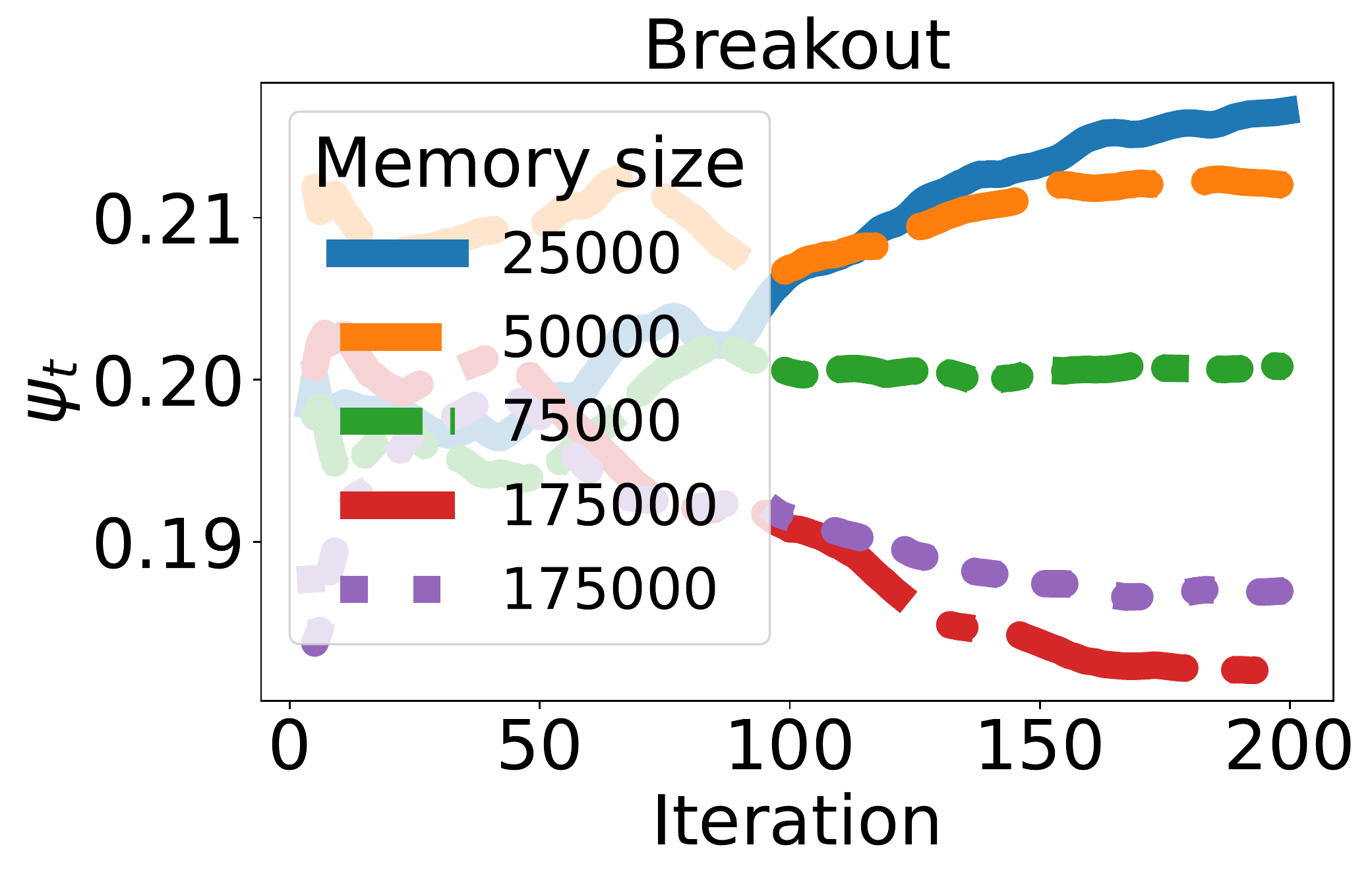}
    \includegraphics[width=0.49\linewidth]{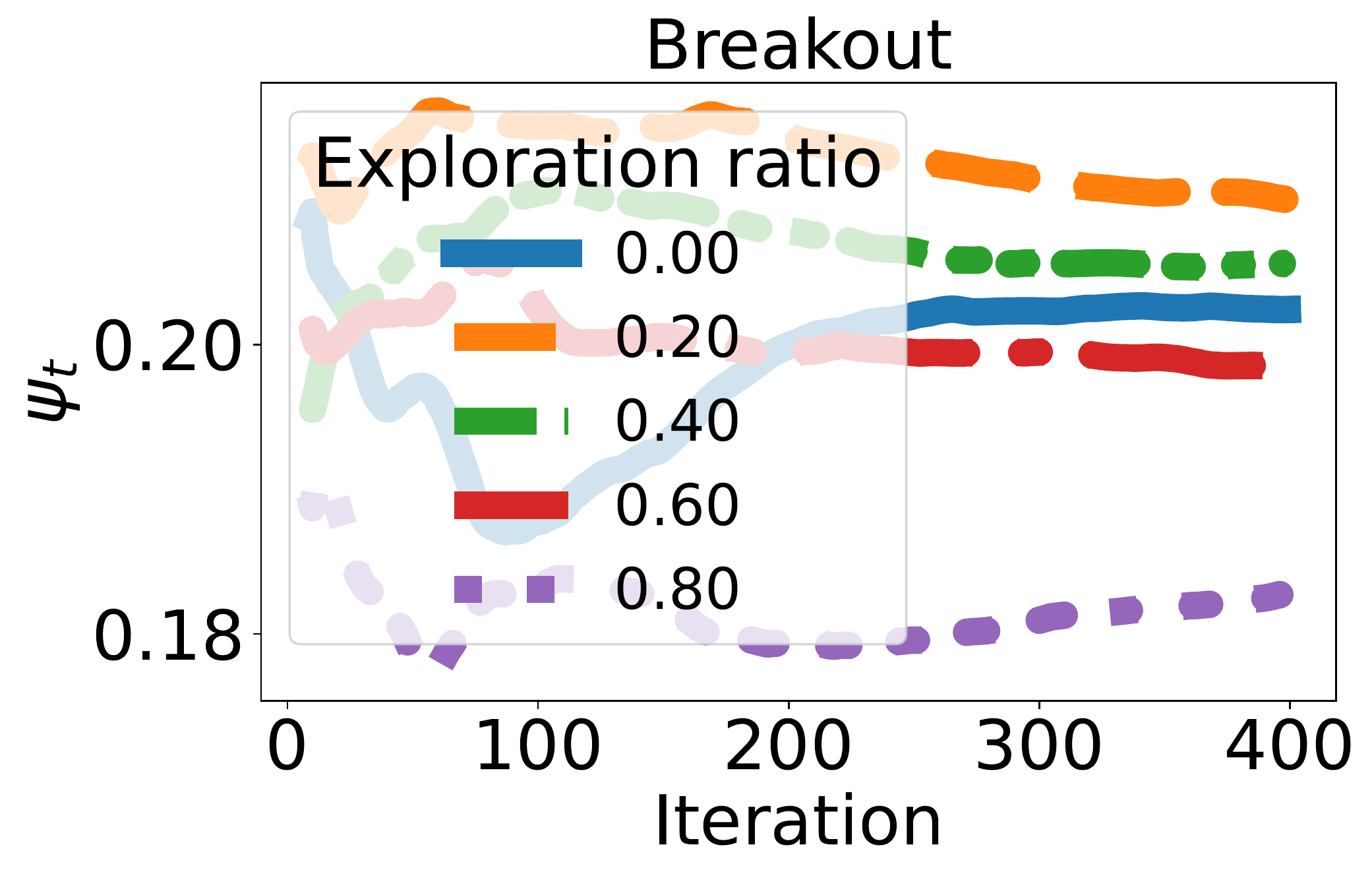}
    \captionof{figure}{Left (right): Contribution evaluation result $\boldsymbol{\psi}_t$ of nodes with different memory size (exploration ratio).
    }
    \label{fig:RL-additional}
\end{figure}

\section{Related Work}\label{sec:related-work}
Designing fair incentives is growing as an important research direction in collaborative learning, and FL in particular \citep{Chen2020-FL-incentive,Cong2020-FL-incentive,Yu-et-al:2020AIES}, through the means of external resources such as monetary incentives \citep{Richardson2020-FL-incentive,Zhang2021-FL-reverse-auction} and inherent resources such as model updates \citep{Nagalapatti2021-game-of-gradients-fl,Xu2021-fair-CML}. 
However, few works have specifically considered the accuracy in the contribution estimates based on which the fair incentives are designed. Importantly, this can negatively affect the fairness of said incentives (shown in \cref{appendix:experiments}). We address this by leveraging the classic explore-exploit paradigm,\footnote{The explore-exploit paradigm is suitable for problems where an accurate modeling (e.g., estimation) is first established and subsequently used in decision making  \citep{Krause2007-explore-vs-exploit}.} which has not been considered for fair incentive design in FL. Some other works have focused on analyzing incentives-aware collaboration by examining the equilibrium~\citep{blum2021one} and stable formation of coalitions~\citep{donahue2021model} in FL. However, these works do not consider the design of the incentive mechanism to achieve fairness.

As highlighted by existing works \citep{Li2021_TERM,Li2021_ditto,li2019-qFFL,Mohri2019_afl}, equality is also important in FL.\footnote{Note while \citep{Li2021_ditto,li2019-qFFL} use the keyword fairness, it is formalized via equality in performance across nodes, and does \emph{not} refer to the fairness in incentive design.} Our investigation also confirms that an undesirable ``rich get richer'' phenomenon (i.e., inequality) can arise from the fairness guarantee. Hence, we introduce an equality-preserving perspective in our fair incentive design. \cref{sec:experiments} empirically compares our proposed method against \citep{li2019-qFFL} which generalizes/is representative of \citep{Mohri2019_afl,Li2021_TERM,Li2021_ditto}.

We highlight our setting assumes \emph{honest} nodes (e.g., they do not strategize) which is commonly assumed in current works~\citep{Sim2020, Xu2021-fair-CML, Seb2022-incentivizing}. This assumption is also supported by quite a few application scenarios. For example, nodes can be hospitals~\citep{Flores2020_nvidia_FL_medical, Leersum2013_medicine_weekly} or regulated financial institutions~\citep{Miller1996_finance_monthly}. \citet{Yuan2021_incentive_FL_long_term,Zhang2022_FL_long_term} relax this assumption while they do not consider fairness. It is an interesting future direction to relax the honest node assumption while guaranteeing fairness.

\section{Discussion and Future Work}
We propose a novel framework to guarantee fair incentives for the nodes in federated learning (to incentivize the more resourceful nodes) while preserving asymptotic equality (to encourage the less resourceful nodes). Interestingly, using a single equalizing coefficient $\beta$, our framework sheds some light on the intricate relationship between fairness and equality in collaborations with finite resources (e.g., the number of nodes selected to synchronize is finite). Achieving absolute equality violates fairness and reduces the effectiveness of incentives. On the other hand, enforcing a larger improvement in incentives due to some increase in contributions can guarantee fairness but potentially creates/worsens inequality and thus discourages the less resourceful nodes from collaborating. We describe how to find a suitable range for $\beta$ and empirically verify its effectiveness in the fairness vs.~equality trade-off.

For future work, it is interesting to explore whether such fairness-equality relationship arises in other collaboration paradigms such as collaborative supervised learning \citep{Sim2020,Phong2022_CML}, unsupervised learning \citep{Seb2022-incentivizing}, parametric learning \citep{agussurja2022_bayesian_para}, (personalized) model fusion \citep{Lam2021ModelFusion, Hoang2021_AID}, active learning \citep{Xu2023ActiveLearning}, reinforcement learning \citep{fan2021faulttolerant} or causal inference \citep{Qiao2023}. Regarding fairness, it also is interesting to explore whether or precisely how a larger number of nodes affects the fairness via the Shapley value \citep{zhou2023}. Moreover, as we adopt the gradient alignment (via the inner product of the gradient vectors) to determine the contributions of the nodes, it is also interesting to investigate the effectiveness of other data valuation methods \citep{sim2022_ijcai} such as \citep{data-shapley-Ghorbani2019, Wu2022DAVINZDV} when a validation dataset is available or \citep{Xu2021ValidationFA} specifically for regression tasks.

\section*{Acknowledgements}
This research/project is supported by the National Research Foundation Singapore and DSO National Laboratories under the AI Singapore Programme (AISG Award No: AISG$2$-RP-$2020$-$018$).
Xinyi Xu is supported by the Institute for Infocomm Research of Agency for Science, Technology and Research (A*STAR).

\bibliography{biblio}
\bibliographystyle{icml2023}

\newpage
\onecolumn
\appendix
\section{Algorithm and Overview} \label{appendix:algorithm}
\begin{table}[H]
\centering
\caption{Specific notations}

\begin{tabular}{|l|l|}
\hline
 Notation &  Meaning  \\ \hline
 $t$ & Iteration index\\
 $N$ & Number of nodes  \\ 
 $\theta_t$ & Global model parameter in iteration $t$  \\
 $\theta_{i,t}$ & Local model parameter for node $i$ in iteration $t$ \\
 $\mathcal{D}_i$ & Local dataset for node $i$ \\
 $\mathbf{L}(\theta;\mathcal{D}_i)$ & Loss function with parameter $\theta$ on dataset $D_i$ \\
 $\mathbf{J}(\theta)$ & Global loss function \\
 $p_i$ & Data size weighted coefficient for node $i$ \\
 $k$ & Number of nodes selected in each iteration \\
 $s_{i,t}$ & Mini-batch data from node $i$ in iteration $t$ when using Stochastic Gradient Decent \\
 $s_t$ & Aggregated mini-batch data from all nodes in iteration $t$: $s_t \coloneqq \bigcup_{i=1}^NS_{i,t}$. \\
$\Delta \theta_{i,t}$ & Gradient of model parameter computed with mini-batch data $s_{i,t}$\\
$\Delta \theta_t$ & Aggregated gradient computed by coordinator in iteration $t$ \\
$[N]$ & Grand coalition formed by all nodes $\{i\}_{1,\dots, N}$\\
$\mathcal{S}$ & Coalitions of nodes $\mathcal{S} \subseteq [N]$ \\
$\mathbf{U}(\mathcal{S})$ & Utility function that takes the coalition $\mathcal{S}$ as inpute \\
$\phi_{i,t}$ & Shapley value for node $i$ in iteration $t$ \\
$\boldsymbol{\phi}_{t}$ & Vector of Shapley values for all nodes in iteration $t$: $\boldsymbol{\phi}_t \coloneqq \{\phi_{i,t};i\in [N]\}$ \\
$\psi_{i,t}$ & Contribution estimate for node $i$ up to iteration $t$\\
$\boldsymbol{\psi}_t$ & Vector of contribution estimates for all nodes up to iteration $t$: $\boldsymbol{\psi}_t \coloneqq \{\psi_{i,t};i\in [N]\}$ \\
$\psi_i^*$ & True contribution for node $i$: $\psi_i^* = \lim_{t \rightarrow \infty}\psi_{i,t}$ \\
$\boldsymbol{\psi}^*$ & Vector of true contributions: $\boldsymbol{\psi}^* = \{\psi_i^*; i\in [N]\}$ \\
$\beta$ & Equalizing coefficient (a.k.a. the temperature parameter) \\
$\varrho_i$ & Probability of node $i$ been selected in exploitation phase \\
$\Gamma_i$ & Expected staleness for node $i$ in exploitation phase\\
$C_i$ & Convergence complexity for node $i$ \\
$\alpha$ & Significance level of hypothesis testing for stopping criterion\\
$\tau$ & Windows size for hypothesis testing \\
$T_{\alpha}$ & Stopping iteration for exploration phase with significance level of $\alpha$. \\
$P_{i,online}$ & Online performance for node $i$ \\
$\Bar{\gamma_i}$ & Average staleness in the experiment for node $i$.\\
$\zeta_i$ & Level of noise/missing values/negative quantity for node $i$ in experiment. \\
$\boldsymbol{\zeta}$ & Vector of level of noise/missing values/negative quantity in experiment: $\boldsymbol{\zeta} \coloneqq \{\zeta_i;i\in [N]\}$. \\
 \hline
\end{tabular}
\label{table:specific_notations}
\end{table} 
\cref{alg:overview} outlines our framework where lines $1-3$ correspond to contribution evaluation (explore) in \cref{sec:motivation} and line $4$ corresponds to incentive realization (exploit) in \cref{sec:stage2}. 
Our proposed algorithm first performs contribution evaluation until $\boldsymbol{\psi}_t$ converge, after which uses $\boldsymbol{\psi}_t$ to design a sampling distribution as in \cref{equ:sampling-probs} and follows it in the remaining training for realizing the incentives.

\begin{algorithm}[!ht]
\caption{Framework Overview}
\label{alg:overview}
\begin{algorithmic}[1]
\STATE \textit{Contribution Evaluation}:
\WHILE{$\boldsymbol{\psi}_{t}$ not converged via \cref{prop:hypothesis-testing}}
\STATE Obtain $\boldsymbol{\psi}_{t+1}$ as in \cref{equ:cumulative-sv-psi}
\ENDWHILE
\STATE Perform \textit{Incentive Realization} via \cref{equ:sampling-probs} w.r.t.~$\boldsymbol{\psi}_{t}$
\end{algorithmic}
\end{algorithm}

\section{Additional Details on Experiment Settings} \label{appendix:setings-experiments}

\subsection{Additional Description on Resourcefulness of Nodes}
We use the term resourcefulness to describe a node's data collection capability, in terms of both quantity and quality of data. For example, a more resourceful node can utilize more resources to clean its data, fix missing values or features, or collect more data. These are all explicitly considered in our experimental settings in FOIL (\cref{sec:experiments}). Resourcefulness is a conceptual description of the node's data (and thus its contribution in FL) and it is made concrete in our implementation via the above specific settings. Therefore, for simplicity, we treat the contributions of the nodes, which are observable in FL, to be an indirect surrogate to the resourcefulness of the nodes, which is \emph{not} observable. With this, a node $i$ that chooses to expend a considerable amount of resources (to collect, clean and preprocess the data) and contribute to the collaboration (by training and uploading model updates on the high-quality data) will be recognized via high $\varphi_i$ and rewarded with a better convergence complexity. On the other hand, if a node does not wish to contribute at all, it can mean the node does not join the collaboration in the first place. However, if a node does want to collaborate but is unfortunately not very resourceful (e.g., on a low budget), then it is important (for equality) that the collaboration does not magnify the effect of this node's less resourcefulness (or widen the inequality gap between the resourceful and less resourceful), hence our equality-preserving perspective.

\subsection{Additional Hyper-parameters for Federated Online Incremental Learning}

The model architectures for the datasets are as follows: (a) CNN with $2$ convolution layers followed by $2$ fully connected (FC) layers, each convolutional layer is followed by a max-pooling layer for MNIST. (b)
CNN with $2$ convolution layers and $3$ FC layers, each convolutional layer is followed by a max-pooling layer for CIFAR-10 and PATH.
(c) Multi-layer perception (MLP) with $3$ FC layers for HFT.
(d) Recurrent neuron network with hidden size of $40$ for ELECTRICITY.

For framework-dependent hyper-parameters, $q=0.1$ in qFFL and the normalization coefficient $\Gamma= 0.01$ in FGFL and the altruism degree $\beta_{\text{altruism}}=2$. Other framework-independent hyper-parameters are as follows.
Except for ELECTRICITY (regression) which uses mean squred error, the other datasets use cross entropy for the loss function. Except MNIST and PATH with a learning rate of $2e^{-3}$, the other datasets use a consistent learning rate of $2e^{-4}$. The size of latest available data $|s_{i,t}|$ for node $i$ in iteration $t$ is $3,6,14,4$ and $7$ for MNIST, CIFAR-10, HFT, ELECTRICITY, and PATH, respectively. All models use SGD as the optimization algorithm and use the batch size is $|s_{i,t}|$.

All the experiments have been run on a server with Intel(R) Xeon(R) Gold 6226R CPU @ $2.90$GHz processor, $256$GB RAM and $4$ NVIDIA GeForce RTX 3080's.

\subsection{Additional Hyper-parameters for Federated Reinforcement Learning} 
The Deep-Q-network (DQN) has $3$ convolutional layers with $[32,64,64]$ number of filters and $[8,4,3]$ for filter size $[4,2,1]$ for stride, followed by an FC layer with $512$ units. We use the \textit{rectified linear unit}~(ReLU) activation function for all layers, and use the initialization method from~\cite{He_2015_ICCV} for the convolutional parameters. The input to the DQN is a $84\times84\times4$ images from the state. 
We follow most of the hyper-parameters from~\cite{mnih2013playing} except a smaller learning rate $2e^{-5}$, as we empirically observe even a marginally larger rate can lead to very ineffective learning and we set a reduced memory size to $10^5$ for each agent due to RAM limitation. 
We use the clipped reward as $\{-1,0,1\}$ and set the reward decay as $\gamma=0.99$. The exploration ratio start from $1$ and linearly decay to $0.1$ after a total of $850$K local training steps. The batch size is $32$. 
The local models synchronize with the coordinator model every $2000$ local steps in the environment, i.e., $2000$ steps in the environment corresponds to an iteration in FL.

\clearpage
\section{Additional Experimental Results} \label{appendix:experiments}

\subsection{Dataset License}
MNIST~\citep{mnist}: Attribution-Share Alike 3.0 License; CIFAR-10~\citep{cifar10}: MIT License; HFT~\citep{hft2018}: Creative Commons Attribution 4.0 International (CC BY 4.0); ELECTRICITY~\citep{electricity2020}: Creative Commons Attribution 4.0 International (CC BY 4.0); PATH~\citep{medmnistv1}: Creative Commons Attribution 4.0 International (CC BY 4.0).

\subsection{Additional results for different settings of the running example} 
Regarding the experiment of using recall fraction to indirectly gauge the accuracy in the contribution estimates, we include additional results with less quantity/missing values/feature noise data to simulate the low-quality data for designated nodes and on CIFAR-10.

For less quantity data, we randomly drop 10\% of the data points in the designated nodes to simulate nodes with less quantity data. For missing values data, we randomly set 50\% of the pixel of images to zero in the designated nodes to simulate that the data have random unfilled features which is common in data collection. For feature noise data, we add standard Gaussian noise $\epsilon \sim \mathcal{N}(0,1)$ to each feature/pixel of the images in the designated nodes. The proportion of data in the designated nodes with missing values/feature noise is the same with label noise.

\begin{figure}[!ht]
    \centering 
    \includegraphics[width=0.990\textwidth]{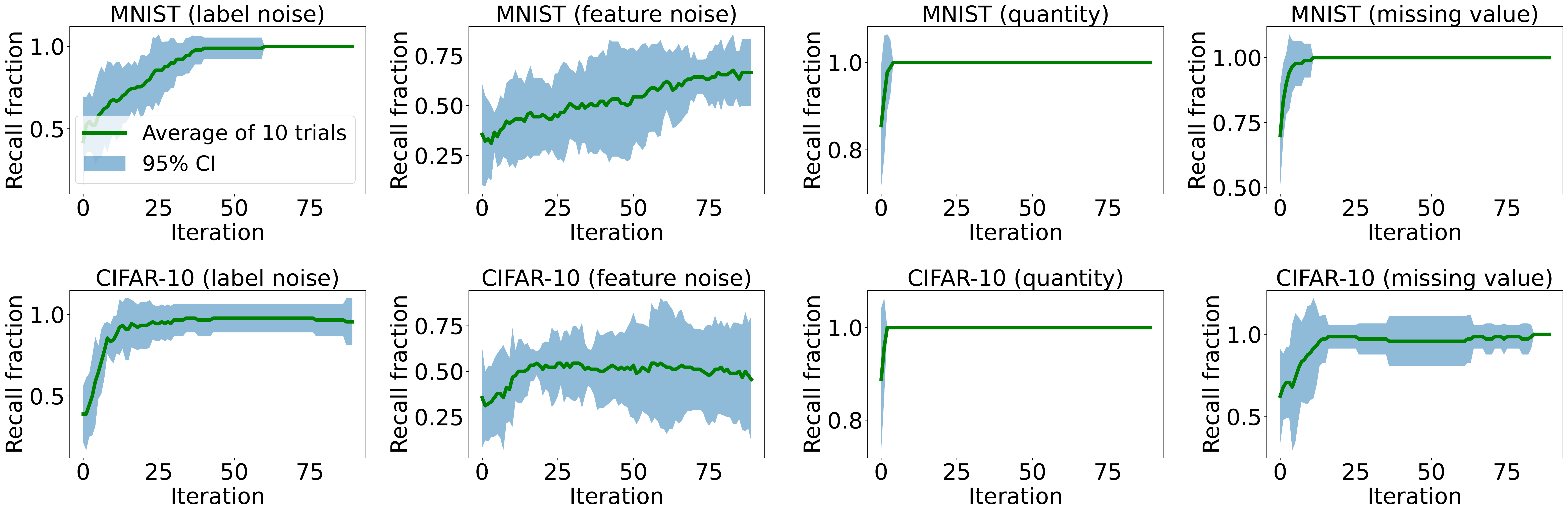}
    \caption{Recall fraction vs.~iteration under different settings.}
    \label{fig:motivation-accurate-sv-appendix}
\end{figure}

\subsection{Empirical Validation of Sub-sampling for Hypothesis Testing}

We empirically validate the effectiveness of using sub-sampling to side-step the technical requirement of Hotelling's $T^2$ distribution of $\tau > N$ in \cref{fig:empirical-model-sv-convergence}. In these various settings, the $p$-values obtained with/without sub-sampling are well aligned so that we can use the sub-sampled set of nodes instead of all $N$ nodes for the stopping criterion.

\begin{figure}[!ht]
    \centering 
    \includegraphics[width=0.99\textwidth]{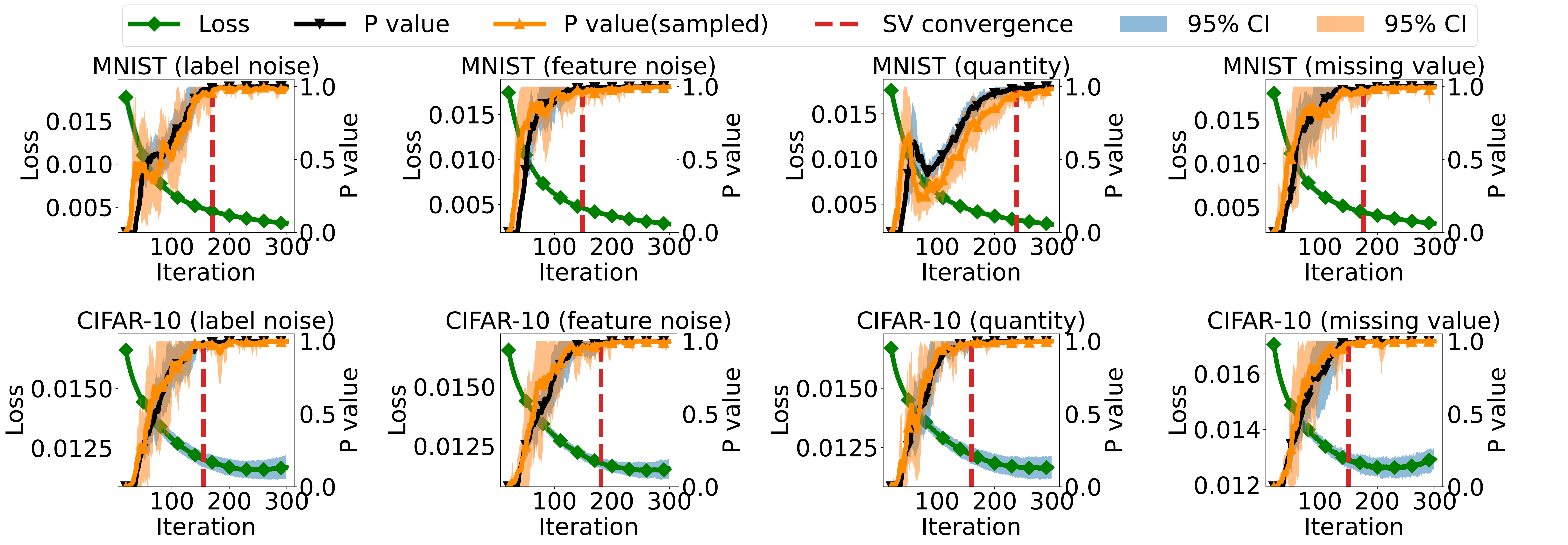}
    \caption{
    The $p$-value with~(orange) and without~(black) sub-sampling using $\tau = 20$~($35$).
    The p-value of original sampling (black) and sub-sampling (orange) vs.~iteration. The shaded area denotes the $95\%$ confidence interval computed from 10 random trials. The results show $\boldsymbol{\psi}_t$ converges (vertical red dashed line, which indicates $p$-value first reaches $0.99$).
    \label{fig:empirical-model-sv-convergence}
    }
\end{figure}

\subsection{Empirical Validation of Contribution Evaluation for Nodes with Unstable Connection.} \label{appendix:exp-unstale-connection}
In the setting with unstable connection, it is impractical for all the nodes to participates in every iteration for the contribution evaluation. Though our framework does not target to this setting, a simple modification can be adopted to make it work. We can simply let $\phi_{i,t} = 0$ if node $i$ does not participate in iteration $t$ when evaluating the contribution of all nodes. And we keep the rest part of our framework the same. Intuitively, if not all the nodes can participate in the contribution evaluation, it would take more iterations for the $\boldsymbol{\psi}_{t}$ to converge. The result in \cref{table:subsample-nodes-stopping-iteration} verifies the intuition, the stopping iteration for contribution evaluation increases if a smaller proportion of nodes participate in the contribution evaluation.

The experiment setting is the same as the setting in \cref{sec:foil}. The $p$ value for the stopping criterion is set to be $0.95$, and we have 10 nodes in the collaboration. We begin to sub-sample $r_{\text{sub}}$ proportion of the nodes to participate the contribution evaluation after iteration $T_{\alpha}-5$ to simulate the case of unstable connection, where $T_{\alpha}$ is the stopping iteration for the full participation. 

\begin{table}[ht]
\setlength{\tabcolsep}{4pt}
\caption{The unstable connection and its effect to the stopping iteration: the stopping iteration (standard error under 3 runs) under different subsampling ratio $r_{\text{sub}}$.}
\label{table:subsample-nodes-stopping-iteration}
\begin{center}
\resizebox{0.60\linewidth}{!}{\begin{tabular}{c|cc|cc}
\multicolumn{1}{c}{}  & \multicolumn{2}{c}{\bf Label Noise} & \multicolumn{2}{c}{\bf Feature Noise}\\
 \toprule
$r_{\text{sub}}$ & MNIST & CIFAR-10 & MNIST & CIFAR-10 \\
 \midrule
0.2&103.00(1.4e+01)&154.67(1.3e+01)&101.67(1.1e+01)&137.67(1.2e+01)\\
0.8&102.67(5.2e+00)&146.00(7.5e+00)&105.33(8.9e+00)&136.00(1.4e+01)\\
0.9&99.00(6.4e+00)&143.67(7.3e+00)&104.00(9.0e+00)&134.33(1.2e+01)\\
1.0&77.00(8.9e+00)&138.00(1.2e+01)&81.67(7.9e+00)&131.33(1.4e+01)\\
\bottomrule 
\end{tabular}
}
\end{center}
\end{table}

\subsection{Additional comparisons of communication complexity and running time among all baselines}\label{appendix:communication-running-time}

Denote $T$ as the overall training iterations, $n_g$ as the number of dimensions of each gradient (i.e., the number of parameters in the model $\theta$) , and $r$ as the selection ratio of nodes in each iteration. Denote $T_1$ as number of iterations for exploration phase and $T_2$ for the exploitation phase. Thus $T_1+T_2 = T$.

\begin{table}[!ht]
\setlength{\tabcolsep}{4pt}
\caption{The communication factors for different baselines. Where $\text{Communication costs} = \text{Accumulated communications}\times \text{Cost per communication}.$}
\label{table:communication-costs}
\begin{center}
\begin{tabular}{c|c|c|c|c}
 \toprule 
Baseline  & Accumulated communications &  Cost per communication & Communication costs & Comparison with ours\\
\midrule 
FedAvg & $r\cdot T \cdot N$ & $2n_g$ & $2n_grTN$& Lower than ours \\
qFFL & $r\cdot T \cdot N$ & $2n_g+1$ & $2n_grTN+rTN$&Higher than ours iff $T_1 < \frac{2n_g}{rT}$ \\
FGFL & $T \cdot N$ & $2n_g$ & $2n_gTN$& Higher than ours \\
GoG & $r\cdot T \cdot N$ & $2n_g$ & $2n_grTN$& Lower than ours\\
Ours & $T_1 \cdot N + r\cdot T_2 \cdot N$ & $2n_g$ & $2n_g(T_1 + rT_2)N$ & N.A. \\
\bottomrule 
\end{tabular}
\end{center}
\end{table}

\textbf{Comparisons of communication costs.}
In \cref{table:communication-costs}, $\text{Accumulated communications}$ denotes the total number of communications between a node and the coordinator and $\text{Cost per communication}$ denotes the cost per such communication. Take FedAvg for example, in one iteration, $rN$ nodes are picked (hence $rN$ communications), so the $ \text{Accumulated communications}$ is $rNT$ for a total of $T$ iterations. In each communication, the node uploads and downloads the gradient containing $n_g$ parameters, so the $\text{Cost per communication}$ is $2n_g$.

Note that the $\text{Cost per communication}$ and $\text{Communication costs}$ denote how many floating point numbers are required. The actual cost in practice additionally depends on how many bits each floating point number requires in storage which incurs a constant linear factor and is omitted for simplicity.

From \cref{table:communication-costs}, we have several observations. When $r < 1$, GoG and FedAvg have the lowest communication costs among all baselines. For Ours, the shorter the exploration phase $T_1$ is, the lower communication cost. However, a shorter exploration phase might cause inaccurate contribution estimates as shown in \cref{fig:motivation-accurate-sv} (left). When $T_1 < 0.2T$ which is observed in our experiments, in the worst-case (i.e.,  $r \rightarrow 0$), our method has an additional communication cost of $0.4n_gTN$ more than GoG (which has the lowest communication cost). Overall, the communication costs do not vary too much among all baselines.

\begin{table}[!ht]
\setlength{\tabcolsep}{4pt}
\caption{Time complexity of fairness mechanism for different baselines. FedAvg is omitted since it does not have a fairness mechanism. $M$ is the number of Monte Carlo simulations in GoG and $n_g |D_{\text{val}}|$ is from the forward passes of a model with $n_g$ parameters on the validation dataset $D_{\text{val}}$.}
\label{table:time-complexity-factor}
\begin{center}
\begin{tabular}{c|c}
 \toprule 
Baseline  & Time complexity\\
\midrule 
qFFL & $O(rn_gNT)$  \\
FGFL & $O(n_gNT)$ \\
GoG & $O(rMNn_g|D_\text{val}|T)$ \\
Ours & $O(N^2 n_g T_1)$ \\
\bottomrule 
\end{tabular}
\end{center}
\end{table}

\begin{table}[!ht]
\setlength{\tabcolsep}{4pt}
\caption{Running time of fairness mechanism in seconds and the fraction of time spent on fairness mechanism w.r.t. the overall training time (in brackets) under the setting of \textbf{label noise}. Results are averaged over 5 runs.}
\label{table:table:running-time-label-noise}
\begin{center}
\begin{tabular}{c|c|c|c|c|c}
 \toprule 
Baseline  & MNIST & CIFAR-10 & HFT & ELECTRICITY & PATH \\
\midrule 
qFFL&7.778(1.2e-02)&24.064(4.0e-02)&4.565(3.1e-02)&20.188(4.5e-02)&28.289(4.5e-02)\\
FGFL&391.282(5.0e-01)&401.867(5.5e-01)&15.727(1.0e-01)&230.978(3.7e-01)&458.683(4.6e-01)\\
GoG&1656.071(9.0e-01)&1311.791(8.7e-01)&1640.153(9.3e-01)&1378.404(9.2e-01)&1400.977(8.8e-01)\\
Ours&444.587(3.9e-01)&400.780(4.1e-01)&7.598(5.9e-02)&237.968(3.7e-01)&241.410(2.8e-01)\\
\bottomrule 
\end{tabular}
\end{center}
\end{table}

\begin{table}[!ht]
\setlength{\tabcolsep}{4pt}
\caption{Running time of fairness mechanism in seconds and the fraction of time spend on fairness mechanism w.r.t. the overall training time (in brackets) under the setting of \textbf{quantity}. Results are averaged over 5 runs.}
\label{table:running-time-quantity}
\begin{center}
\begin{tabular}{c|c|c|c|c|c}
 \toprule 
Baseline  & MNIST & CIFAR-10 & HFT & ELECTRICITY & PATH \\
\midrule 
qFFL&4.346(8.4e-03)&12.173(2.1e-02)&7.768(3.5e-02)&10.829(2.6e-02)&8.793(1.9e-02)\\
FGFL&276.947(5.1e-01)&278.551(5.7e-01)&14.986(5.9e-02)&212.794(4.6e-01)&316.002(5.6e-01)\\
GoG&1237.899(8.3e-01)&835.653(7.7e-01)&2839.274(9.3e-01)&873.078(8.5e-01)&903.396(7.9e-01)\\
Ours&1099.390(5.7e-01)&999.485(5.9e-01)&7.179(3.3e-02)&879.026(6.2e-01)&741.787(5.3e-01)\\
\bottomrule 
\end{tabular}
\end{center}
\end{table}

\textbf{Comparisons of running time.} From \cref{table:time-complexity-factor}, we have some observations. qFFL seems the most efficient since its fairness mechanism is a re-weighting of the objective function using the reported losses from $rN$ nodes and the dependence on $n_g$ is due to calculating the norm of each gradient. The comparison between FGFL and ours depends on $T$ vs. $NT_1$: if $T_1$ is small (contribution estimates converge quickly), specifically $T_1 < T/N$, then ours is more efficient. For GoG, the time complexity for it depends on the number of Monte Carlo simulations $M$ that is often set according to the number of nodes $N$ For example, \cite{Nagalapatti2021-game-of-gradients-fl} sets $M=10$ for $N=5$ and in our experiments we set $M=30$ for $N=10$. Furthermore, GoG has an extra factor of $O(|D_\text{val}|)$ due to the usage of an extra validation dataset. Therefore when $M=30$ and $N=10$ which we used in our experiments later, GoG has the highest time complexity among all baselines. 

To verify these observations, we perform experiments to compare the running time for the fairness mechanism and its proportion in the whole training process. We consider two settings: label noise and quantity. We consider $N=10$ nodes and the total training iteration of $T=250$ for all datasets and baseline. The fraction of nodes selected in each iteration for qFFL, GoG, Ours is set to be $r=0.4$. We note that FedAvg is excluded as it does not have an explicit fairness mechanism.

Based on \cref{table:table:running-time-label-noise} and \cref{table:running-time-quantity}, we have some empirical observations. Since qFFL does not involve a contribution estimation, its running time is always the lowest in \cref{table:table:running-time-label-noise} and \cref{table:running-time-quantity} except on dataset HFT in \cref{table:running-time-quantity} where Ours runs the fastest. Previously in \cref{table:time-complexity-factor}, if $T_1 < T/N$, the time complexity of the fairness mechanism of Ours is lower than that of the FGFL, which is empirically observed in \cref{table:table:running-time-label-noise} on CIFAR-10, HFT and PATH. Otherwise, FGFL has lower complexity as shown in \cref{table:running-time-quantity} on all datasets except HFT. Due to the usuage of validation dataset and $M$'s dependence on the number of nodes $N$, GoG almost always has the highest running time of fairness mechanism as shown in \cref{table:table:running-time-label-noise} and \cref{table:running-time-quantity}. From \cref{table:time-complexity-factor}, qFFL is $1/r$ times faster than FGFL ($r=0.4$ in our experiments which means qFFL is $2.5$ times faster than FGFL from the expressions of the factor). However, from \cref{table:table:running-time-label-noise} and \cref{table:running-time-quantity}, qFFL is around 20 times faster than FGFL in most experiments. The reason is that the constant facor in $O(\cdot)$ for FGFL is larger than that for qFFL due to the extra operations in FGFL (e.g., flattening and unflattening gradients for calculating cosine-similarity, sparsifications of gradients). These operations add to the constant factor but their exact runtimes w.r.t. $n_g$ depends on the implementation library and is not known.

As an additional remark to the analyses and experiments above, we highlight that since our targeted scenario is the long-term FL, the true bottleneck of time in practice would be the time for data collection/preprocessing instead of the running time of our fairness mechanism.

\subsection{Additional experiments to specifically investigate the effect of heterogeneity on model/predictive performance.}
Our additional experiments below demonstrate that our method outperforms (compares favorably to) existing methods in terms of predictive performance when there is heterogeneity, and the performance indeed degrades with the degree of heterogeneity but the degradation is graceful.

To investigate how the degree of heterogeneity affects the model performance, we perform experiments for classification tasks on $N=30$
nodes; the detailed experiment setting can be found in Sec.~\ref{sec:foil}. The difference here is that we vary the degree of heterogeneity of nodes' local data distributions by assigning them different numbers of classes. For example, in the most homogeneous setting, each node has an equal amount of data uniformly randomly sampled from all $10$ classes. While in the most heterogenous setting, each node only has $2$ classes (uniformly randomly sampled from $10$ classes) of data points. Under this setting, few nodes have the same data from one specific class, so we can quantify the heterogeneity by the number of classes each node has. The lower the number of classes in each node, the higher the heterogeneity of the nodes' local data. Note that HFT and ELECTRICITY datasets are exempt from the experiments here since HFT is binary classification and ELECTRICITY is a regression task to which the setting does not apply. Note that PATH only has 9 classes in total, so we set the number of classes for each corresponding setting to $\text{round}(9 \times \frac{\text{Number of classes}}{10})$.

\begin{table}[!ht]
\setlength{\tabcolsep}{4pt}
\caption{The maximum of the \textbf{final model accuracy} among 30 nodes (standard error over 5 runs) for our approach under different levels of heterogeneous local data distribution. Number of classes indicates the level of heterogeneity among nodes: a smaller number of classes corresponds to a higher degree of heterogeneity.}
\label{table:heterogeneity-final-acc}
\begin{center}
\begin{tabular}{c|c|c|c}
 \toprule 
Number of classes  & MNIST & CIFAR-10 & PATH\\
\midrule 
10&0.754(0.029)&0.220(0.012)&0.392(0.014)\\
8&0.723(0.017)&0.218(0.021)&0.393(0.009)\\
6&0.741(0.020)&0.168(0.007)&0.358(0.012)\\
4&0.702(0.022)&0.183(0.014)&0.335(0.005)\\
2&0.688(0.047)&0.149(0.015)&0.188(0.014)\\
\bottomrule 
\end{tabular}
\end{center}
\end{table}

\begin{table}[!ht]
\setlength{\tabcolsep}{4pt}
\caption{The maximum of the \textbf{online model accuracy} among 30 nodes (standard error over 5 runs) for our approach under different levels of heterogeneous local data distribution. Number of classes indicates the level of heterogeneity among nodes: a smaller number of classes corresponds to a higher degree of heterogeneity.}
\label{table:heterogeneity-online-acc}
\begin{center}
\begin{tabular}{c|c|c|c}
 \toprule 
Number of classes  & MNIST & CIFAR-10 & PATH\\
\midrule 
10&0.635(0.017)&0.181(0.009)&0.328(0.006)\\
8&0.624(0.015)&0.183(0.009)&0.340(0.006)\\
6&0.639(0.010)&0.146(0.006)&0.312(0.010)\\
4&0.624(0.014)&0.154(0.014)&0.290(0.007)\\
2&0.627(0.018)&0.139(0.007)&0.168(0.013)\\
\bottomrule 
\end{tabular}
\end{center}
\end{table}

From Table~\ref{table:heterogeneity-final-acc} and Table~\ref{table:heterogeneity-online-acc}, heterogeneity does negatively affect the model performance. Specifically, when the heterogeneity increases (as the number of classes goes from 10 to 2), the model performance will degrade. 

\begin{table}[!ht]
\setlength{\tabcolsep}{4pt}
\caption{The maximum of the \textbf{online model accuracy} among 30 nodes under the heterogenous data setting (how much the performance degrades compared to the homogenous data setting) for \textbf{different FL algorithms}.}
\label{table:heterogeneity-baselines}
\begin{center}
\begin{tabular}{c|c|c|c}
 \toprule 
Number of classes  & MNIST & CIFAR-10 & PATH\\
\midrule 
FedAvg&0.498(0.002)&0.118(-0.034)&0.151(-0.141)\\
qFFL&0.117(0.015)&0.096(-0.003)&0.105(0.020)\\
FGFL&0.188(-0.293)&0.114(-0.014)&0.129(-0.007)\\
GoG&0.599(-0.028)&0.136(-0.052)&0.155(-0.150)\\
Ours&0.627(-0.008)&0.139(-0.042)&0.168(-0.161)\\
\bottomrule 
\end{tabular}
\end{center}
\end{table}

To compare how the different degrees of heterogeneity affect different federated learning baseline algorithms in our paper, we perform experiments to see how different algorithms perform under the same degree of heterogeneity and how much their corresponding performances degrade due to the heterogeneity of nodes' local data. The heterogeneous setting refers to the case when the number of classes = 2 (most heterogeneous) and the homogenous setting refer to the case when the number of classes = 10. From Table~\ref{table:heterogeneity-baselines}, qFFL and FGFL perform relatively poorly in highly heterogenous data settings. FedAvg performs better than qFFL and FGFL. Our approach and GoG perform significantly better than other approaches. Our approach achieves the best performance among all datasets. Under the heterogeneous setting, our approach has a similar degree of model performance degradation (shown in the brackets) as FedAvg and GoG. 

These additional experimental results demonstrate that compared to baselines (1) our method performs well under the heterogeneous data distributions compared to other baselines and (2) the performance of our method degrades gracefully as the degree of heterogeneity increases.

\subsection{Additional Fairness Comparison for FOIL and FRL}

\paragraph{Corresponding results for label noise, quantity and missing values.}
We provide the fairness results under the setting of label noise, quantity, and missing values corresponding in \cref{table:pearson-online-loss,table:average-online-acc,table:minimum-online-acc}.
\begin{table}[!ht]
\setlength{\tabcolsep}{4pt}
\caption{Correlation coefficient $\rho$ between $\zeta$ and \textbf{online loss} under the setting of label noise, quantity, and missing values. \textit{Higher} $\rho$ indicates better fairness result.
}
\label{table:pearson-online-loss}
\begin{center}
\resizebox{0.5\linewidth}{!}{
\begin{tabular}{c|ccccc}
\multicolumn{1}{c}{} &\multicolumn{5}{c}{\bf Label Noise}\\
\toprule
 & MNIST & CIFAR-10 & HFT & ELECTRICITY & PATH \\
 \midrule
FedAvg&-0.224(0.060)&-0.072(0.041)&0.113(0.067)&0.000(0.102)&0.024(0.104)\\
qFFL&0.097(0.088)&0.095(0.144)&0.218(0.091)&-0.162(0.088)&-0.171(0.128)\\
FGFL&0.593(0.056)&0.396(0.031)&-0.282(0.045)&\textbf{0.834(0.017)}&0.314(0.035)\\
GoG&0.119(0.034)&0.260(0.076)&0.091(0.034)&0.174(0.072)&-0.054(0.053)\\
\midrule 
Ours&\textbf{0.678(0.016)}&\textbf{0.455(0.059)}&\textbf{0.347(0.078)}&0.376(0.073)&\textbf{0.469(0.063)}\\
\bottomrule 
\end{tabular}
}
\resizebox{\linewidth}{!}{
\begin{tabular}{c|ccccc|ccccc}
\multicolumn{5}{c}{} \\
\multicolumn{1}{c}{}  &\multicolumn{5}{c}{\bf Quantity} &\multicolumn{5}{c}{\bf Missing Values} \\
\toprule 
 & MNIST & CIFAR-10 & HFT & ELECTRICITY & PATH & MNIST & CIFAR-10 & HFT & ELECTRICITY & PATH \\
 \midrule 
FedAvg&-0.094(0.055)&-0.111(0.019)&-0.020(0.065)&0.125(0.078)&-0.026(0.039)
&0.020(0.088)&-0.041(0.140)&-0.035(0.039)&-0.095(0.059)&0.086(0.085)\\
qFFL&0.338(0.118)&0.131(0.048)&-0.211(0.030)&-0.115(0.058)&0.076(0.077)
&0.167(0.081)&-0.012(0.063)&0.093(0.078)&-0.174(0.088)&-0.075(0.145)\\
FGFL&0.682(0.004)&0.445(0.083)&-0.071(0.073)&-0.317(0.063)&0.279(0.109)
&0.349(0.068)&0.347(0.070)&-0.084(0.110)&\textbf{0.933(0.012)}&\textbf{0.315(0.067)}\\
GoG&0.252(0.027)&0.127(0.098)&\textbf{0.079(0.028)}&\textbf{0.297(0.059)}&0.090(0.026)
&0.358(0.029)&0.198(0.037)&0.208(0.018)&0.262(0.020)&0.242(0.027)\\
\midrule 
Ours&\textbf{0.785(0.005)}&\textbf{0.666(0.040)}&0.050(0.070)&-0.030(0.083)&\textbf{0.580(0.022)}
&\textbf{0.526(0.019)}&\textbf{0.581(0.015)}&\textbf{0.252(0.029)}&0.223(0.059)&0.179(0.057)\\
\bottomrule 
\end{tabular}
}
\end{center}
\end{table}

\begin{table}[!ht]
\setlength{\tabcolsep}{4pt}
\caption{The average of the online accuracy~(standard error) over all nodes under the setting of label noise, quantity, and missing values. 
For ELECTRICITY, we measure mean absolute percentage error(MAPE), so lower is better.}
\label{table:average-online-acc}
\begin{center}
\resizebox{0.5\linewidth}{!}{
\begin{tabular}{c|ccccc}
\multicolumn{1}{c}{} &\multicolumn{5}{c}{\bf Label Noise}\\
\toprule
 & MNIST & CIFAR-10 & HFT & ELECTRICITY & PATH \\
 \midrule
FedAvg&0.509(0.013)&0.161(0.009)&0.472(0.049)&1.386(0.052)&0.307(0.008)\\
qFFL&0.112(0.006)&0.103(0.002)&0.374(0.091)&1.618(0.117)&0.109(0.013)\\
FGFL&0.545(0.008)&0.157(0.007)&0.544(0.009)&1.703(0.041)&0.188(0.010)\\
GoG&0.599(0.007)&0.193(0.005)&0.536(0.017)&1.664(0.079)&0.321(0.006)\\
Standalone&0.519(0.007)&0.158(0.005)&0.580(0.017)&1.661(0.079)&0.233(0.002)\\
\midrule
Ours&\textbf{0.664(0.003)}&\textbf{0.206(0.009)}&\textbf{0.566(0.014)}&\textbf{0.125(0.004)}&\textbf{0.380(0.008)}\\
\bottomrule 
\end{tabular}
}
\resizebox{\linewidth}{!}{
\begin{tabular}{c|ccccc|ccccc}
\multicolumn{5}{c}{} \\
\multicolumn{1}{c}{}  &\multicolumn{5}{c}{\bf Quantity} &\multicolumn{5}{c}{\bf Missing Values} \\
\toprule 
 & MNIST & CIFAR-10 & HFT & ELECTRICITY & PATH & MNIST & CIFAR-10 & HFT & ELECTRICITY & PATH \\
 \midrule 
FedAvg&0.510(0.010)&0.138(0.009)&0.462(0.031)&1.755(0.128)&0.299(0.007)
&0.467(0.011)&0.149(0.003)&0.461(0.047)&1.422(0.058)&0.292(0.006)\\
qFFL&0.114(0.010)&0.104(0.004)&0.350(0.071)&1.697(0.148)&0.129(0.007)
&0.112(0.021)&0.108(0.004)&0.317(0.065)&1.425(0.080)&0.099(0.006)\\
FGFL&0.579(0.011)&0.169(0.006)&0.524(0.023)&1.845(0.071)&0.185(0.007)
&0.582(0.005)&0.156(0.010)&0.556(0.016)&2.293(0.069)&0.184(0.018)\\
GoG&0.606(0.009)&0.185(0.004)&0.543(0.020)&1.978(0.072)&0.320(0.004)
&0.564(0.012)&0.185(0.003)&0.576(0.013)&1.728(0.086)&0.299(0.005)\\
Standalone&0.606(0.006)&0.159(0.002)&0.536(0.014)&1.829(0.056)&0.236(0.003)
&0.577(0.003)&0.148(0.005)&0.539(0.016)&1.805(0.055)&0.228(0.004)\\
\midrule
Ours&\textbf{0.652(0.009)}&\textbf{0.213(0.007)}&\textbf{0.571(0.012)}&\textbf{0.113(0.002)}&\textbf{0.352(0.005)}
&\textbf{0.634(0.003)}&\textbf{0.195(0.004)}&\textbf{0.569(0.017)}&\textbf{0.126(0.003)}&\textbf{0.329(0.004)}\\
\bottomrule 
\end{tabular}
}

\end{center}
\end{table}

\begin{table}[!ht]
\setlength{\tabcolsep}{4pt}
\caption{The minimum of the online accuracy~(standard error) over all nodes under the setting of label noise, quantity, and missing values. For ELECTRICITY, we measure mean absolute percentage error(MAPE), so lower is better.
}
\label{table:minimum-online-acc}
\begin{center}
\resizebox{0.5\linewidth}{!}{
\begin{tabular}{c|ccccc}
\multicolumn{1}{c}{} &\multicolumn{5}{c}{\bf Label Noise}\\
\toprule
 & MNIST & CIFAR-10 & HFT & ELECTRICITY & PATH \\
 \midrule
FedAvg&0.503(0.012)&0.160(0.009)&0.466(0.050)&1.381(0.053)&0.303(0.008)\\
qFFL&0.112(0.006)&0.103(0.002)&0.373(0.091)&1.618(0.117)&0.109(0.013)\\
FGFL&0.533(0.013)&0.155(0.007)&0.542(0.009)&1.533(0.037)&0.185(0.009)\\
GoG&0.578(0.008)&0.187(0.005)&0.531(0.018)&1.644(0.079)&0.306(0.005)\\
Standalone&0.396(0.009)&0.132(0.006)&0.567(0.020)&1.301(0.059)&0.182(0.004)\\
\midrule 
Ours&\textbf{0.661(0.004)}&\textbf{0.204(0.009)}&\textbf{0.564(0.015)}&\textbf{0.124(0.004)}&\textbf{0.376(0.007)}\\
\bottomrule 
\end{tabular}
}
\resizebox{\linewidth}{!}{
\begin{tabular}{c|ccccc|ccccc}
\multicolumn{5}{c}{} \\
\multicolumn{1}{c}{}  &\multicolumn{5}{c}{\bf Quantity} &\multicolumn{5}{c}{\bf Missing Values} \\
\toprule 
 & MNIST & CIFAR-10 & HFT & ELECTRICITY & PATH & MNIST & CIFAR-10 & HFT & ELECTRICITY & PATH \\
 \midrule 
FedAvg&0.502(0.010)&0.138(0.009)&0.459(0.031)&1.749(0.129)&0.295(0.007)
&0.462(0.011)&0.148(0.003)&0.456(0.047)&1.417(0.058)&0.288(0.006)\\
qFFL&0.113(0.010)&0.104(0.004)&0.350(0.071)&1.697(0.148)&0.129(0.007)
&0.112(0.021)&0.108(0.004)&0.317(0.065)&1.425(0.080)&0.099(0.006)\\
FGFL&0.551(0.011)&0.165(0.007)&0.522(0.024)&1.704(0.075)&0.184(0.007)
&0.580(0.006)&0.154(0.010)&0.555(0.016)&2.109(0.076)&0.183(0.018)\\
GoG&0.586(0.012)&0.180(0.003)&0.538(0.021)&1.953(0.071)&0.299(0.003)
&0.553(0.010)&0.180(0.003)&0.574(0.013)&1.693(0.082)&0.288(0.005)\\
Standalone&0.392(0.009)&0.134(0.004)&0.487(0.026)&1.689(0.058)&0.177(0.002)
&0.528(0.004)&0.121(0.005)&0.517(0.019)&1.671(0.056)&0.181(0.009)\\
\midrule 
Ours&\textbf{0.643(0.011)}&\textbf{0.209(0.008)}&\textbf{0.571(0.012)}&\textbf{0.113(0.002)}&\textbf{0.345(0.006)}
&\textbf{0.633(0.003)}&\textbf{0.193(0.004)}&\textbf{0.568(0.017)}&\textbf{0.126(0.003)}&\textbf{0.326(0.003)}\\
\bottomrule 
\end{tabular}
}
\end{center}
\end{table}

\paragraph{Additional fairness results.}

We provide additional fairness results w.r.t.~the average staleness, online accuracy in \cref{table:pearson-staleness,table:pearson-online-accuracy}.

\begin{table}[!ht]
\setlength{\tabcolsep}{4pt}
\caption{Correlation coefficient $\rho$ between $\zeta$ and \textbf{average staleness} and \textit{higher} $\rho$ indicates better fairness. FGFL is excluded as staleness is not well defined.
}
\label{table:pearson-staleness}
\begin{center}
\resizebox{\linewidth}{!}{
\begin{tabular}{c|ccccc|ccccc}
\multicolumn{1}{c}{}  & \multicolumn{5}{c}{\bf Feature Noise}&\multicolumn{5}{c}{\bf Label Noise}\\
\toprule
 & MNIST & CIFAR-10 & HFT & ELECTRICITY & PATH & MNIST & CIFAR-10 & HFT & ELECTRICITY & PATH \\
 \midrule
FedAvg& -0.213(0.073)&0.000(0.173)&0.012(0.128)&-0.098(0.194)&0.023(0.190) &-0.213(0.073)&0.000(0.173)&0.012(0.128)&-0.098(0.194)&0.023(0.190)\\
qFFL&0.407(0.056)&0.319(0.249)&0.035(0.219)&0.502(0.048)&0.268(0.121)
&0.375(0.033)&0.421(0.117)&0.205(0.221)&\textbf{0.596(0.065)}&0.370(0.116)\\
GoG&0.378(0.054)&-0.151(0.021)&-0.066(0.002)&-0.112(0.028)&0.338(0.029)
&-0.179(0.063)&-0.104(0.016)&-0.271(0.002)&-0.093(0.006)&0.042(0.024)\\
\midrule 
Ours&\textbf{0.650(0.038)}&\textbf{0.627(0.091)}&\textbf{0.387(0.092)}&\textbf{0.670(0.018)}&\textbf{0.646(0.121)}
&\textbf{0.679(0.022)}&\textbf{0.583(0.049)}&\textbf{0.364(0.165)}&0.376(0.172)&\textbf{0.577(0.104)}\\
\bottomrule 
\multicolumn{5}{c}{} \\
\multicolumn{1}{c}{}  &\multicolumn{5}{c}{\bf Quantity} &\multicolumn{5}{c}{\bf Missing Values} \\
\toprule 
 & MNIST & CIFAR-10 & HFT & ELECTRICITY & PATH & MNIST & CIFAR-10 & HFT & ELECTRICITY & PATH \\
 \midrule 
FedAvg&-0.088(0.116)&-0.028(0.126)&-0.067(0.162)&-0.059(0.135)&-0.080(0.100)
&0.018(0.176)&0.028(0.130)&-0.037(0.129)&-0.154(0.135)&-0.008(0.138)\\
qFFL&0.761(0.033)&\textbf{0.801(0.055)}&\textbf{0.778(0.052)}&\textbf{0.728(0.041)}&\textbf{0.762(0.071)}
&-0.066(0.140)&0.411(0.100)&0.042(0.165)&\textbf{0.471(0.155)}&\textbf{0.386(0.164)}\\
GoG&0.052(0.014)&0.250(0.035)&0.157(0.004)&0.298(0.031)&0.188(0.033)
&-0.018(0.008)&0.061(0.006)&0.162(0.001)&0.227(0.007)&-0.247(0.024)\\
\midrule 
Ours&\textbf{0.789(0.015)}&0.717(0.069)&0.057(0.139)&-0.106(0.073)&0.701(0.054)
&\textbf{0.572(0.040)}&\textbf{0.620(0.086)}&\textbf{0.249(0.058)}&0.330(0.109)&0.355(0.113)\\
\bottomrule 
\end{tabular}
}
\end{center}
\end{table}

\begin{table}[!ht]
\setlength{\tabcolsep}{4pt}
\caption{$\rho$ calculated between between $\zeta$ and \textbf{online accuracy}. \textit{Lower} $\rho$ indicates better fairness, except for the ELECTRICITY dataset where the MAPE is used and higher $\rho$ value indicates better fairness. 
}
\label{table:pearson-online-accuracy}
\begin{center}
\resizebox{\linewidth}{!}{
\begin{tabular}{c|ccccc|ccccc}
\multicolumn{1}{c}{}  & \multicolumn{5}{c}{\bf Feature Noise}&\multicolumn{5}{c}{\bf Label Noise}\\
\toprule
 & MNIST & CIFAR-10 & HFT & ELECTRICITY & PATH & MNIST & CIFAR-10 & HFT & ELECTRICITY & PATH \\
 \midrule
FedAvg&-0.015(0.123)&-0.043(0.118)&-0.035(0.093)&-0.088(0.027)&-0.093(0.089)
&0.115(0.038)&-0.060(0.028)&-0.129(0.098)&-0.031(0.098)&0.059(0.038)\\
qFFL&-0.266(0.104)&0.023(0.071)&-0.100(0.068)&0.176(0.136)&-0.005(0.005)
&0.013(0.017)&-0.068(0.031)&0.063(0.040)&0.318(0.063)&-0.161(0.161)\\
FGFL&-0.484(0.051)&-0.120(0.037)&-0.039(0.130)&-0.764(0.025)&\textbf{-0.306(0.133)}
&-0.493(0.084)&-0.144(0.180)&0.124(0.087)&-0.784(0.012)&\textbf{-0.203(0.084)}\\
GoG&-0.296(0.050)&-0.067(0.068)&\textbf{-0.268(0.112)}&-0.296(0.111)&-0.213(0.115)
&-0.389(0.024)&0.029(0.061)&-0.076(0.128)&-0.093(0.065)&0.144(0.061)\\
\midrule 
Ours&\textbf{-0.621(0.026)}&\textbf{-0.429(0.069)}&-0.087(0.057)&\textbf{0.687(0.007)}&-0.072(0.138)
&\textbf{-0.608(0.013)}&\textbf{-0.484(0.047)}&\textbf{-0.178(0.087)}&\textbf{0.384(0.080)}&-0.132(0.142)\\
\bottomrule 
\multicolumn{5}{c}{} \\
\multicolumn{1}{c}{}  &\multicolumn{5}{c}{\bf Quantity} &\multicolumn{5}{c}{\bf Missing Values} \\
\toprule 
 & MNIST & CIFAR-10 & HFT & ELECTRICITY & PATH & MNIST & CIFAR-10 & HFT & ELECTRICITY & PATH \\
 \midrule 
FedAvg&0.014(0.088)&0.039(0.046)&-0.017(0.063)&-0.057(0.097)&0.071(0.179)
&0.024(0.073)&-0.159(0.040)&0.018(0.041)&-0.064(0.126)&-0.172(0.057)\\
qFFL&-0.556(0.184)&0.057(0.057)&-0.025(0.077)&\textbf{0.309(0.056)}&-0.003(0.003)
&-0.089(0.057)&0.026(0.026)&0.085(0.103)&0.278(0.088)&0.042(0.042)\\
FGFL&-0.682(0.005)&\textbf{-0.554(0.065)}&-0.030(0.075)&0.278(0.156)&-0.098(0.206)
&-0.167(0.087)&-0.027(0.114)&0.179(0.077)&\textbf{0.759(0.068)}&\textbf{-0.386(0.074)}\\
GoG&-0.427(0.084)&-0.269(0.074)&\textbf{-0.085(0.047)}&-0.195(0.045)&-0.145(0.072)
&-0.156(0.075)&-0.074(0.043)&-0.047(0.063)&-0.017(0.018)&-0.085(0.089)\\
\midrule 
Ours&\textbf{-0.793(0.014)}&-0.412(0.281)&0.037(0.071)&-0.101(0.040)&\textbf{-0.312(0.134)} 
&\textbf{-0.524(0.024)}&\textbf{-0.360(0.085)}&\textbf{-0.133(0.109)}&0.318(0.052)&-0.048(0.120)\\

\bottomrule 
\end{tabular}
}
\end{center}
\end{table}

\clearpage
\paragraph{Generally positive $\psi_{i,T_\alpha}$.}
We observe  $\psi_{i,T_\alpha}$ is all positive, and validate our assumption that a node with noisy data has lower contributions in \cref{fig:sv_plot_appendix}.

\begin{figure}[!htb]
    \centering 
  \includegraphics[width=0.195\textwidth]{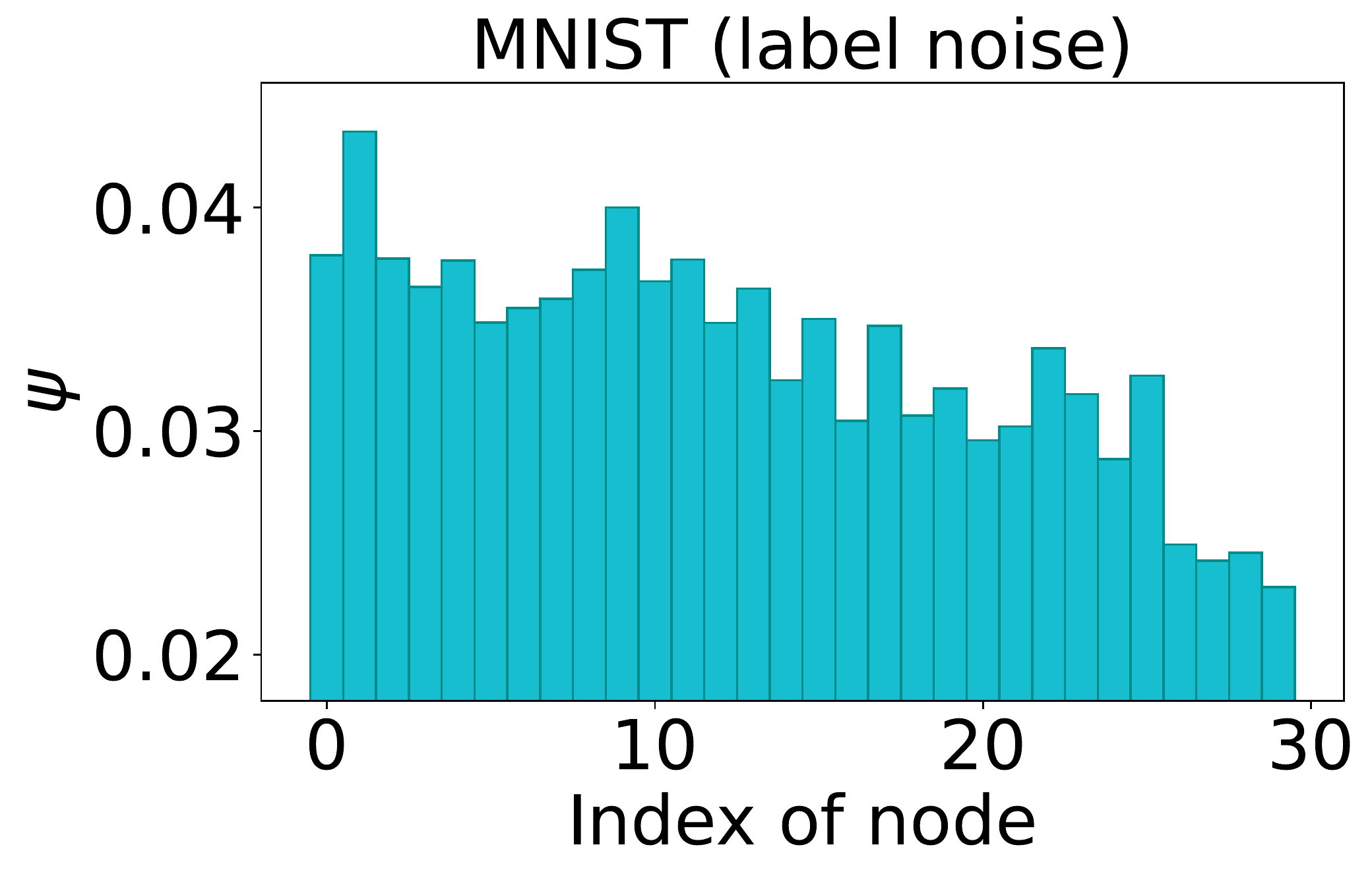}
  \includegraphics[width=0.195\textwidth]{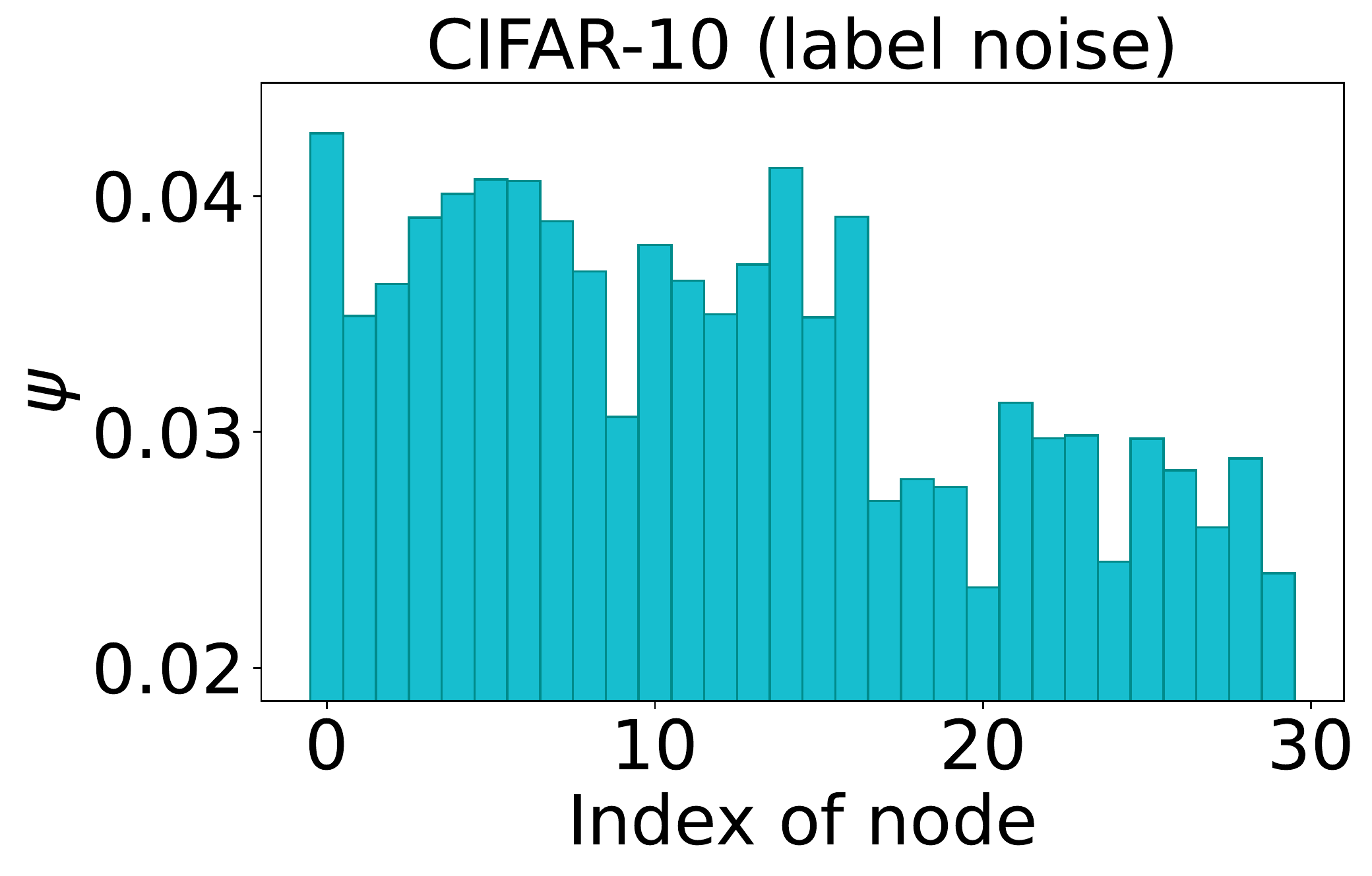}
  \includegraphics[width=0.195\textwidth]{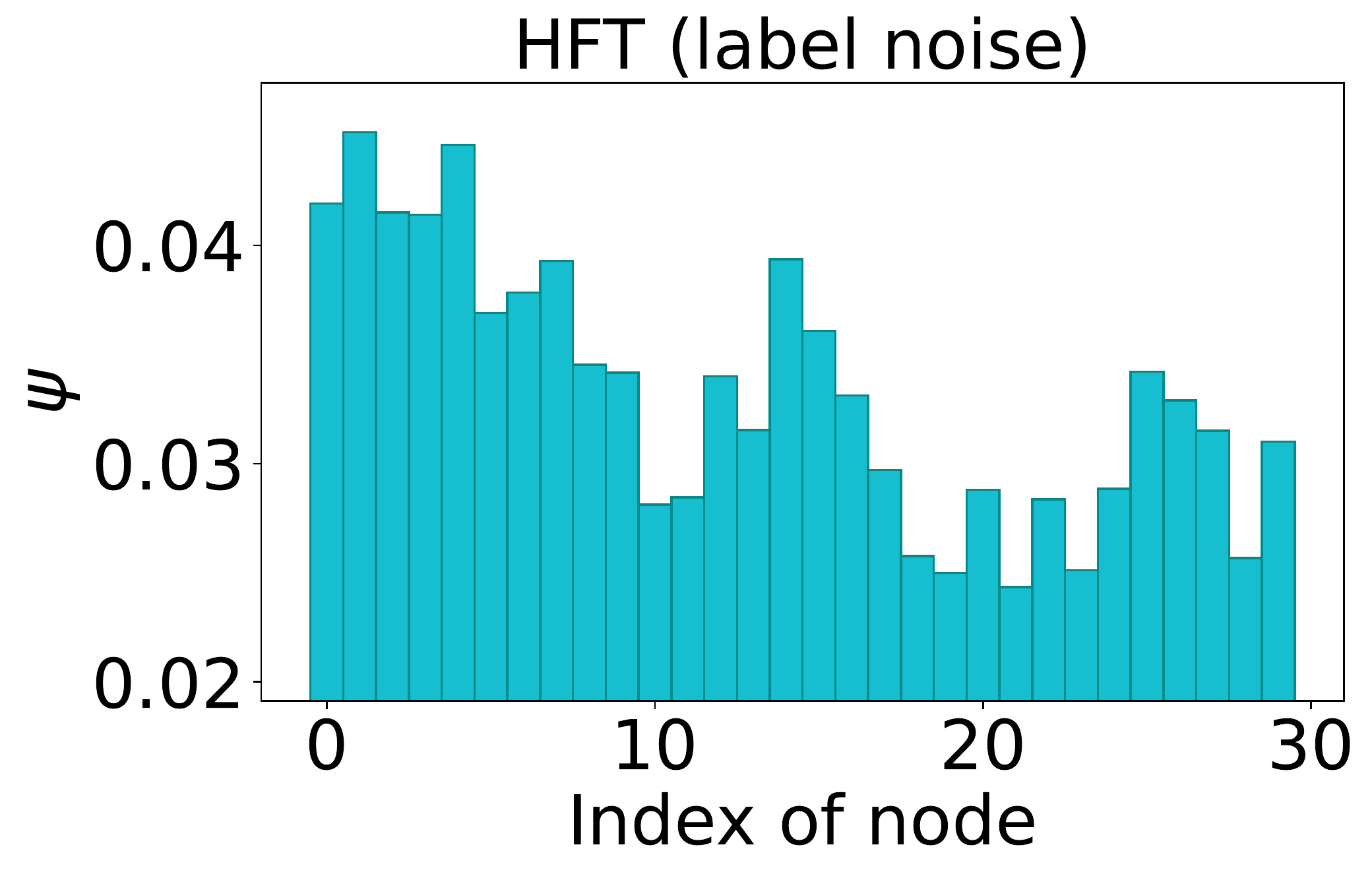}
  \includegraphics[width=0.195\textwidth]{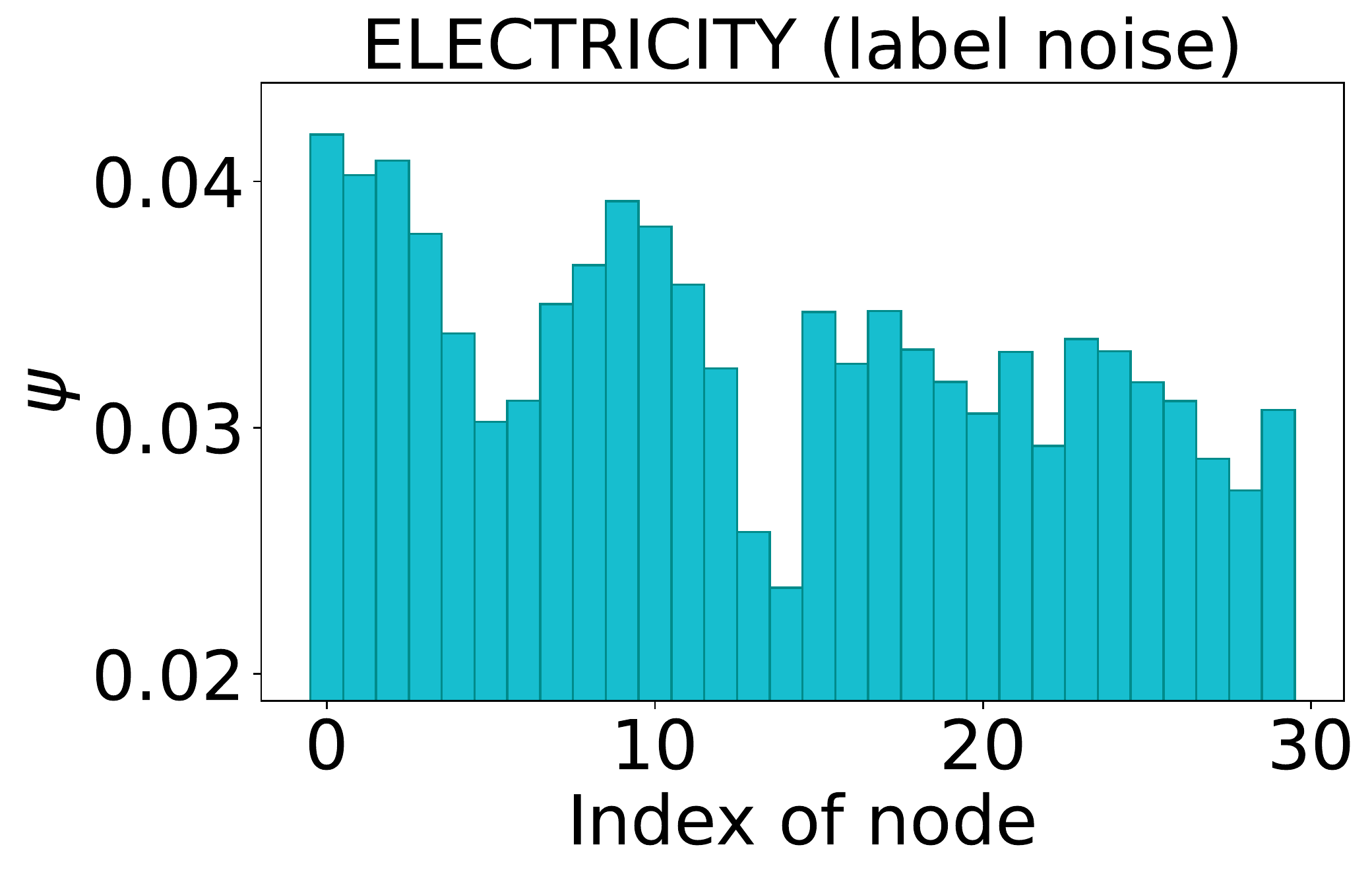}
  \includegraphics[width=0.195\textwidth]{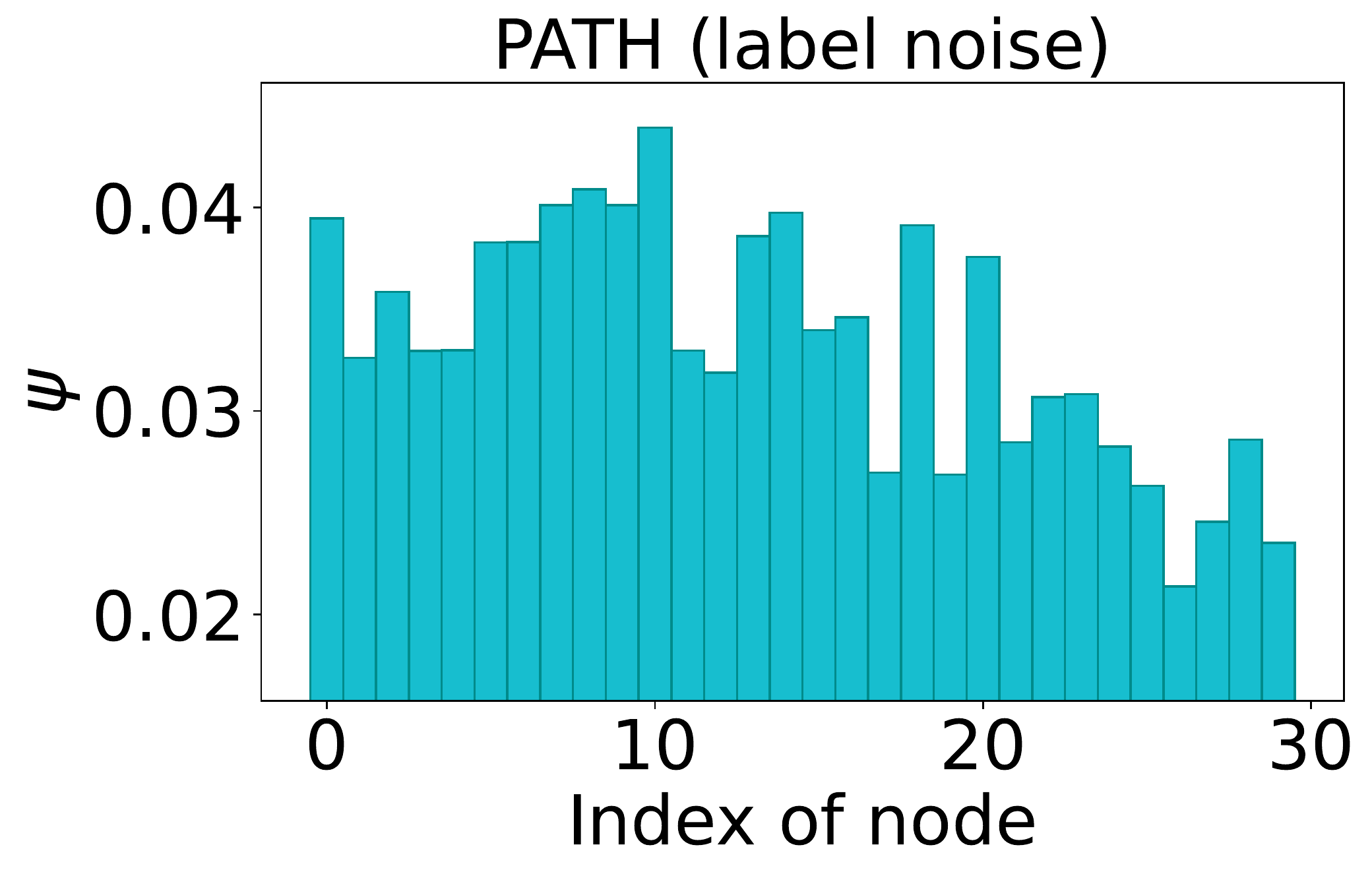}

  \includegraphics[width=0.195\textwidth]{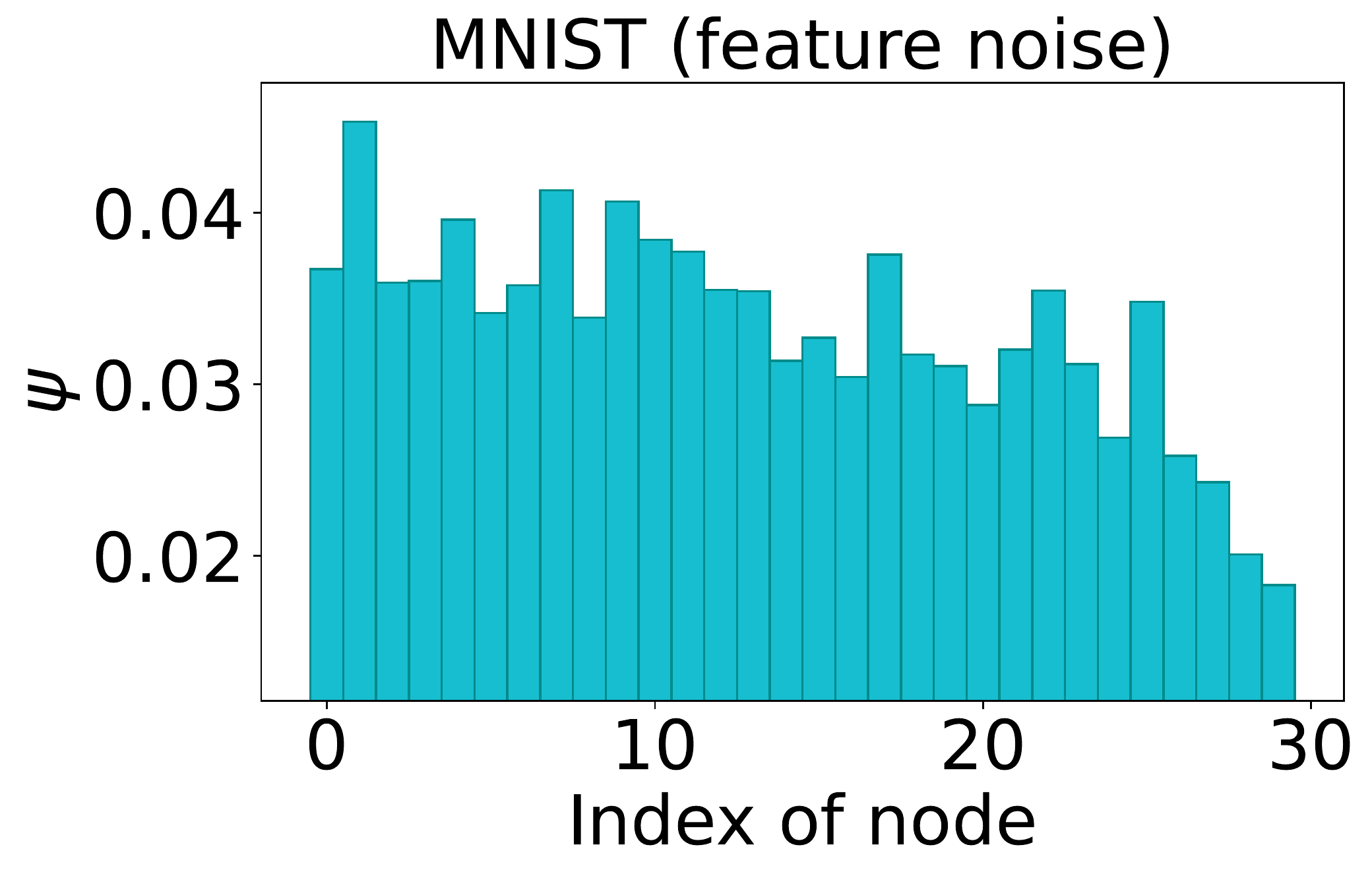}
  \includegraphics[width=0.195\textwidth]{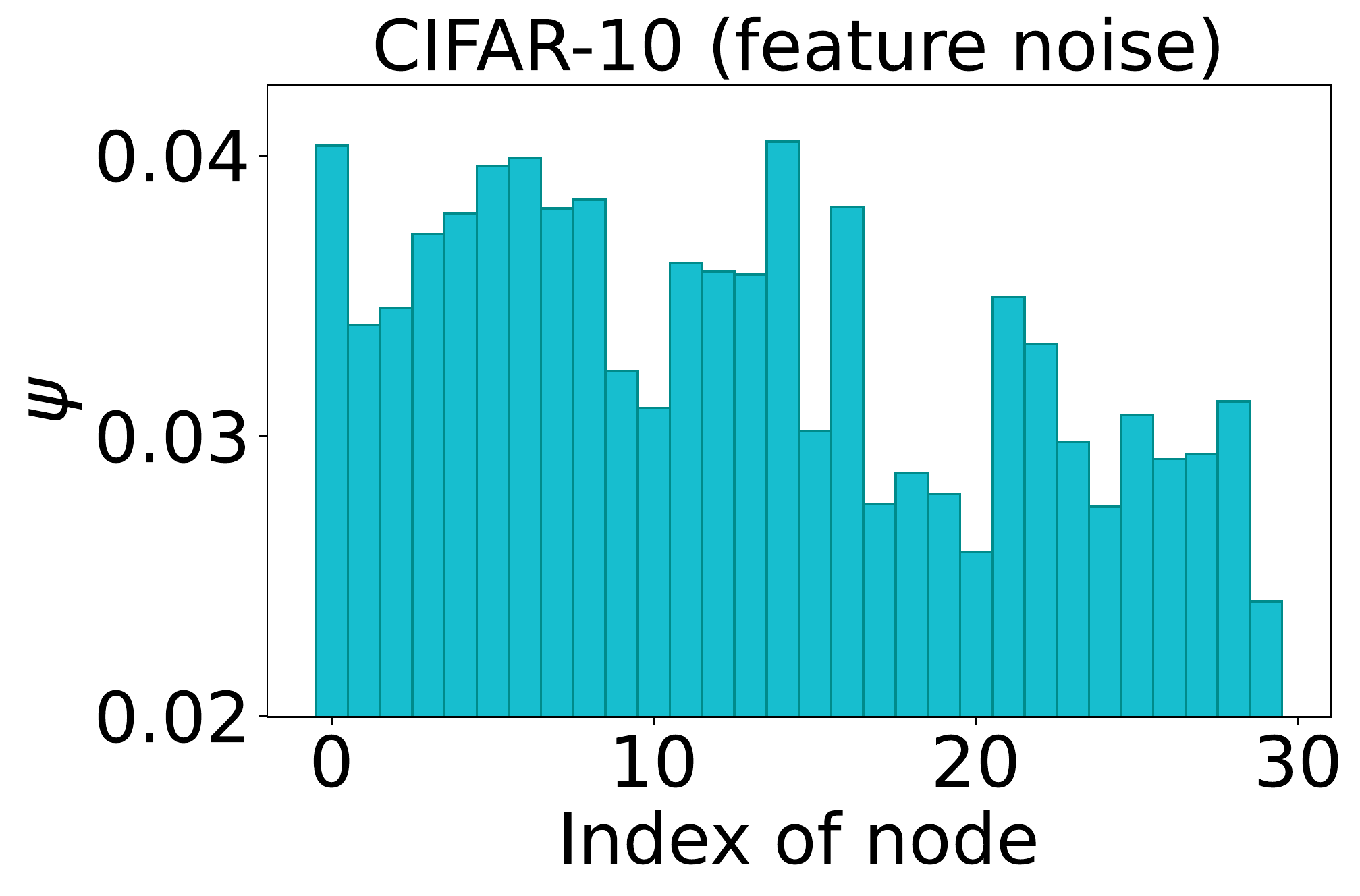}
  \includegraphics[width=0.195\textwidth]{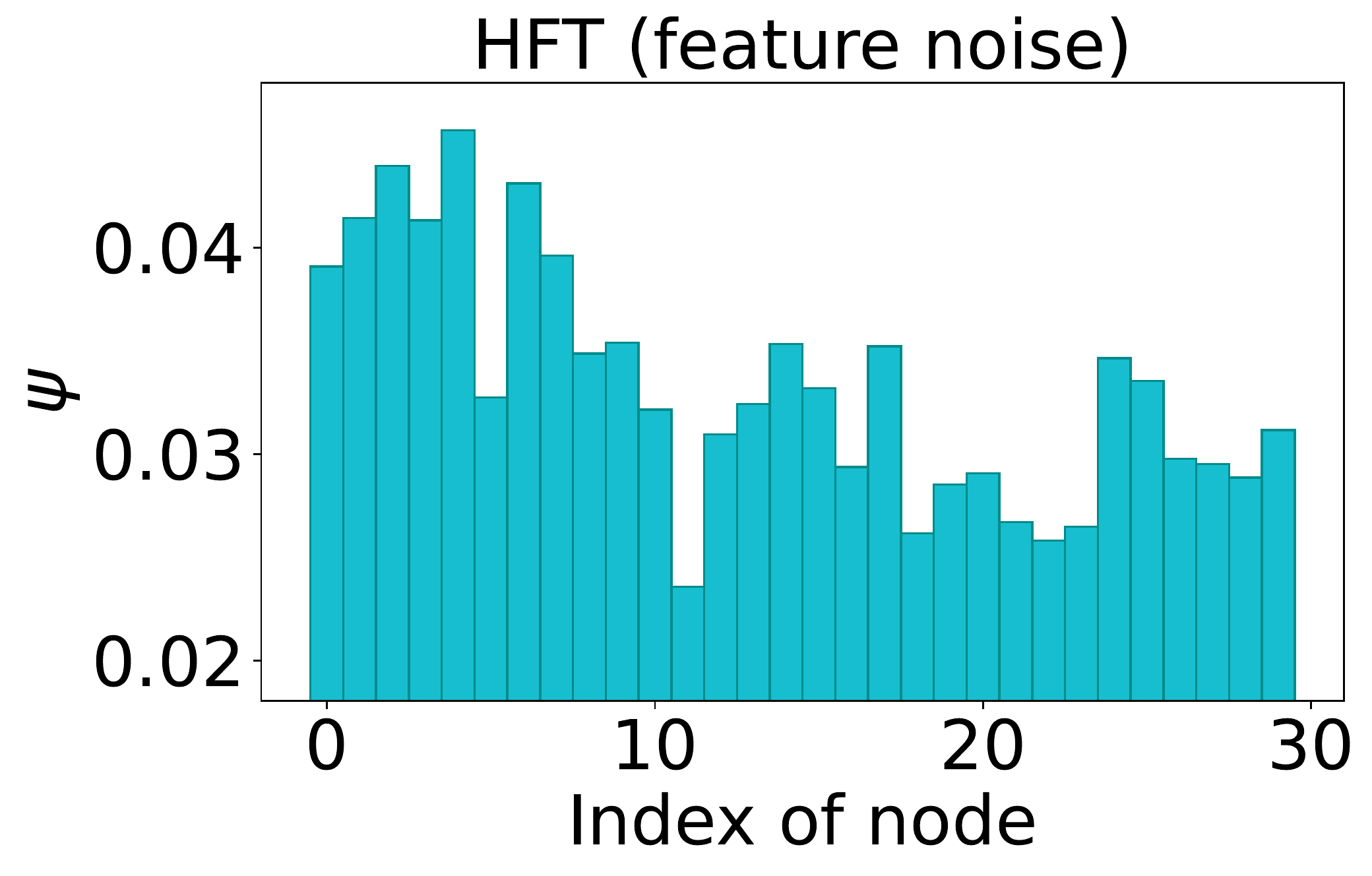}
  \includegraphics[width=0.195\textwidth]{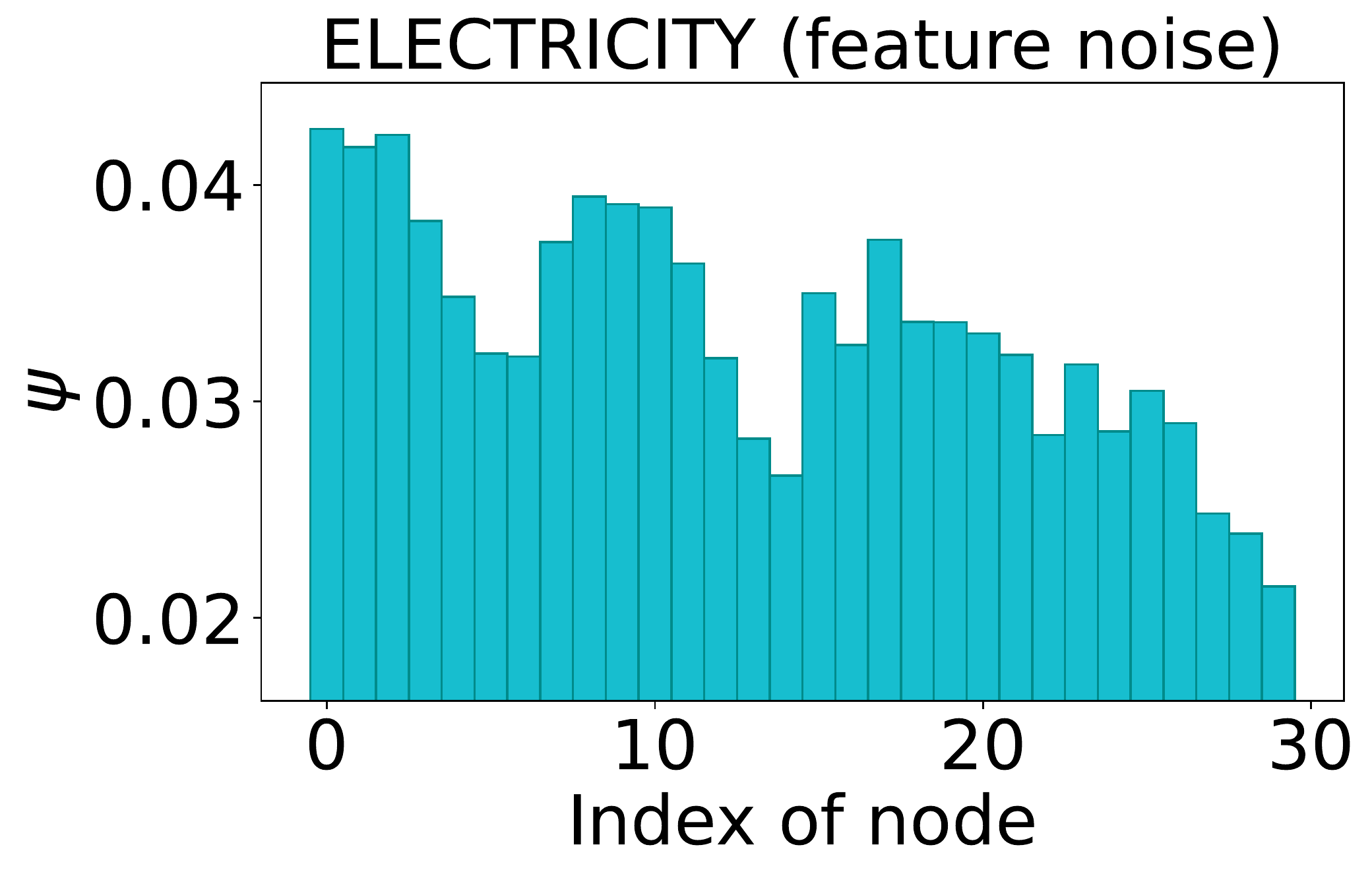}
  \includegraphics[width=0.195\textwidth]{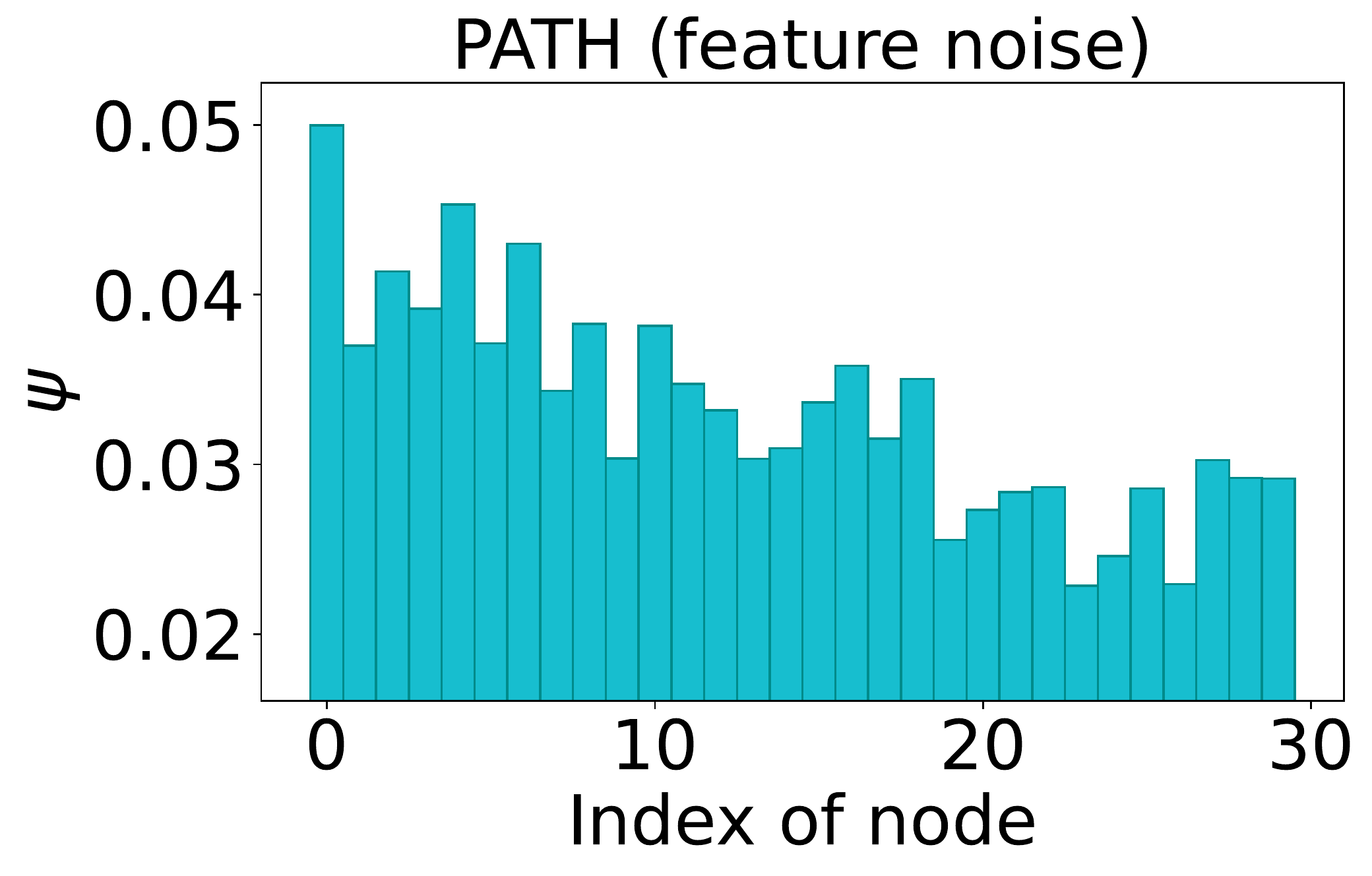}

  \caption{$\boldsymbol{\psi}_{T_{\alpha}}$ vs. node index in an increasing order of $\zeta_i$, the proportion noisy data. In general, higher $\zeta_i$ leads to lower $\psi_{i,T_{\alpha}}$.
    }
    \label{fig:sv_plot_appendix}
\end{figure}

\paragraph{Additional learning performance results.} We provide the performance result on the maximum online accuracy (final iteration accuracy) across different nodes in \cref{table:maximum-online-acc} which shows our approach outperforms other baselines overall. This can incentivize the more resourceful nodes to join the collaboration using our framework for better performance.

\begin{table}[ht]
\setlength{\tabcolsep}{4pt}
\caption{The maximum of the online accuracy~(standard error) over all nodes. Lower is better for ELECTRICITY.
}
\label{table:maximum-online-acc}
\begin{center}
\resizebox{\linewidth}{!}{
\begin{tabular}{c|ccccc|ccccc}
\multicolumn{1}{c}{}  & \multicolumn{5}{c}{\bf Feature Noise}&\multicolumn{5}{c}{\bf Label Noise}\\
\toprule
 & MNIST & CIFAR-10 & HFT & ELECTRICITY & PATH & MNIST & CIFAR-10 & HFT & ELECTRICITY & PATH \\
 \midrule
FedAvg&0.487(0.019)&0.167(0.011)&0.501(0.045)&1.411(0.081)&0.258(0.004)
&0.512(0.012)&0.162(0.009)&0.474(0.048)&1.388(0.052)&0.311(0.009)\\
qFFL&0.101(0.011)&0.100(0.004)&0.281(0.079)&1.413(0.055)&0.101(0.011)
&0.112(0.006)&0.103(0.002)&0.374(0.091)&1.618(0.117)&0.109(0.013)\\
FGFL&0.491(0.017)&0.170(0.010)&0.548(0.013)&1.645(0.096)&0.155(0.009)
&0.547(0.008)&0.159(0.008)&0.553(0.010)&1.946(0.046)&0.188(0.010)\\
GoG&0.588(0.014)&\textbf{0.197(0.005)}&0.559(0.015)&1.402(0.031)&0.300(0.004)
&0.612(0.006)&0.197(0.005)&0.539(0.016)&1.674(0.079)&0.344(0.006)\\
Standalone&0.571(0.012)&0.174(0.004)&0.561(0.014)&1.946(0.073)&0.266(0.004)
&0.587(0.004)&0.178(0.003)&\textbf{0.587(0.014)}&1.889(0.102)&0.263(0.002)\\
\midrule 
Ours&\textbf{0.613(0.009)}&0.196(0.007)&\textbf{0.581(0.014)}&\textbf{0.140(0.002)}&\textbf{0.305(0.005)}
&\textbf{0.664(0.003)}&\textbf{0.206(0.009)}&0.566(0.014)&\textbf{0.126(0.004)}&\textbf{0.384(0.008)}\\
\bottomrule 
\multicolumn{5}{c}{} \\
\multicolumn{1}{c}{}  &\multicolumn{5}{c}{\bf Quantity} &\multicolumn{5}{c}{\bf Missing Values} \\
\toprule 
 & MNIST & CIFAR-10 & HFT & ELECTRICITY & PATH & MNIST & CIFAR-10 & HFT & ELECTRICITY & PATH \\
 \midrule 
FedAvg&0.512(0.010)&0.139(0.009)&0.464(0.031)&1.759(0.128)&0.303(0.007)
&0.470(0.011)&0.150(0.003)&0.463(0.047)&1.429(0.057)&0.295(0.006)\\
qFFL&0.114(0.010)&0.104(0.004)&0.350(0.071)&1.697(0.148)&0.129(0.007)
&0.112(0.021)&0.108(0.004)&0.317(0.065)&1.425(0.080)&0.099(0.006)\\
FGFL&0.580(0.011)&0.170(0.006)&0.533(0.020)&1.956(0.066)&0.185(0.007)
&0.583(0.005)&0.157(0.010)&0.558(0.016)&2.614(0.066)&0.185(0.018)\\
GoG&0.615(0.009)&0.189(0.004)&0.545(0.020)&1.990(0.072)&0.333(0.005)
&0.572(0.011)&0.189(0.003)&\textbf{0.578(0.012)}&1.740(0.087)&0.308(0.004)\\
Standalone&0.646(0.006)&0.178(0.003)&0.557(0.014)&1.926(0.052)&0.278(0.004)
&0.630(0.003)&0.173(0.003)&0.554(0.013)&1.978(0.052)&0.260(0.003)\\
\midrule 
Ours&\textbf{0.654(0.009)}&\textbf{0.214(0.007)}&\textbf{0.571(0.012)}&\textbf{0.114(0.002)}&\textbf{0.357(0.006)}
&\textbf{0.635(0.003)}&\textbf{0.195(0.004)}&0.569(0.016)&\textbf{0.127(0.003)}&\textbf{0.333(0.003)}\\
\bottomrule 
\end{tabular}
}

\end{center}
\end{table}

\paragraph{Additional contribution estimates vs.~memory size/exploration ratio result in FRL on SpaceInvaders and Pong.} We provide the contribution estimates result of our framework on SpaceInvaders and Pong in \cref{fig:RL-additional-appendix}. The results from both these games/environments provide the consistent observations where a small memory size/a well-moderated exploration ratio results in a high contribution.

Note that our finding from \cref{fig:RL-additional} (left) is consistent with \cite{zhang2018deeper} who has a similar setting to ours, and empirically shows that a memory size of $10^4$ leads to better performance and faster improvement than $10^6$.

\begin{figure}[ht]
    \centering 
  \includegraphics[width=0.35\textwidth]{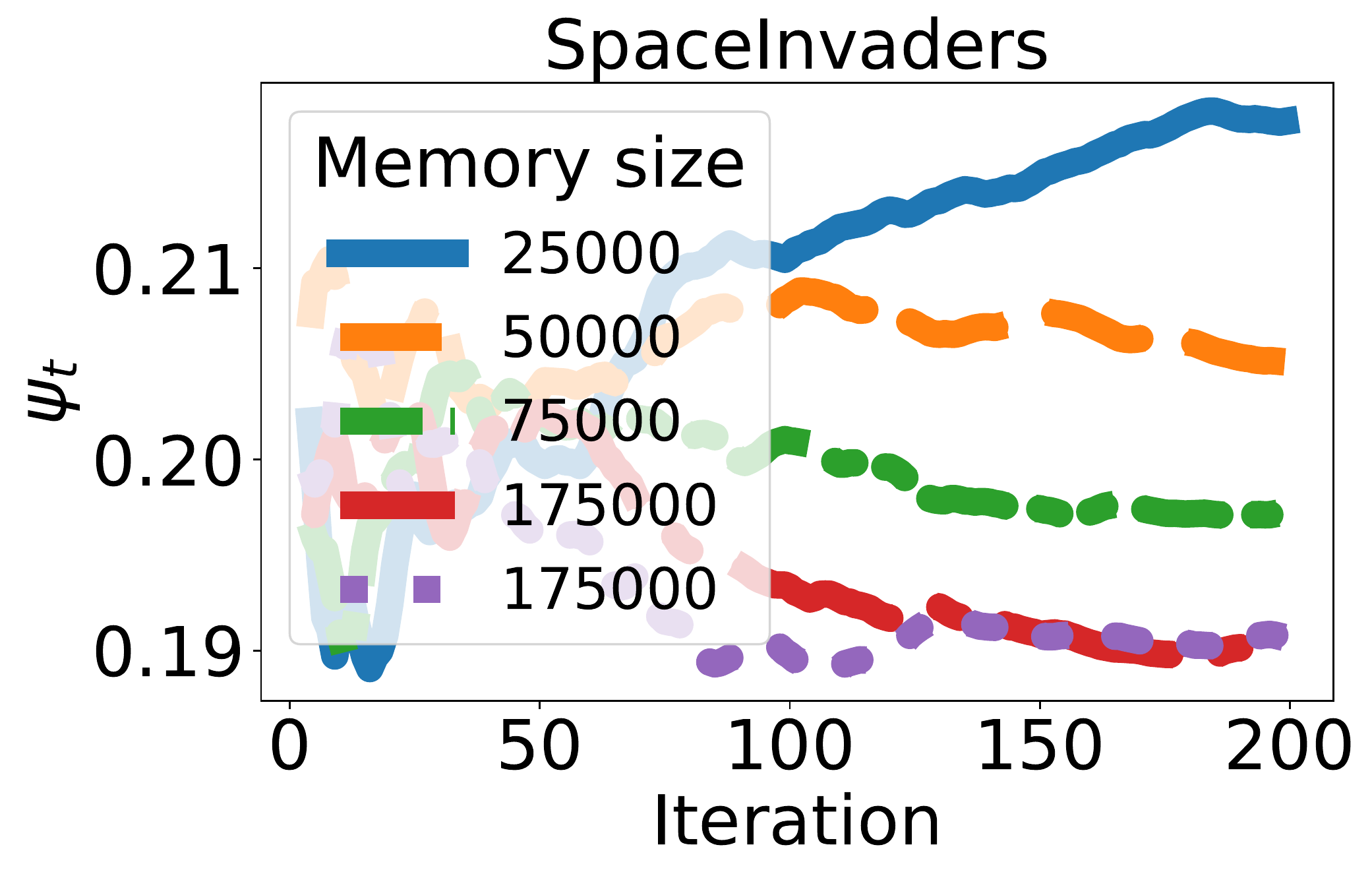}
  \includegraphics[width=0.35\textwidth]{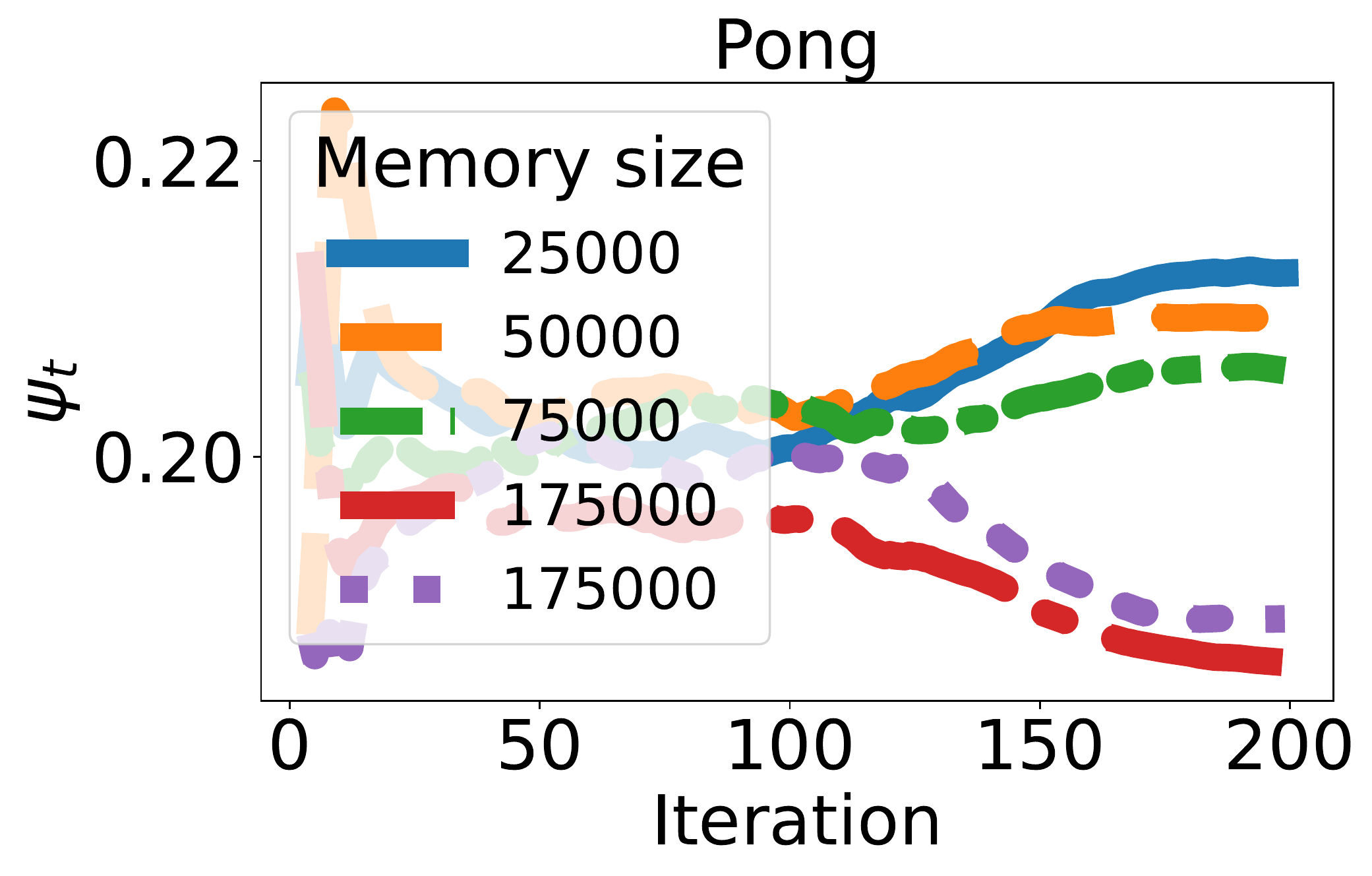} \\
  
  \includegraphics[width=0.35\textwidth]{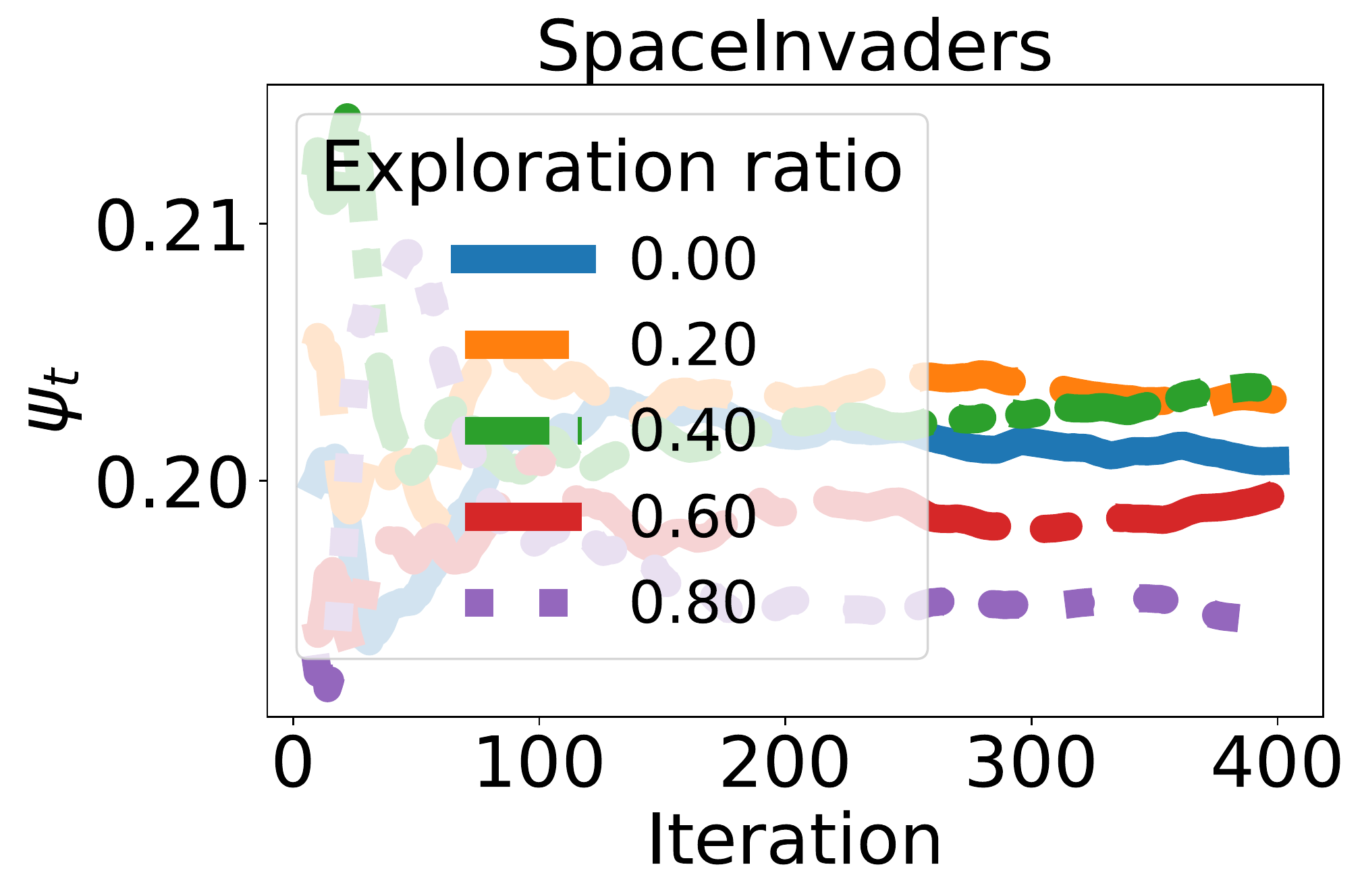}
  \includegraphics[width=0.35\textwidth]{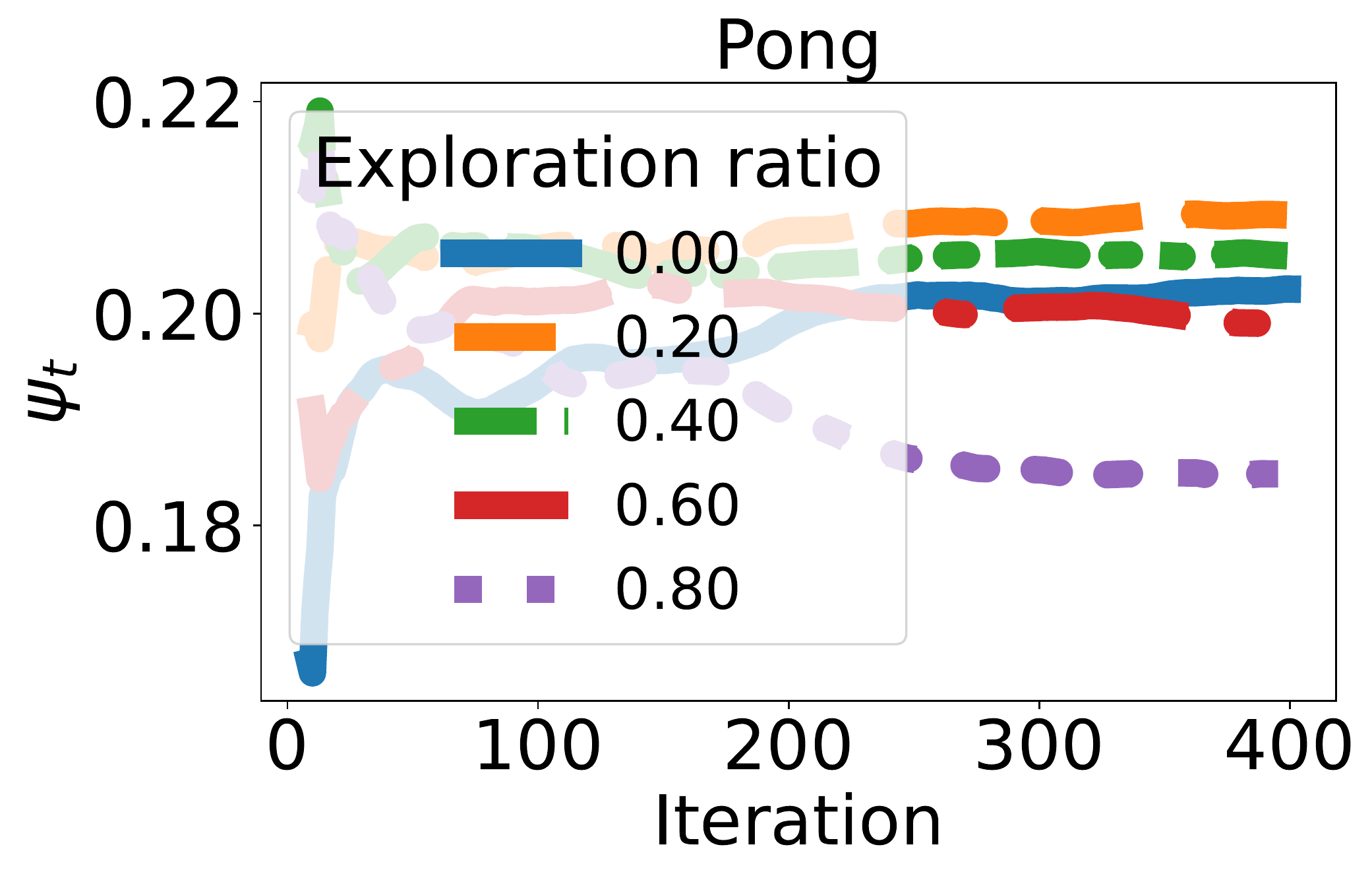}

  \caption{Contribution estimates $\boldsymbol{\psi}_t$ of nodes with different memory size (exploration ratio), smoothing (over $5$ consecutive $\psi_{i,t}$) is applied for plotting. 
  }
    \label{fig:RL-additional-appendix}
\end{figure}

\clearpage
\subsection{Additional Equality Comparison}
We provide additional equality comparison across different baselines and settings in \cref{table:standard-dev-online-loss} w.r.t.~loss and \cref{table:standard-dev-online-accu} w.r.t.~accuracy. In general a lower standard deviation suggests the performance among the nodes are more equitable~\cite{Li2019-fedavg-noniid}. While qFFL seems to have lowest standard deviation overall, the difference (in terms of standard deviation) among baselines (excluding Standalone) is quite small (all smaller or around $10^{-3}$).

\begin{table}[ht]
\setlength{\tabcolsep}{4pt}
\small
\caption{Standard deviation of the online loss and (standard error over $5$ runs) over all nodes.}
\label{table:standard-dev-online-loss}
\begin{center}
\begin{tabular}{c|ccccc}
\multicolumn{1}{c}{}  &\multicolumn{4}{c}{\bf Feature Noise} \\
\toprule 
 & MNIST & CIFAR-10 & HFT & ELECTRICITY & PATH \\
 \midrule 
FedAvg&4.8e-05(6.1e-06)&4.7e-06(1.5e-07)&7.1e-05(1.3e-05)&\textbf{6.1e-05(3.8e-06)}&1.9e-05(1.2e-06)\\
qFFL&\textbf{6.0e-06(4.2e-07)}&\textbf{3.7e-06(3.8e-07)}&\textbf{3.4e-08(2.1e-09)}&7.6e-05(5.0e-06)&\textbf{9.7e-07(1.0e-07)}\\
FGFL&2.5e-04(3.3e-05)&6.6e-06(1.3e-07)&8.9e-04(6.4e-04)&2.9e-04(5.6e-05)&4.3e-06(1.5e-06)\\
GoG&1.9e-04(7.6e-06)&1.3e-05(1.3e-06)&1.2e-04(5.1e-06)&1.1e-04(3.7e-06)&1.2e-04(1.0e-05)\\
Standalone&4.4e-03(2.8e-04)&1.5e-04(1.1e-05)&2.1e-03(9.6e-05)&4.4e-03(1.4e-04)&6.6e-04(1.3e-05)\\
\midrule 
Ours&4.1e-04(4.3e-05)&1.4e-05(1.4e-06)&2.1e-04(8.9e-05)&1.2e-04(7.1e-06)&2.9e-05(5.2e-06)\\
\bottomrule 
\multicolumn{5}{c}{} \\
\multicolumn{1}{c}{}  & \multicolumn{4}{c}{\bf Label Noise}\\
 \toprule
 & MNIST & CIFAR-10 & HFT & ELECTRICITY & PATH \\
 \midrule
FedAvg&5.7e-05(2.0e-06)&\textbf{4.3e-06(2.3e-07)}&6.5e-05(1.1e-05)&\textbf{6.5e-05(4.7e-06)}&1.7e-05(1.3e-06)\\
qFFL&\textbf{6.7e-06(5.1e-07)}&5.3e-06(7.8e-07)&\textbf{3.8e-08(4.5e-09)}&8.4e-05(4.0e-06)&\textbf{1.1e-06(7.2e-08)}\\
FGFL&1.0e-04(1.1e-05)&8.7e-06(1.4e-06)&8.0e-04(5.2e-04)&6.9e-04(3.0e-05)&2.1e-06(2.9e-07)\\
GoG&1.3e-04(5.1e-06)&1.3e-05(3.9e-07)&1.4e-04(6.1e-06)&1.0e-04(2.0e-06)&1.8e-04(1.8e-05)\\
Standalone&2.4e-03(5.0e-05)&1.9e-04(7.0e-06)&1.9e-03(5.0e-05)&2.1e-03(8.4e-05)&4.9e-04(2.9e-05)\\

\midrule 
Ours&1.2e-04(2.4e-05)&1.8e-05(6.6e-07)&7.8e-04(1.5e-04)&1.1e-04(1.9e-05)&3.5e-05(1.1e-05)\\
\bottomrule 
\multicolumn{5}{c}{} \\
\multicolumn{1}{c}{}  & \multicolumn{4}{c}{\bf Quantity}\\
 \toprule
 & MNIST & CIFAR-10 & HFT & ELECTRICITY & PATH \\
 \midrule
FedAvg&6.0e-05(3.0e-06)&\textbf{4.7e-06(5.5e-07)}&7.1e-05(6.7e-06)&1.5e-04(1.2e-05)&2.1e-05(7.3e-07)\\
qFFL&\textbf{4.7e-06(6.7e-07)}&5.1e-06(7.7e-07)&\textbf{1.6e-06(2.4e-07)}&1.8e-04(9.6e-06)&\textbf{1.1e-06(1.2e-07)}\\
FGFL&2.9e-04(1.2e-05)&8.1e-06(9.3e-07)&5.6e-04(1.5e-04)&1.4e-04(1.1e-05)&2.3e-06(3.9e-07)\\
GoG&1.1e-04(1.4e-06)&1.5e-05(1.7e-06)&1.3e-04(1.0e-05)&1.3e-04(7.9e-06)&2.8e-04(3.9e-05)\\
Standalone&1.8e-03(7.9e-05)&1.9e-04(8.2e-06)&3.2e-03(1.7e-04)&1.2e-03(7.3e-05)&6.5e-04(1.6e-05)\\
\midrule 
Ours&3.9e-04(2.8e-05)&3.5e-05(5.8e-06)&9.3e-04(2.0e-04)&\textbf{7.9e-05(1.2e-05)}&3.4e-05(6.0e-06)\\
\bottomrule 
\multicolumn{5}{c}{} \\
\multicolumn{1}{c}{}  & \multicolumn{4}{c}{\bf Missing Values}\\
 \toprule
 & MNIST & CIFAR-10 & HFT & ELECTRICITY & PATH \\
 \midrule
FedAvg&4.7e-05(4.3e-06)&\textbf{4.5e-06(3.3e-07)}&6.1e-05(5.3e-06)&8.8e-05(7.5e-06)&2.1e-05(1.9e-06)\\
qFFL&\textbf{3.9e-06(6.7e-07)}&4.6e-06(2.8e-07)&\textbf{3.6e-08(6.4e-09)}&\textbf{7.5e-05(5.2e-06)}&\textbf{9.8e-07(1.3e-07)}\\
FGFL&4.0e-05(5.3e-06)&6.6e-06(6.8e-07)&1.3e-04(2.3e-05)&7.2e-04(3.3e-05)&1.9e-06(1.4e-07)\\
GoG&9.5e-05(3.2e-06)&1.3e-05(4.3e-07)&1.2e-04(5.0e-06)&1.1e-04(8.6e-06)&2.3e-04(2.7e-05)\\
Standalone&1.6e-03(6.1e-05)&1.6e-04(1.6e-05)&1.7e-03(1.1e-04)&1.5e-03(5.0e-05)&5.1e-04(2.3e-05)\\
\midrule 
Ours&5.9e-05(3.6e-06)&1.7e-05(2.0e-06)&5.0e-04(5.4e-05)&8.0e-05(7.0e-06)&1.4e-05(1.1e-06)\\
\bottomrule 
\end{tabular}
\end{center}
\end{table}

\begin{table}[ht]
\setlength{\tabcolsep}{4pt}
\small
\caption{Standard deviation of the online accuracy and (standard error over $5$ runs) over all nodes. 
}
\label{table:standard-dev-online-accu}
\begin{center}
\begin{tabular}{c|ccccc}
\multicolumn{1}{c}{}  &\multicolumn{4}{c}{\bf Feature Noise} \\
\toprule 
 & MNIST & CIFAR-10 & HFT & ELECTRICITY & PATH \\
 \midrule 
FedAvg&2.3e-03(3.1e-04)&4.2e-04(3.7e-05)&1.1e-03(4.2e-04)&1.3e-03(1.9e-04)&1.7e-03(4.7e-05)\\
qFFL&\textbf{7.2e-05(1.5e-05)}&\textbf{1.8e-05(5.0e-06)}&\textbf{1.5e-06(9.7e-07)}&\textbf{5.1e-07(1.7e-07)}&\textbf{1.3e-05(1.3e-05)}\\
FGFL&1.0e-02(3.3e-03)&6.5e-04(8.0e-05)&1.0e-02(9.1e-03)&3.3e-02(6.8e-03)&5.9e-04(2.3e-04)\\
GoG&9.1e-03(8.3e-04)&1.7e-03(1.4e-04)&2.3e-03(5.0e-04)&4.8e-03(5.5e-04)&6.8e-03(4.4e-04)\\
Standalone&6.8e-02(3.8e-03)&1.0e-02(7.8e-04)&1.1e-02(1.3e-03)&2.0e-01(1.2e-02)&3.2e-02(6.3e-04)\\

\midrule 
Ours&2.2e-03(4.0e-04)&6.5e-04(1.5e-04)&5.8e-05(4.4e-05)&3.0e-04(1.9e-05)&1.4e-03(2.1e-04)\\
\bottomrule 
\multicolumn{5}{c}{} \\
\multicolumn{1}{c}{}  & \multicolumn{4}{c}{\bf Label Noise}\\
 \toprule
 & MNIST & CIFAR-10 & HFT & ELECTRICITY & PATH \\
 \midrule
FedAvg&2.3e-03(1.5e-04)&4.8e-04(7.6e-05)&1.6e-03(5.0e-04)&1.6e-03(2.3e-04)&1.7e-03(1.9e-04)\\
qFFL&\textbf{5.9e-05(2.3e-05)}&\textbf{1.0e-05(6.3e-06)}&\textbf{4.6e-06(2.9e-06)}&\textbf{5.8e-07(1.2e-07)}&\textbf{5.8e-06(5.8e-06)}\\
FGFL&3.1e-03(1.3e-03)&1.0e-03(3.7e-04)&2.4e-03(1.1e-03)&1.1e-01(1.3e-02)&5.6e-04(8.2e-05)\\
GoG&6.9e-03(6.4e-04)&2.4e-03(1.2e-04)&2.0e-03(3.7e-04)&6.5e-03(2.7e-04)&8.6e-03(3.2e-04)\\
Standalone&4.1e-02(1.1e-03)&1.2e-02(7.0e-04)&4.9e-03(1.8e-03)&1.4e-01(1.5e-02)&2.0e-02(1.8e-03)\\

\midrule 
Ours&7.5e-04(2.1e-04)&6.3e-04(1.4e-04)&3.4e-04(2.5e-04)&3.3e-04(5.6e-05)&1.9e-03(1.6e-04)\\
\bottomrule 
\multicolumn{5}{c}{} \\
\multicolumn{1}{c}{}  & \multicolumn{4}{c}{\bf Quantity}\\
 \toprule
 & MNIST & CIFAR-10 & HFT & ELECTRICITY & PATH \\
 \midrule
FedAvg&2.1e-03(2.9e-04)&3.4e-04(3.2e-05)&1.2e-03(1.5e-04)&2.5e-03(2.7e-04)&1.6e-03(1.1e-04)\\
qFFL&\textbf{1.5e-04(4.4e-05)}&\textbf{9.8e-06(6.2e-06)}&\textbf{4.1e-06(2.5e-06)}&\textbf{1.3e-06(2.5e-07)}&\textbf{1.3e-05(1.3e-05)}\\
FGFL&5.2e-03(6.4e-04)&1.2e-03(3.0e-04)&2.2e-03(1.0e-03)&6.1e-02(9.4e-03)&1.9e-04(4.2e-05)\\
GoG&5.7e-03(6.7e-04)&2.5e-03(1.6e-04)&1.5e-03(3.3e-04)&7.2e-03(6.5e-04)&8.7e-03(7.2e-04)\\
Standalone&4.3e-02(1.0e-03)&1.1e-02(1.0e-03)&1.6e-02(2.3e-03)&6.4e-02(3.7e-03)&2.4e-02(1.1e-03)\\
\midrule 
Ours&1.8e-03(3.4e-04)&1.2e-03(3.4e-04)&5.4e-05(1.2e-05)&1.8e-04(2.1e-05)&2.7e-03(3.6e-04)\\
\bottomrule 
\multicolumn{5}{c}{} \\
\multicolumn{1}{c}{}  & \multicolumn{4}{c}{\bf Missing Values}\\
 \toprule
 & MNIST & CIFAR-10 & HFT & ELECTRICITY & PATH \\
 \midrule
FedAvg&1.9e-03(2.5e-04)&4.1e-04(3.3e-05)&1.7e-03(1.4e-04)&2.8e-03(2.8e-04)&1.6e-03(1.8e-04)\\
qFFL&\textbf{3.8e-05(4.8e-06)}&\textbf{3.7e-06(3.7e-06)}&\textbf{7.3e-06(2.5e-06)}&\textbf{5.4e-07(5.0e-08)}&\textbf{2.2e-05(2.2e-05)}\\
FGFL&7.0e-04(1.8e-04)&5.5e-04(9.9e-05)&7.9e-04(3.6e-04)&1.3e-01(8.7e-03)&4.3e-04(7.4e-05)\\
GoG&4.4e-03(2.9e-04)&1.9e-03(6.2e-05)&1.1e-03(2.3e-04)&1.1e-02(1.9e-03)&5.2e-03(2.7e-04)\\
Standalone&2.7e-02(4.3e-04)&1.3e-02(7.4e-04)&8.8e-03(1.7e-03)&8.9e-02(5.3e-03)&2.1e-02(1.6e-03)\\
\midrule 
Ours&3.3e-04(3.9e-05)&5.8e-04(8.2e-05)&2.7e-04(2.5e-04)&2.6e-04(2.9e-05)&1.4e-03(2.2e-04)\\
\bottomrule 

\end{tabular}
\end{center}
\end{table}
\subsection{Additional comparison to the simple extension of FedAvg using fairness or equality constraints}\label{appendix:constraint}

For fairness, to the best of our knowledge, there does not seem to be a simple/straightforward or sensible extension to achieve fairness (i.e., giving commensurately more rewards to the nodes with higher contributions). As for equality, we can consider the following simple extension (by penalizing the deviation from equal performance in the objective function):
    $$\min_{\theta} \mathbf{J}(\theta)=\sum_{i=1}^n p_i \mathbf{L}(\theta; \mathcal{D_i}) \quad s.t.\sum_{i=1}^{n}\sum_{j=1}^{n} \bigl(\mathbf{L}(\theta; \mathcal{D_i}) - \mathbf{L}(\theta; \mathcal{D_j})\bigl)^2 \le \alpha \ .$$ Intuitively, it minimizes the overall federated objective and constraints difference of the losses among different nodes within some $\alpha$ to achieve equality of model performance among nodes. To make this optimization tractable, we transform the constraint to a penalty term in the objective as follows, $$\mathbf{J}(\theta)=\sum_{i=1}^n p_i \mathbf{L}(\theta; \mathcal{D_i}) + \lambda \sum_{i=1}^{n}\sum_{j=1}^{n} \bigl(\mathbf{L}(\theta; \mathcal{D_i}) - \mathbf{L}(\theta; \mathcal{D_j})\bigl)^2 $$ where $\lambda$ balances the importance of model performance and equality. Based on this, the modified FedAvg algorithm computes gradient update (for each node) as $$\sum_{i=1}^n p_i \nabla_{\theta}\mathbf{L}(\theta; \mathcal{D_i}) + 2\lambda\sum_{j=1}^n\sum_{k=1}^n \bigl(\mathbf{L}(\theta; \mathcal{D_j}) - \mathbf{L}(\theta; \mathcal{D_k})\bigl)\bigl(\nabla_{\theta}\mathbf{L}(\theta; \mathcal{D_j}) - \nabla_{\theta}\mathbf{L}(\theta; \mathcal{D_k})\bigl)\ .$$
    
    We compare the fairness, equality, and performance result of this constraint-based approach (denoted by Cons in the following tables) with our approach and other baselines. The parameter $\lambda$ is selected/tuned over the range $[0,1]$ to be $\lambda=0.3$ here since it achieves some degree of equality while maintaining a relatively good performance.

\begin{table}[!ht]
    \caption{Fairness comparison. $\rho(\text{online loss}, \zeta)$ FOIL under the setting of feature noise. Higher $\rho$ implies better fairness.
    }
    \centering
    \label{tab:fairness-online-constraint}
    \begin{tabular}{c|ccccc}
    \toprule 
     & MNIST & CIFAR-10 & HFT & ELECTRICITY & PATH \\
     \midrule 
    FedAvg&-0.020(0.097)&0.137(0.049)&0.038(0.045)&-0.033(0.038)&0.135(0.026)\\
    qFFL&-0.022(0.114)&-0.109(0.140)&0.060(0.078)&0.036(0.126)&-0.236(0.030)\\
    FGFL&0.556(0.032)&0.313(0.081)&0.055(0.033)&0.476(0.057)&0.419(0.098)\\
    GoG&0.551(0.023)&0.130(0.067)&0.201(0.021)&0.512(0.027)&0.189(0.102)\\
    Cons&0.101(0.103)&-0.085(0.136)&0.115(0.092)&0.033(0.048)&0.380(0.072)\\
    \midrule 
    Ours&0.647(0.018)&0.400(0.069)&0.378(0.055)&0.676(0.018)&0.557(0.060)\\
    \bottomrule 
    \end{tabular}
\end{table}

\begin{table}[!ht]
    \caption{Average of online accuracy (standard error) over all nodes under the setting of feature noise. For ELECTRICITY, we measure MAPE, so lower is better.
    }
    \centering
    \label{tab:online-acc-constraint}
    \begin{tabular}{c|ccccc}
    \toprule 
     & MNIST & CIFAR-10 & HFT & ELECTRICITY & PATH \\
     \midrule 
    FedAvg&0.483(0.019)&0.166(0.011)&0.499(0.046)&1.408(0.081)&0.255(0.004)\\
    qFFL&0.101(0.011)&0.100(0.004)&0.281(0.079)&1.413(0.055)&0.101(0.011)\\
    FGFL&0.485(0.018)&0.169(0.010)&0.496(0.056)&1.571(0.090)&0.154(0.009)\\
    GoG&0.572(0.015)&0.193(0.006)&0.556(0.016)&1.394(0.032)&0.288(0.004)\\ 
    Standalone&0.481(0.013)&0.153(0.004)&0.540(0.014)&1.581(0.083)&0.202(0.003)\\ 
    Cons&0.410(0.005)&0.154(0.012)&0.531(0.011)&2.314(0.091)&0.201(0.008)\\
    \midrule 
    Ours&0.611(0.009)&0.195(0.007)&0.581(0.014)&0.139(0.002)&0.302(0.005)\\
    \bottomrule 
    \end{tabular}
\end{table}

\begin{table}[!ht]
    \caption{Standard deviation of the online loss and (standard error over 5 runs) over all nodes. A lower value implies better equality. 
    }
    \centering
    \label{tab:std-online-loss-constraint}
    \begin{tabular}{c|ccccc}
    \toprule 
     & MNIST & CIFAR-10 & HFT & ELECTRICITY & PATH \\
     \midrule 
    FedAvg&4.8e-05(6.1e-06)&4.7e-06(1.5e-07)&7.1e-05(1.3e-05)&6.1e-05(3.8e-06)&1.9e-05(1.2e-06)\\
    qFFL&6.0e-06(4.2e-07)&3.7e-06(3.8e-07)&3.4e-08(2.1e-09)&7.6e-05(5.0e-06)&9.7e-07(1.0e-07)\\
    FGFL&2.5e-04(3.3e-05)&6.6e-06(1.3e-07)&8.9e-04(6.4e-04)&2.9e-04(5.6e-05)&4.3e-06(1.5e-06)\\
    GoG&1.9e-04(7.6e-06)&1.3e-05(1.3e-06)&1.2e-04(5.1e-06)&1.1e-04(3.7e-06)&1.2e-04(1.0e-05)\\
    Standalone&4.4e-03(2.8e-04)&1.5e-04(1.1e-05)&2.1e-03(9.6e-05)&4.4e-03(1.4e-04)&6.6e-04(1.3e-05)\\
    Cons&3.8e-05(3.1e-06)&2.2e-06(3.3e-07)&5.7e-05(7.0e-06)&4.7e-04(1.1e-04)&1.3e-05(1.3e-06)\\
    \midrule 
    Ours&4.1e-04(4.3e-05)&1.4e-05(1.4e-06)&2.1e-04(8.9e-05)&1.2e-04(7.1e-06)&2.9e-05(5.2e-06)\\    \bottomrule 
    \end{tabular}
\end{table}

From Table~\ref{tab:fairness-online-constraint}, Cons achieves very poor fairness compared to our approach since it does not have a fairness mechanism.

From Table~\ref{tab:online-acc-constraint}, Cons sacrifice performance by a lot (compared to FedAvg) due to its additional constraint in the objective function. 

From Table~\ref{tab:std-online-loss-constraint}, Cons achieves slightly better equality than our approach (better than ours in MNIST, CIFAR-10, HFT, worse than or similar to ours in ELECTRICITY and PATH). However, $q$-FFL outperforms Cons significantly in preserving equality. Additionally, according to the result in Tables A and B, Cons sacrifices the model performance considerably to achieve the marginal improvement of equality and achieves low fairness due to the lack of a fairness mechanism.

In conclusion, this simple extension can achieve equality reasonably well but unfortunately sacrifices two other important aspects, fairness and model performance. 

\clearpage
\subsection{Empirical Fairness vs.~Equality Trade-Off}

\cref{table:fairness-varying-beta-final-acc,table:fairness-varying-beta} show the empirical trade-off between fairness equality.
However, \cref{table:fairness-varying-beta-final-acc} calculates the standard deviation w.r.t.~\textbf{final test accuracy} while \cref{table:fairness-varying-beta} calculates the standard deviation w.r.t.~\textbf{online test accuracy}. The fairness results in both tables are the same.
Therefore, \cref{table:fairness-varying-beta-final-acc} shows the equality w.r.t.~the asymptotic model performance, instead of whether the models converge with the equal asymptotic complexities as guaranteed by \cref{prop:equal-convergence} and verified in \cref{table:fairness-varying-beta-feature-noise,table:fairness-varying-beta}.

\begin{table}[ht]
\setlength{\tabcolsep}{4pt}
\caption{Empirical trade-off between fairness and equality via $\beta$: Pearson coefficient $\rho$ between $\boldsymbol{\zeta}$ and online loss (standard deviation of final accuracy) under the setting of label noise.
High $\rho$ indicates fairness while lower standard deviation indicates better equality.
}
\label{table:fairness-varying-beta-final-acc}
\begin{center}
\begin{tabular}{c|ccccc}
 \toprule
 & MNIST & CIFAR-10 & HFT & ELECTRICITY & PATH \\
 \midrule
1/350&\textbf{0.713}(1.05e-02)&\textbf{0.555}(1.35e-02)&0.357(3.01e-04)&0.279(3.06e-03)&0.486(1.98e-02)\\
 1/150&0.678(2.2e-03)&0.455(3.2e-03)&0.347(3.4e-05)&\textbf{0.376}(9.85e-04)&\textbf{0.469}(2.27e-02)\\
1/100&0.657(3.13e-03)&0.361(5.43e-03)&0.357(8.45e-05)&0.323(5.71e-04)&0.316(1.30e-02)\\
1/50&0.427(3.56e-03)&0.182(3.14e-03)&\textbf{0.374}(\textbf{0})&0.268(\textbf{5.23e-04})&0.045(1.51e-02)\\
1/20&0.075(6.57e-03)&0.173(\textbf{2.58e-03})&0.154(\textbf{0})&0.133(5.28e-04)&0.033(1.04e-02)\\
1/10&0.015(\textbf{1.86e-03})&0.048(4.07e-03)&-0.027(\textbf{0})&0.214(6.02e-04)&0.045(6.43e-03)\\
1&-0.180(2.19e-03)&-0.033(5.98e-03)&-0.153(\textbf{0})&-0.048(5.92e-04)&-0.015(1.04e-02)\\
10&-0.192(2.44e-03)&0.109(4.58e-03)&-0.222(\textbf{0})&0.148(5.67e-04)&-0.096(\textbf{7.37e-03})\\
1000&-0.158(2.28e-03)&-0.019(6.08e-03)&-0.209(\textbf{0})&-0.006(5.94e-04)&-0.155(8.88e-03)\\
\bottomrule 
\end{tabular}
\end{center}
\end{table}

\begin{table}[!ht]
\setlength{\tabcolsep}{4pt}
\caption{Empirical trade-off between fairness and equality via $\beta$: Pearson coefficient $\rho$ between $\boldsymbol{\zeta}$ and online loss (standard deviation of online accuracy) under the setting of label noise.
High $\rho$ indicates fairness while lower standard deviation indicates better equality.
}
\label{table:fairness-varying-beta}
\begin{center}
\begin{tabular}{c|ccccc}
 \toprule
 & MNIST & CIFAR-10 & HFT & ELECTRICITY & PATH \\
 \midrule
1/350&\textbf{0.713}(1.75e-03)&\textbf{0.555}(2.62e-03)&0.357(8.03e-04)&0.279(1.34e-03)&\textbf{0.486}(4.35e-03)\\
 1/150&0.678(7.5e-04)&0.455(6.3e-04)&0.347(3.4e-04)&\textbf{0.376}(3.30e-04)&0.469(1.85e-03)\\
1/100&0.657(5.25e-04)&0.361(4.53e-04)&0.357(9.55e-04)&0.323(1.86e-04)&0.316(1.21e-03)\\
1/50&0.427(\textbf{1.94e-04})&0.182(\textbf{2.33e-04})&\textbf{0.374}(2.10e-05)&0.268(1.11e-04)&0.045(9.74e-04)\\
1/20&0.075(2.33e-04)&0.173(2.84e-04)&0.154(3.34e-04)&0.133(8.40e-05)&0.033(1.27e-03)\\
1/10&0.015(3.04e-04)&0.048(2.89e-04)&-0.027(6.77e-05)&0.214(8.66e-05)&0.045(1.09e-03)\\
1&-0.180(2.98e-04)&-0.033(3.06e-04)&-0.153(1.73e-04)&-0.048(\textbf{6.67e-05})&-0.015(\textbf{9.24e-04})\\
10&-0.192(2.45e-04)&0.109(2.96e-04)&-0.222(\textbf{4.76e-06})&0.148(7.46e-05)&-0.096(1.03e-03)\\
1000&-0.158(2.52e-04)&-0.019(2.58e-04)&-0.209(1.71e-05)&-0.006(8.07e-05)&-0.155(9.60e-04)\\
\bottomrule 
\end{tabular}
\end{center}
\end{table}

\subsection{Additional Results on Poor Fairness Performance from Inaccurate Contribution Estimates}
In this experiment, we have $N=10$ nodes on MNIST~(each with uniformly randomly selected $600$ images without noise) using the same CNN model as before. There is no data partitioning to simulate the online setting. In each iteration $t$, $20\%$ nodes are randomly sampled with probabilities directly proportional to $\boldsymbol{\psi}_{t}$, so their probabilities are dynamically updated with $\boldsymbol{\psi}_{t}$. The selected nodes synchronize their local models with the latest model~(as their incentives), conduct training and upload the updates. Importantly, $\boldsymbol{\psi}_{t}$ for only the selected nodes are evaluated and updated because the coordinator receives no updates from the other nodes. We plot $\boldsymbol{\psi}_t$ and the validation accuracy in \cref{fig:error-propagation}.

\begin{figure}[ht]
    \centering
    \includegraphics[width=0.35\linewidth]{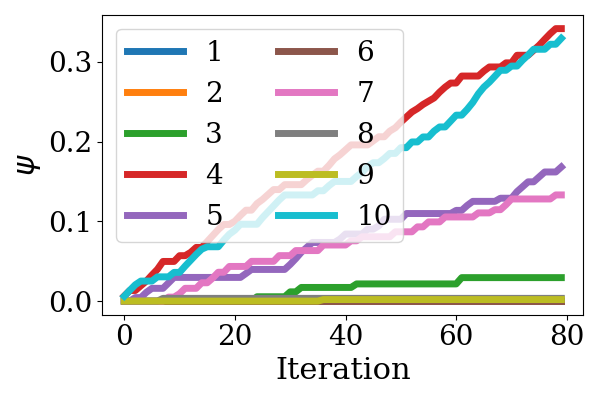}
    \includegraphics[width=0.35\linewidth]{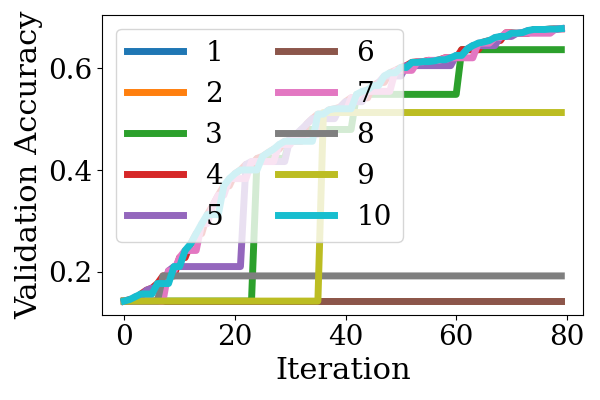}
    \caption{Contribution estimates~(left) and validation accuracy~(right) vs.~iterations $t$ under the paradigm which concurrently performs contribution evaluation and incentive realization~(updating the sampling probabilities).}
    \label{fig:error-propagation}
\end{figure}

Since the data are i.i.d.~without noise, the true contributions $\boldsymbol{\psi}^*$ are statistically equal. However, we observe an increasing separation in the contribution estimates $\boldsymbol{\psi}_t$. For instance, $\psi_{4,t}$ (red line for node $4$) is larger than $\psi_{9,t}$ (yellow line for node $9$) in the beginning, and the difference in $\psi_{4,t}, \psi_{9,t}$ is increasing over iterations.
As the true contributions should be approximately equal, the fair incentives should result in approximately equal model performance. However, the inaccuracy $\boldsymbol{\psi}_t$ affects the incentives as in \cref{fig:error-propagation}~(right). 
This further validates our approach that decouples contribution evaluation and incentive realization and first ensures the accuracy in the contribution estimates before constructing and realizing the incentives.

\clearpage
\section{Proofs and Derivations} \label{sec:appendix-proofs}
\subsection{Shapley Value Approximation} \label{sec:sv-approximation}
\textbf{A linear estimation to Shapley value.}
The contribution of node $i$ in iteration $t$ is computed as $\phi_{i,t} \coloneqq N^{-1} \sum_{\mathcal{S} \subseteq{[N]}\setminus \{i\}} \binom{N-1}{|\mathcal{S}|}^{-1} \mathbf{U}( \mathcal{S} \cup \{i\}) - \mathbf{U}( \mathcal{S})$. Here we focus on the exponential complexity that arises due to the summation i.e., $\sum_{\mathcal{S} \subseteq [N]\setminus \{i\}}$. This summation enumerates all $2^{N-1}$ coalitions that do not contain $i$. This complexity becomes infeasible even with a medium number of nodes $N$ so we need to provide an efficient estimation. For notational simplicity, we omit the subscript $t$ since we are referring to a particular iteration $t$.

Denote $\Delta_{i,\mathcal{S}} \coloneqq \mathbf{U}( \mathcal{S} \cup \{i\}) - \mathbf{U}( \mathcal{S})$. We use an unbiased estimator with time complexity linear in $N$ as follows:
\begin{equation}
    \label{equ:linear-estimator}
    \hat{\phi}_{i} \coloneqq  N^{-1} \textstyle \sum_{ m=0 }^{N-1} 
    \Delta_{i,\mathcal{S}_m}\ ,
    \quad \mathcal{S}_m\sim \{  \mathcal{S} \}_{\mathcal{S}\subseteq [N]\setminus\{i\}: |\mathcal{S}|=m}
\end{equation}
where $\mathcal{S}_m$ is a uniformly randomly sampled coalition of size $m$.
The Shapley value computes an average of the expected marginal contribution 
that $i$ makes to a coalition $\mathcal{S}_m$ of size $m$. 
The estimator $\hat{\phi}_i$ computes an average of the marginal contribution that $i$ makes to a \emph{randomly} selected coalition $\mathcal{S}_m$ of size $m$. 

\begin{proposition}[\bf Unbiased Estimator for SV]\label{prop:unbiased-SV-estimator}
Assume the $i$'s marginal contribution to a random coalition of size $m$, $\Delta_{i,\mathcal{S}_m}$ has mean and variance $\mu_m, \sigma_m^2$, then $\hat{\phi}_i$ is an unbiased estimator for $\phi_i$ with variance $\text{Var}( \hat{\phi}_i ) = \sum_{m=0}^{N-1} (\sigma_m/N)^2$. 

\end{proposition}

The assumption placed on $\Delta_{i,\mathcal{S}_m}$~\citep{Fatima2008-linear-SV-approximation,Rozemberczki2021-ensemble-games} states that the average effect node $i$ makes depends on the size of coalition that $i$ joins. In our scenario, it means the relative value of $i$'s gradient in improving the coordinator model depends on how many others gradients have already been applied to the coordinator model. Intuitively, if many other nodes' uploaded gradients have already been applied to the coordinator model, the relative value of $i$'s gradient in improving the coordinator model further may be limited. On the other hand, if no gradient has been applied to the coordinator model and applying $i$'s gradient improves the coordinator model, then node $i$ can be viewed as making valuable contributions.

\subsection{Proof of \cref{prop:hypothesis-testing}, Normality Test for $\phi$, and Empirical Verification on Independence of $\phi$}\label{appendix:normality-assumption-verification}

\begin{proof}[Proof of \cref{prop:hypothesis-testing}]
With a normality assumption (discussed below): $\{\boldsymbol{\phi}_{t}\}_{t=1,\ldots, t_s}$ are i.i.d.~samples from an $N$-dimensional multivariate Gaussian $\mathcal{N}(\boldsymbol{\mu}_{t_s}, \Sigma)$, the statistic \texttt{T2} follows Hotelling's $T$-squared distribution $T^2_{\mathcal{N}, 2(t_s-1)}$ \cite{hotelling} when the null hypothesis~(i.e.,  $\boldsymbol{\mu}_{t_s} = \boldsymbol{\mu}_{t_{s'}}$) is true~\cite{hotelling}. 

Therefore, we can reject the null hypothesis with at most $\alpha$ type-$1$ error if $\texttt{T2} \geq T^2_{1-\alpha, \mathcal{N}, 2(t_s - 1)}$, the $1-\alpha$ quantile for the distribution $T^2_{\mathcal{N}, 2(t_s-1)}$. 
\end{proof}

\begin{remark}
Contrast the hypothesis testing in \cref{prop:hypothesis-testing} with how hypothesis testing is more commonly used where we tend to reject the null hypothesis. For instance, in testing whether the parameters $\vartheta$ in a linear regression are $0$, we usually expect the null hypothesis $h_0:\vartheta =0$ to be rejected, so the $p$-value and the corresponding significance level $\alpha$ are often small~(e.g., set to $0.05$) for rejecting $h_0$. 
However, in our formulation, we expect $h_{0,t_s}$ to hold (not to be rejected) over more iterations, so we set larger $\alpha$ values (e.g., $0.5$).
\end{remark}

\paragraph{On the normality assumption.}
For notational simplicity, we suppress the subscript $t$ in this theoretical analysis since we are focusing on an arbitrary iteration $t$.

Define the marginal contribution from node $i$ to a subset/coalition $\mathcal{S}\subseteq [N]\setminus \{i\}$ as $\text{MC}_{i,\mathcal{S}} \coloneqq \mathbf{U}(\mathcal{S} \cup \{i\}) - \mathbf{U}(\mathcal{S})$.
Note that the linearity of inner product as $\mathbf{U}$ is useful in decomposing terms to compute the marginal contribution as follows:
\begin{align} \label{equ:normality-of-inner-product}
    \text{MC}_{i,\mathcal{S}} & \triangleq \left\langle \Delta \theta_{\mathcal{S}\cup \{i\}}\ , \Delta \theta_{[N]}  \right\rangle -  \left\langle \Delta \theta_{\mathcal{S}}\ , \Delta \theta_{[N]}  \right\rangle \notag\\
    &= \left\langle \sum_{i' \in \mathcal{S} \cup\{i\}} p_{i'}\ \Delta \theta_{i'}\ ,  \sum_{i' \in [N]} p_{i'}\ \Delta \theta_{i'} \right\rangle - \left\langle \sum_{i' \in \mathcal{S}} p_{i'}\ \Delta \theta_{i'}\ ,  \sum_{i' \in [N]} p_{i'}\ \Delta \theta_{i'} \right\rangle \notag\\
    &=  \left\langle p_i\ \Delta \theta_i\ ,  \sum_{i' \in [N]} p_{i'}\ \Delta \theta_{i'}  \right\rangle  \notag\\
    &= p_i \sum_{i'\in[N]} p_{i'} \left\langle \Delta \theta_i\ , \Delta\theta_{i'} \right\rangle\ .
\end{align}
Next, we focus on arguing for the normality of $\langle \Delta \theta_i, \Delta \theta_{i'} \rangle $ since the factors $p_i, p_{i'}$ are constants and the summation does not affect the normality.
Recall $\Delta \theta_i$ is an empirical mean estimate via a randomly selected mini-batch $\mathcal{B}_i$ of size $B$ as follows: 
\[\Delta \theta_i \triangleq \frac{1}{B}\sum_{(x,y) \in \mathcal{B}_i} \Delta \theta_{i, (x,y)} \]
where $\Delta \theta_{i, (x,y)}$ denotes the single-sample gradient/model update w.r.t.~input data $(x,y)$ on the selected loss function and the model from the previous iteration (omitted for notational simplicity). Then,
\[
\langle \Delta \theta_i, \Delta \theta_{i'} \rangle = \frac{1}{B^2} 
\sum_{ \substack{(x,y) \in \mathcal{B}_i\ ,\\ (x',y') \in \mathcal{B}_{i'} }} \langle \Delta \theta_{i, (x,y)}\ , \Delta \theta_{i', (x',y')} \rangle 
\]
is an empirical mean estimate dependent on the joint distribution of $(x,y), (x',y')$ (since the loss function is fixed and the model parameter from the previous iteration is also fixed). Note that in the case of $i=i'$, $(x,y) = (x',y')$, but $\langle \Delta \theta_i, \Delta \theta_{i'} \rangle$ is an empirical mean estimate nonetheless. 

Consequently, we can apply the \textit{central limit theorem} so that as $B \to \infty$, $\langle \Delta \theta_i, \Delta \theta_{i'} \rangle$ follows a normal distribution. Since the linear scaling and summation in \cref{equ:normality-of-inner-product} does not affect normality, $\text{MC}_{i,\mathcal{S}}$ also follows a normal distribution. Subsequently, since $\phi_{i}$ is a linear combination of $\text{MC}_{i,\mathcal{S}}$ with fixed constants, $\phi_{i}$ also follows a normal distribution. 

Regarding the independence assumption over iterations, since each $\phi_{i,t}$ is calculated using the model updates from that iteration, it is independent of previous $\phi_{i,t'<t}$ conditioned on the model parameter $\theta_{t-1}$. 
Regarding the assumption on identical distribution, we can view $\phi_{i,t}$ in each $t$ is from some distribution that depends on node $i$'s true contribution.

\paragraph{Normality test and some implementation details.}

In experiments (i.e., implementation), we use cosine similarity $\text{cos-sim}(a,b) \coloneqq \langle a,b \rangle/(\Vert a \Vert \times \Vert b \Vert)$ (instead of inner product) for $\mathbf{U}$ as the utility function in \cref{equ:cumulative-sv-psi} for the following practical considerations:
\textbf{(i)} cosine similarity empirically performs well as the additional normalization (compared to only inner product) can be used to mitigate the vanishing/exploding gradient/model update issue~\cite{Xu2021-fair-CML}; \textbf{(ii)} the normalization in cosine similarity makes the analytic argument difficult (see below), but in practice this small linear scaling does not seem to violate normality, verified below.

In addition, since in practice we cannot have $B \to \infty$, we perform the Henze-Zirkler test~\cite{Henze1990ACO} on the normality of $\{\boldsymbol{\phi}_t \}$ in \cref{table:normality_test_appendix}.
The null hypothesis for this test is that $\{\boldsymbol{\phi}_{t}\}_{t=1,\ldots, t_s}$ are from a multivariate Gaussian distribution, and we reject the null hypothesis if the $p$-value is less than $0.05$. 
We take the samples from a fixed window size of iterations as $\{\boldsymbol{\phi}_{t}\}_{t= t_s,\ldots, t_s+\tau}$ where $\tau = 50$ and conduct multiple tests from a single trial and additionally perform $10$ random trials. Each experiment is conducted on a batch size of $32$, and other settings are the same as \cref{sec:setting} with label noise.
We perform multiple random trials and show the percentage of trials of \textit{not} rejecting the null hypothesis in \cref{table:normality_test_appendix}. The values in the brackets are the average $p$-values (all above $0.05$).

\begin{table}[!ht]
\setlength{\tabcolsep}{4pt}
\caption{The percentage of random trials of \textit{not} rejecting the null hypothesis for the Henze-Zirkler test on the normality of $\{\boldsymbol{\phi}_{t}\}_{t=1,\ldots, t_s}$. The numbers in the brackets are the average $p$-values.
}
\label{table:normality_test_appendix}
\begin{center}
\begin{tabular}{c|ccc}
 \toprule
 & Label Noise & Feature Noise\\
 \midrule
MNIST & 0.949 (0.171)& 0.984 (0.194)\\
CIFAR-10 &0.930 (0.115)& 0.924 (0.114)\\
\bottomrule
\end{tabular}
\end{center}
\end{table}

Lastly  we describe the difficulties in providing an analytical argument for normality if we use cosine similarity as $\mathbf{U}$. Specifically, the difficulty stems from the division by $ \Vert a \Vert \Vert b \Vert$ where $a =\Delta \theta_{\mathcal{S}}, b=\Delta \theta_{[N]}$. We focus on $\Delta \theta_{\mathcal{S}}$.
Suppose $\Delta \theta_{\mathcal{S}} \sim \mathcal{N}(\boldsymbol{\mu}, \Sigma)$ for some unknown $\boldsymbol{\mu}, \Sigma$. If $\Sigma = \sigma^2 \mathbf{I}$ (i.e., isotropic), then $\Vert \Delta \theta_{i,t} \Vert$ follows a noncentral chi distribution, which seems to suggest the overall expression (from cosine similarity) $\langle \Delta \theta_{\mathcal{S}},\Delta \theta_{[N]}\rangle / ( \Vert \Delta \theta_{\mathcal{S}} \Vert  \Vert \Delta \theta_{[N]}  \Vert)$ may not be normal even if the numerator is normal (as discussed above w.r.t.~inner product).
Furthermore, $\mathbb{E}[\Delta \theta_{\mathcal{S}}]$ has a complicated expression~\citep[Theorem 3.2b.5]{Mathai1994-non-normality} which can further suggest the difficulties in deriving an analytic expression.

\textbf{Empirical verification on independence of $\phi$ from different iterations.} An additional verification on the independence of $\phi_{i,t}$ i.e. the contribution estimate of node $i$ in iteration $t$ is provided. The experiment setting can be found in \cref{sec:foil}. We compute the autocorrelation function (ACF) for the series of contribution estimates of a specific node $i$ $\{\phi_{i,t}\}_{t=0,\ldots, T_{\alpha}}$. The ACF gives the correlation between the series and a $k$ step lag of the series. If a series has no significant non-zero autocorrelation value in all $k \ge 1$, it means that every step of the series is uncorrelated to each other, which can serve as an empirical evidence of independence. \cref{fig:sv_acf_plot_appendix} presents the ACF plot for different dataset under the setting of label noise. And most autocorrelation values lie within the $95\%$ confidence interval from $0$, which means that no significant correlation is found in all the series. Even though some series have non-zero values at lag $k=1$, the correlation decreases to zero quickly with a larger lag ($k\ge2$). This result verifies the rationality of our assumption by empirically showing that the contribution estimates for a specific node from different iterations are uncorrelated (empirically independent) in general.

\begin{figure}[!htb]
    \centering 
  \includegraphics[width=0.195\textwidth]{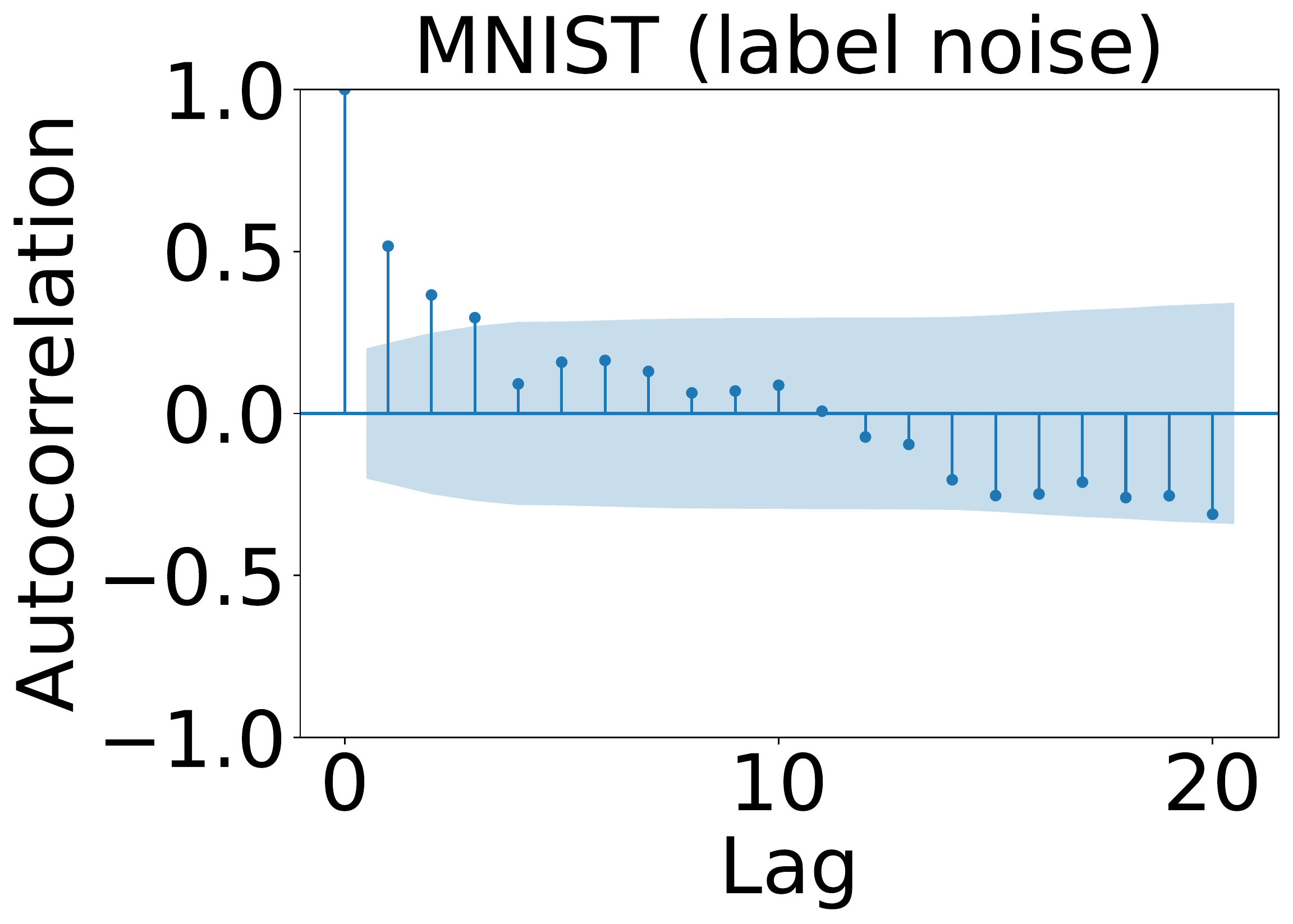}
  \includegraphics[width=0.195\textwidth]{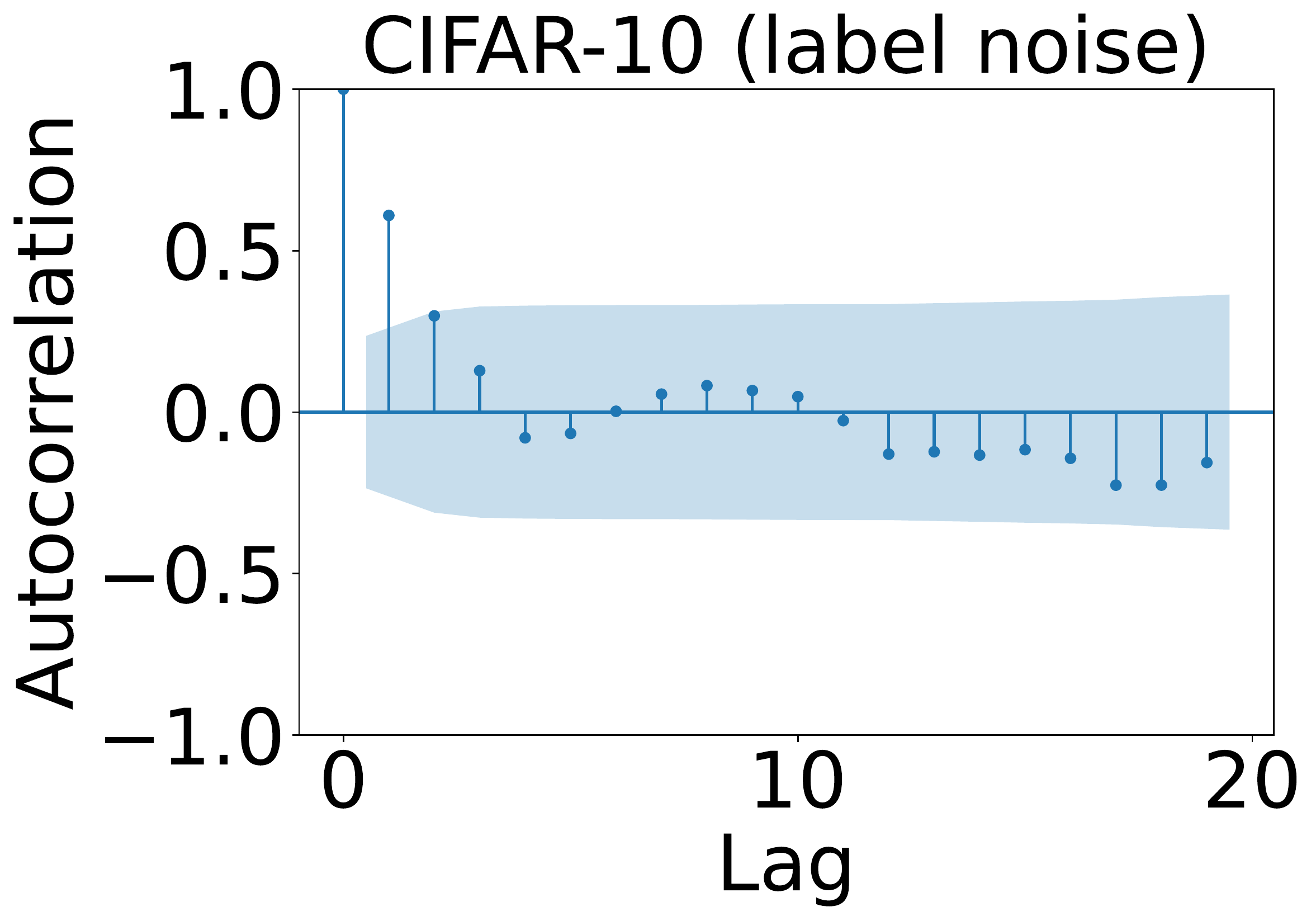}
  \includegraphics[width=0.195\textwidth]{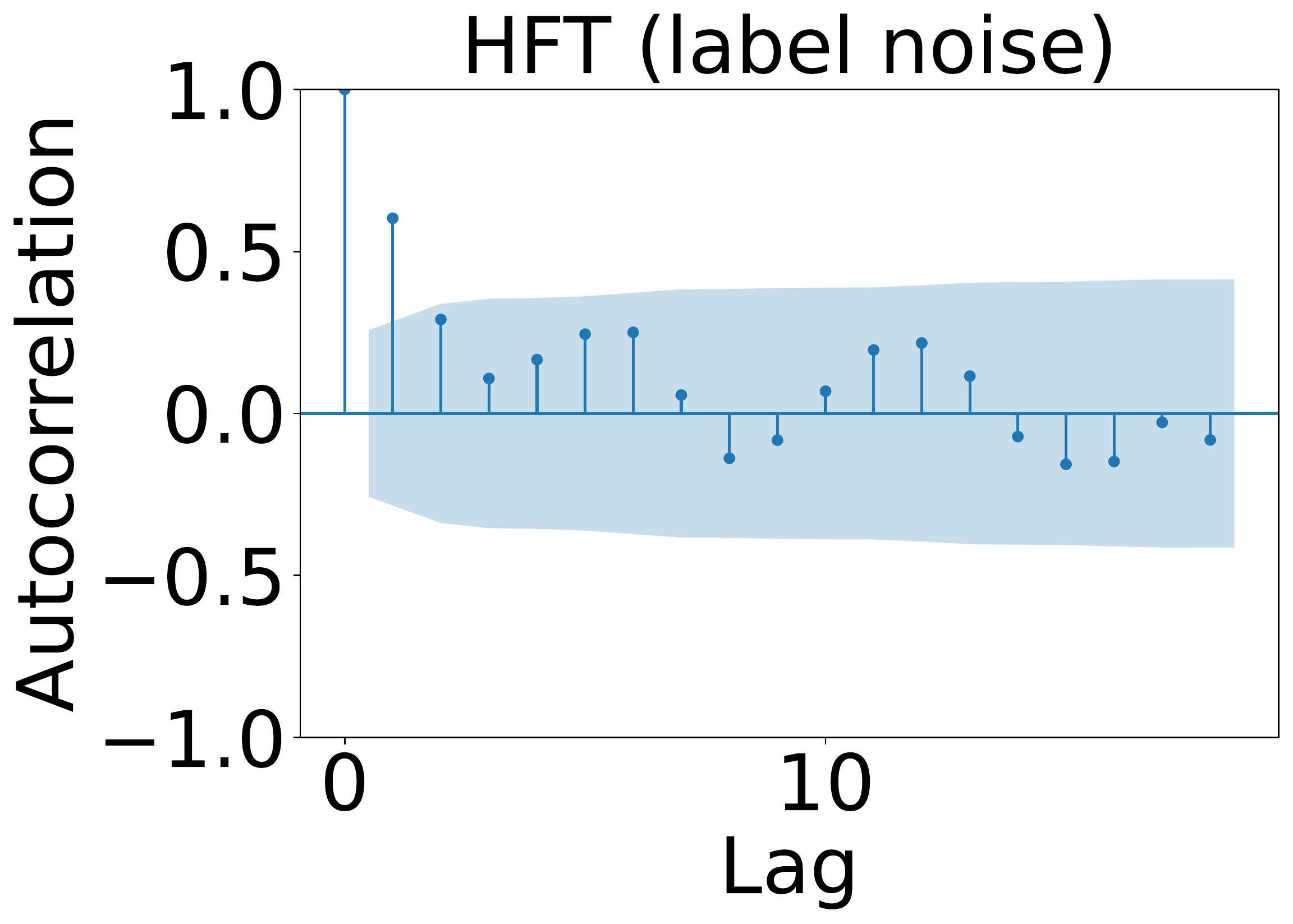}
    \includegraphics[width=0.195\textwidth]{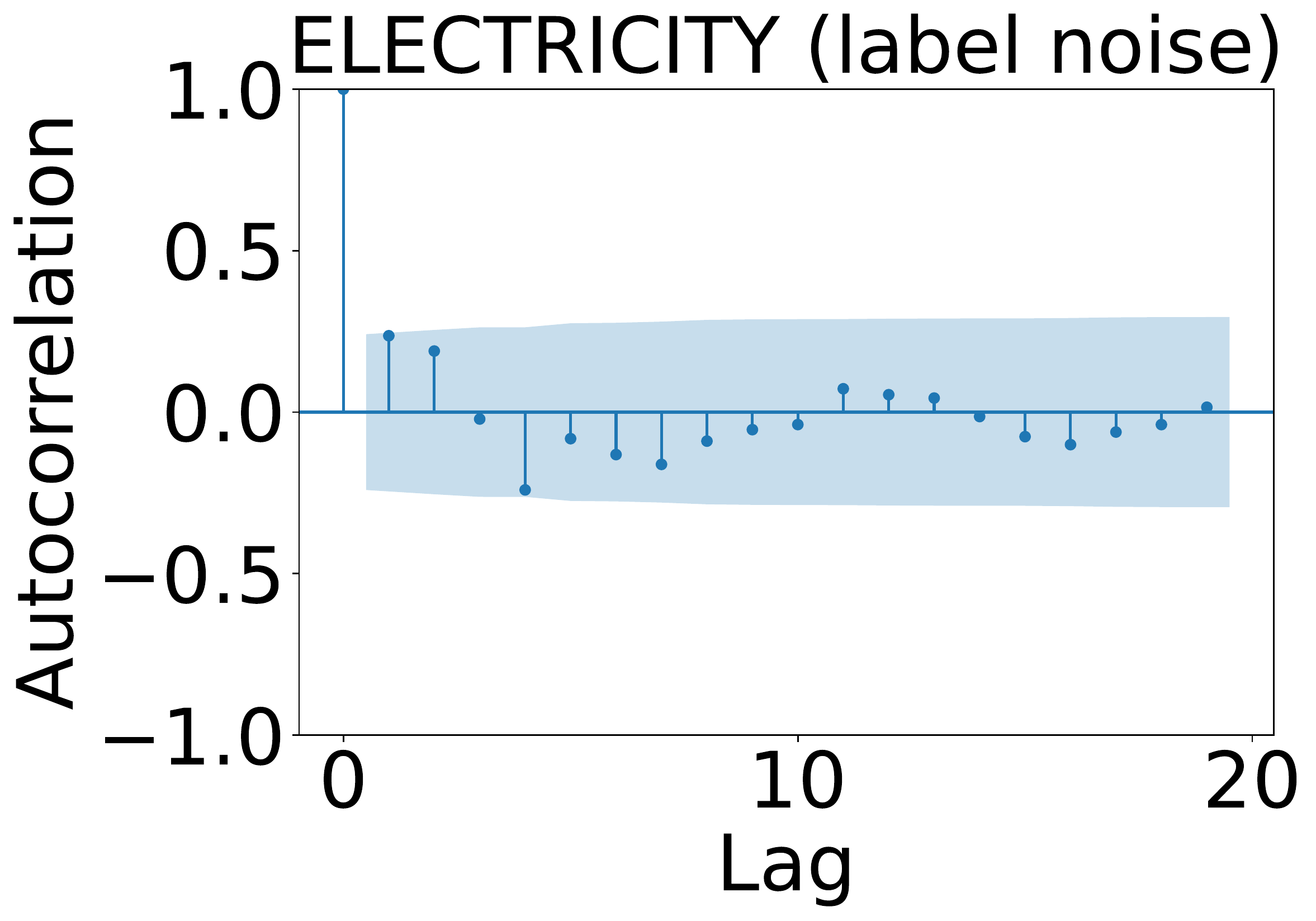}
     \includegraphics[width=0.195\textwidth]{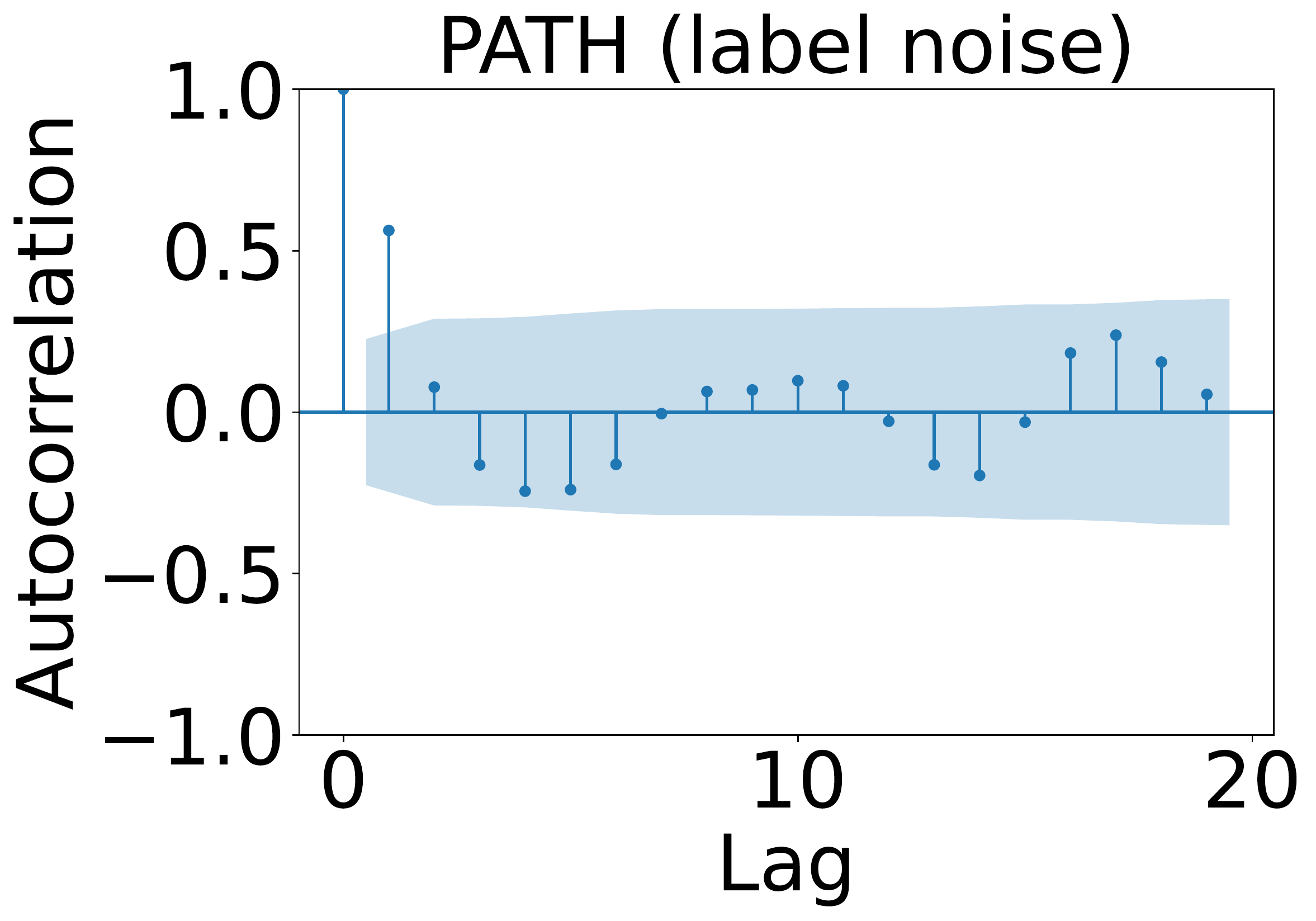}
  \caption{Autocorrelation function (ACF) plot with different lags of the series $\{\phi_{29,t}\}_{t=0, \ldots, T_{\alpha}}$, i.e., the contribution estimates series of node $29$ (the least contribution node). If the autocorrelation value lies within the $95\%$ confidence interval (blue shadow area), it means that we do not have enough evidence to conclude that the autocorrelation value is non-zero (no significant correlation is found in the series) for a specific lag.
    }
    \label{fig:sv_acf_plot_appendix}
\end{figure}

\subsection{Proof of \cref{proposition:fairness}}

We first derive the closed-form expression for the expected staleness $\Gamma_i$ and then show the proof of \cref{proposition:fairness}.

\begin{proof}[Derivation of expected staleness $\Gamma_i$]
We substitute $\mathbb{P}[\text{stale for } \gamma \text{ iterations}] = (1-q_i)^\gamma$ and apply the formula for geometric series. The derivation consists of two parts: (a) Show that $\Gamma_i$ is a convergent series, and (b) determine the value it converges to.

(a) By ratio test, we require 
\[\lim_{n\to \infty} (1-q_i) \frac{n+1}{n}  < 1 \]
which is satisfied if $q_i < 1$.

(b) Next, we determine what $\Gamma_i$ converges to. Observe the following:
\[ \sum_{n=0}^{\infty} x^n n = x \sum_{n=0}^{\infty} x^{n-1} n  = x \frac{\text{d}}{\text{d}x} \sum_{n=0}^{\infty} x^n =x \frac{\text{d}}{\text{d}x} \frac{1}{1-x} = \frac{x}{(1-x)^2}\ .\]
Substituting $x = (1-q_i)$ completes the derivation.
\end{proof}

\begin{proof}[Proof of \cref{proposition:fairness}]
Firstly, note that the comparison between $C_i$ and $C_{i'}$ boils down to that between the expected staleness $\Gamma_i$ and $\Gamma_{i'}$ since $C_i$ and $C_{i'}$ share the other term of $\mathcal{O}(1/\epsilon)$. Subsequently, we use the derived the closed-form expression for $\Gamma_i$ as above.

Symmetry follows directly. Strict desirability and strict monotonicity follow by observing $\Gamma_i$ is monotonic in $\psi_{i, T_{\alpha}}$: a larger $\psi_{i, T_{\alpha}}$ implies a smaller $\Gamma_i$.
\end{proof}

Before we describe and interpret when the conditions \cref{proposition:fairness} are satisfied (next), we first note that we exclude the null player property as we assume positive contributions from the nodes that represent companies/organizations and are unlikely to have zero contributions (e.g., free-riders) due to their reputations being at stake. Empirically, we observe $\psi_{i,t}$ to be generally positive (\cref{fig:sv_plot_appendix} in \cref{appendix:experiments}). 

We can view each iteration $t$ as a game of $N$ nodes and the $\phi_{i,t}$ is node $i$'s SV in this game, so $\psi_{i,t}$ is the average of node $i$'s SV in $t$ such games and intuitively represents node $i$'s aggregate contributions. In particular, we can define $\mathbf{U}_t(\mathcal{S}) \coloneqq \mathbf{U}(\Delta \theta_{\mathcal{S},t}, \Delta \theta_{[N],t})$ where the dependence on $t$ is from the $\Delta \theta_{\mathcal{S},t}$. With this formulation, we have 
$\phi_{i,t} \triangleq \phi_i(\mathbf{U}_t)$ and can subsequently exploit the linearity in SV:
\[ \phi_i(\mathbf{U} + \mathbf{U}') = \phi_i(\mathbf{U}) + \phi_i(\mathbf{U'})\ .\]

First, we focus on the game defined by $\mathbf{U}_t$ (i.e., in iteration $t$) and analyze the symmetry, strict desirability and strict monotonicity guaranteed by the original SV w.r.t.~$\phi_i(\mathbf{U}_t)$:
\begin{enumerate}
    \item symmetry \cite{shapley1953value}: $\forall i,i', (\forall \mathcal{S}\subseteq [N]\setminus \{i,i'\}, \mathbf{U}_t(\mathcal{S}\cup\{i\}) =\mathbf{U}_t(\mathcal{S}\cup\{i'\})    ) \implies \phi_i(\mathbf{U}_t) = \phi_i'(\mathbf{U}_t) $.
    \item strict desirability \cite{Bahir1966-desirability}:
    $\forall i,i', (\forall \mathcal{S}\subseteq [N]\setminus \{i,i'\}, \mathbf{U}_t(\mathcal{S}\cup\{i\})  \geq  \mathbf{U}_t(\mathcal{S}\cup\{i'\})   ) \wedge 
    (\exists \mathcal{P}\subseteq [N]\setminus \{i,i'\}, \mathbf{U}_t(\mathcal{P}\cup\{i\}) >  \mathbf{U}_t(\mathcal{P}\cup\{i'\})  ) \implies \phi_i(\mathbf{U}_t) > \phi_i'(\mathbf{U}_t) $.
    \item strict monotonicity \cite{Young1985-monotonicity}: 
    for two $\mathbf{U}_t, \mathbf{U}_t'$, 
    $( \forall \mathcal{S}\subseteq [N]\setminus \{i\}, \mathbf{U}_t(\mathcal{S}\cup\{i\}) \geq  \mathbf{U}'_t(\mathcal{S}\cup\{i\}) ) \wedge ( \exists \mathcal{P}\subseteq [N]\setminus \{i\}, \mathbf{U}_t(\mathcal{S}\cup\{i\}) >  \mathbf{U}'_t(\mathcal{P}\cup\{i\}) )  \implies \phi_i(\mathbf{U}_t) > \phi_i(\mathbf{U}'_t) $.
\end{enumerate}

Then by exploiting the linearity property, we can derive the corresponding sufficient conditions in \cref{proposition:fairness} in a straightforward way: adding and taking average of all the $\mathbf{U}_t$ and $\phi_{i,t}$ up to $t = T_{\alpha}$ to obtain $\psi_{i,T_{\alpha}}$.

This affords us the following simple, albeit somewhat restrictive interpretations. For symmetry, if two nodes $i,i'$ contribute to \textit{all} possible $\mathcal{S} \subseteq [N]\setminus \{i,i'\}$ equally over \textit{all} iterations up to $T_\alpha$ (so that $\phi_{i,t} = \phi_{i',t}$), then their contribution estimates are equal, $\psi_{i, T_\alpha} =\psi_{i',T_\alpha}$.
The interpretation for strict desirability is similar by replacing the equality relationship with a greater than relationship.

The interpretation for strict monotonicity is only slightly different as it concerns with two different $\mathbf{U}_t, \mathbf{U}_t'$: if a node $i$ does something (e.g., makes a better contribution) in each iteration $t$ to improve its contribution  
(i.e., $\phi_{i,t} \coloneqq \phi_i (\mathbf{U}_t) > \phi'_{i,t} \coloneqq \phi_i (\mathbf{U}'_t)$), then the corresponding overall contribution estimate is also improved (i.e., $\psi_{i, T_\alpha} > \psi'_{i,T_\alpha}$).

These interpretations, while theoretically straightforward, are fairly restrictive due to the two \textit{for all} clauses ($\forall \mathcal{S} \subseteq [N]\setminus \{i,i'\}$ and $\forall t \leq T_{\alpha}$). In fact, for the sufficient conditions in \cref{proposition:fairness}, there is some leeway. 

To illustrate, node $i$ contributes more in some iteration $t$ while node $i'$ contributes more in some other iteration $t'$, as long as the difference in their contributions `balances out' when averaged over $T_{\alpha}$ iterations, the symmetry property is applicable in \cref{proposition:fairness}. Similar interpretations are available for the strict desirability and strict monotonicity. Intuitively, \cref{proposition:fairness} says to receive a good incentive/low convergence complexity, a node needs to do well/make high contributions in aggregate/overall (over $T_\alpha$ iterations) and one such (restrictive) possibility is the node makes a high contribution in \textit{all} the iterations.

\paragraph{Comparison with FGFL~\cite{Xu2021-fair-CML}.}
Contrast this with~\citep[Theorem 2]{Xu2021-fair-CML}, our result explicitly guarantees fair model convergence~(to global optimum), empirically compared in \cref{fig:cml_vs_ours_convergence_path}. Their result provides fairness restricted to each iteration $t$ instead of overall model performance or convergence. Furthermore, their fairness result requires additional regularity conditions on the models and the objective function, which may be difficult to verify w.r.t.~complicated models such as deep neural networks.

\subsection{Honest Nodes Assumption and Its Relaxation}
\label{sec:appendix-optimal-strategy-proofs}

\paragraph{Honest nodes assumption.} Honest nodes are assumed not to deviate from the proposed algorithm. In our context, honest nodes do not have strategic behavior that exploits our algorithm. In contrast, a dishonest node may try to exploit the algorithm by uploading contributions honestly in the exploration phase (when their contributions are evaluated) and deliberately lowering contributions in the exploitation phase (e.g., via uploading random/zero values). The motivation for such dishonest behaviors is that once the exploration phase is over, the contributions of the nodes are no longer evaluated and the nodes are rewarded subsequently according to the latest contribution estimates (which are fixed after the exploration phase). This assumption is plausible for cross-silo FL where each node represents a company or organisation (e.g., hospitals) where the dishonest and exploitative behavior can damage their reputations, especially if they are in a long-term collaboration with several partners.
Nevertheless, we provide some additional discussion by relaxing this assumption and argue that honesty is the optimal strategy under some mild conditions.

\paragraph{Proof of optimal strategy for nodes and empirical verification.}
Assume that the utility function $\mathbf{U}$ in \cref{sec:setting} is a submodular function. For example, when using the test accuracy or negative log-likelihood as the utility function, the submodularity of the utility function means that as the size of coalition increases the improvement in the performance of an ML model will be smaller and smaller~\cite{Wang2021-learnability}. Denote node $i^s$, that potentially has positive contribution, is the dishonest node that tries to predict the stopping iteration of contribution evaluation as $T_{\text{pred}}$. 
Denote the contribution estimate for node $i$ in iteration $t$ \emph{without} the exploitative behavior of stopping contributions as $\phi^n_{i,t}$, and the contribution estimate for \emph{with} such exploitative behavior for $i^s$ stopping contributing after its predicted iteration as $\phi^s_{i,t}$. Denote $T_{\text{next}}=\min(T_{\alpha}, T_{\text{pred}}+1)$ and $C_{i^s, T_{\text{next}}}^s$ ($\rho_{i^s,T_{\text{next}}}^s$) as the convergence complexity (selection probability) for $i^s$ in the setting where the node $i^s$ stop contributing after $T_{\text{pred}}$ and the coordinator stops the contribution evaluation in $T_{\text{next}}$. Similarly, $C_{i^s,T_{\text{next}}}^n$ ($\rho_{i^s,T_{\text{next}}}^n$) is for $i^s$ without stopping contributions (similar notations for other quantities under these two settings: stop contributing vs.~not stop contributing).
\begin{proposition}[\bf Optimal Strategy for Nodes]
\label{prop:optimal-strategy}
With a submodular utility function $\mathbf{U}$ and suppose $\phi^n_{i^s,t} > 0$ for all $t >0$. Then, $\mathbb{E}[\rho_{i^s,T_{\text{next}}}^s] \le \mathbb{E}[\rho_{i^s,T_{\text{next}}}^n]$ and $\mathbb{E}[C_{i^s,T_{\text{next}}}^s] \ge \mathbb{E}[C_{i^s,T_{\text{next}}}^n]$. Equality holds when $\sum_{T_{\alpha}, T_{\text{pred}}: T_{\alpha} > T_{\text{pred}}} p(T_{\alpha}, T_{\text{pred}}) =0$.

\begin{proof}[Proof of \cref{prop:optimal-strategy}]
Since utility function $\mathbf{U}$ is a submodular function, i.e., for every $\mathcal{S}_1,\mathcal{S}_2 \subseteq [N]$ with $\mathcal{S}_1 \subseteq \mathcal{S}_2$ and every $i \in [N]\setminus \mathcal{S}_2$, $\mathbf{U}(\mathcal{S}_1 \cup {i}) - \mathbf{U}(\mathcal{S}_1) \ge \mathbf{U}(\mathcal{S}_2 \cup {i}) - \mathbf{U}(\mathcal{S}_2)$. Denote the contribution estimate for non-stop contribution setting as $\boldsymbol{\psi}_t^n = \{\psi_{i,t}^n\}_{i \in [N]}$ and the contribution estimate for the setting with node $i^s$ stopping to contribute after its predicted iteration as $\boldsymbol{\psi}_t^s = \{\psi_{i,t}^s\}_{i \in [N]}$.
For node $i^s$, we assume that it is a normally behaved node with positive contribution in the non-stop contribution setting i.e., $\phi_{i^s,t}^n > 0$, while it makes no contribution in the stopping setting after $T_{\text{pred}}$, i.e., $\phi_{i^s, t}^s = 0$ for all $t > T_{\text{pred}}$. In our framework, the node can simply upload gradient with all $0$ elements to simulate a $0$ contribution node. Denote $T_{\text{next}} \coloneqq \min(T_{\alpha}, T_{\text{pred}}+1)$. Then, when $i \in [N]\setminus\{i^s\}$,
\begin{equation*}
    \begin{aligned}
    &\mathbb{E}_{T_{\alpha}, T_{\text{pred}}}\left[\psi_{i, T_{\text{next}}}^s\right] \\
    &=  \mathbb{E}_{T_{\alpha}, T_{\text{pred}}}\left[\frac{1}{T_{\text{next}}}\sum_{t=1}^{T_{\text{next}}}\phi_{i,t}^s\right] \\
    &= \sum_{T_{\alpha}, T_{\text{pred}}: T_{\alpha} \le T_{\text{pred}}} p(T_{\alpha}, T_{\text{pred}})\frac{1}{T_{\alpha}}\sum_{t=1}^{T_{\alpha}}\phi_{i,t}^s + \sum_{T_{\alpha}, T_{\text{pred}}: T_{\alpha} > T_{\text{pred}}} p(T_{\alpha}, T_{\text{pred}})\frac{1}{T_{\text{pred}}+1}\sum_{t=1}^{T_{\text{pred}}+1}\phi_{i,t}^s \\
    &= \sum_{T_{\alpha}, T_{\text{pred}}: T_{\alpha} \le T_{\text{pred}}} p(T_{\alpha}, T_{\text{pred}})\frac{1}{T_{\alpha}}\sum_{t=1}^{T_{\alpha}}\phi_{i,t}^n + \sum_{T_{\alpha}, T_{\text{pred}}: T_{\alpha} > T_{\text{pred}}} p(T_{\alpha}, T_{\text{pred}})\frac{1}{T_{\text{pred}}+1}\left[\phi_{i,T_{\text{pred}}+1}^s + \sum_{t=1}^{T_{\text{pred}}}\phi_{i,t}^n\right] \\
    &\ge \sum_{T_{\alpha}, T_{\text{pred}}} p(T_{\alpha}, T_{\text{pred}})\frac{1}{T_{\text{next}}}\sum_{t=1}^{T_{\text{next}}}\phi_{i,t}^n \\
    &= \mathbb{E}_{T_{\alpha}, T_{\text{pred}}}\left[\psi_{i, T_{\text{next}}}^n\right] 
    \end{aligned}
\end{equation*}
where the inequality holds because when $T_{\text{pred}} < T_{\alpha}$ and $t=T_{\text{pred}}+1$, 
\begin{equation*}
    \begin{aligned}
    \phi_{i, t}^s & = \frac{1}{N} \sum_{\mathcal{S}^s \subseteq [N]\setminus \{i\}} \binom{N-1}{|S^s|}^{-1}\mathbf{U}(\mathcal{S}^s\cup \{i\}) - \mathbf{U}(\mathcal{S}^s)\\
    & = \frac{1}{N} \sum_{\mathcal{S}^s \subseteq [N]\setminus \{i\}} \binom{N-1}{|S^s|}^{-1}\mathbf{U}(\mathcal{S}^s\setminus \{i^s\}\cup \{i\}) - \mathbf{U}(\mathcal{S}^s\setminus \{i^s\})\\
    & = \frac{1}{N} \sum_{\mathcal{S}^n \subseteq [N]\setminus \{i\}} \binom{N-1}{|S^n|}^{-1}\mathbf{U}(\mathcal{S}^n\setminus \{i^s\}\cup \{i\}) - \mathbf{U}(\mathcal{S}^n\setminus \{i^s\})\\
    & \ge \frac{1}{N} \sum_{\mathcal{S}^n \subseteq [N]\setminus \{i\}} \binom{N-1}{|S^n|}^{-1}\mathbf{U}(\mathcal{S}^n\cup \{i\}) - \mathbf{U}(\mathcal{S}^n)\\
    & = \phi_{i,t}^n\ .
    \end{aligned}
\end{equation*}    
For node $i^s$,
\begin{equation*}
    \begin{aligned}
    &\mathbb{E}_{T_{\alpha}, T_{\text{pred}}}\left[\psi_{i^s, T_{\text{next}}}^s\right] \\
    &=\mathbb{E}_{T_{\alpha}, T_{\text{pred}}}\left[\frac{1}{T_{\alpha}}\sum_{t=1}^{T_{\alpha}}\phi_{i^s,t}^s\right] \\
    &=\sum_{T_{\alpha}, T_{\text{pred}}: T_{\alpha} \le T_{\text{pred}}} p(T_{\alpha}, T_{\text{pred}})\frac{1}{T_{\alpha}}\sum_{t=1}^{T_{\alpha}}\phi_{i^s,t}^s + \sum_{T_{\alpha}, T_{\text{pred}}: T_{\alpha} > T_{\text{pred}}} p(T_{\alpha}, T_{\text{pred}})\frac{1}{T_{\text{pred}}+1}\sum_{t=1}^{T_{\text{pred}}+1}\phi_{i^s,t}^s \\
    &=\sum_{T_{\alpha}, T_{\text{pred}}: T_{\alpha} \le T_{\text{pred}}} p(T_{\alpha}, T_{\text{pred}})\frac{1}{T_{\alpha}}\sum_{t=1}^{T_{\alpha}}\phi_{i^s,t}^n + \sum_{T_{\alpha}, T_{\text{pred}}: T_{\alpha} > T_{\text{pred}}} p(T_{\alpha}, T_{\text{pred}})\frac{1}{T_{\text{pred}}+1}\left[\phi_{i^s,T_{\text{pred}}+1}^s+\sum_{t=1}^{T_{\text{pred}}}\phi_{i^s,t}^n\right] \\
    &\le \mathbb{E}_{T_{\alpha}, T_{\text{pred}}}\left[\frac{1}{T_{\text{next}}}\sum_{t=1}^{T_{\text{next}}}\phi_{i^s,t}^n\right]\\
    &= \mathbb{E}_{T_{\alpha}, T_{\text{pred}}}\left[\psi_{i^s, T_{\text{next}}}^n\right]
    \end{aligned}
\end{equation*}
where the inequality holds due to $\phi_{i^s,T_{\text{pred}}+1}^n > \phi_{i^s,T_{\text{pred}}+1}^s = 0$ and the equality only holds when $\sum_{T_{\alpha}, T_{\text{pred}}: T_{\alpha} > T_{\text{pred}}} p(T_{\alpha}, T_{\text{pred}}) =0$. It is impractical for the predictive model to guarantee this condition to hold. From the two inequalities above, $\mathbb{E}[\rho_{i^s, T_{\text{next}}}^s] \le \mathbb{E}[\rho_{i^s,T_{\text{next}}}^n]$ and $\mathbb{E}[C_{i^s, T_{\text{next}}}^s] \ge \mathbb{E}[C_{i^s, T_{\text{next}}}^n]$ which follow similar derivations. Equality holds iff $\sum_{T_{\alpha}, T_{\text{pred}}: T_{\alpha} > T_{\text{pred}}} p(T_{\alpha}, T_{\text{pred}}) =0$.
\end{proof}

\end{proposition}

\cref{prop:optimal-strategy} states that, for a node that can give positive contribution to the collaboration, the optimal strategy for it to get the best reward, i.e. lowest convergence complexity if the stopping iteration is $T_{\text{next}}$, is to contribute all the time until the training ends. Though we only consider the reward  when $T_{\text{next}}$ is the stopping iteration due to the difficulty of analysis on the following iterations,  we empirically verify the effectiveness of the result on the reward of the ground truth stopping iteration $T_{\alpha}$. 

Next, we empirically verify this by comparing the rewards (i.e., model performance via online loss) for nodes with different behaviors. The experiment setting is exactly the same as that in \cref{sec:foil} except that we choose the node $i^*$ that has maximum contribution estimate under the non-stop contribution setting as the node that tries to predict the stopping iteration of the contribution evaluation and stop contributing after that. 

We consider two different prediction strategies used by $i^*$ on the stopping iteration: (a) Random guess (RANDOM): $P(T_{\text{pred}} = t|T_{\alpha}) = 1/T^*$ for $t=0,\ldots, T^*$. (b) Poisson predication (POISSON): $P(T_{\text{pred}} = t|T_{\alpha}) = {{T_{\alpha}^t\exp({- T_{\alpha}})}}/{{t!}}\ .$ We also include the honest baseline, (c) Non-stop contribution (NONSTOP): contribute all the time until the end of the training. The result for two utility functions is presented in the experiment: (A) negative loss on test set as utility function, i.e., the utility function $U(\mathcal{S})$ in iteration $t$ is computed as the negative loss on a randomly sampled test set $\mathcal{D}_{test}$ using the model $\theta_{\mathcal{S},t}$, i.e., $\theta_{t-1}$ updated only using the data from coalition $\mathcal{S}$, the detail can be found in~\cite{profit-allocation-FL-Song2019}. (B) inner product utility function in \cref{sec:setting}. The negative loss utility function usually satisfies the submodularity assumption in \cref{prop:optimal-strategy} and the inner product utility function may not satisfy the assumption. From \cref{table:pred-stopping-iteration-neg-loss,table:pred-stopping-iteration-inner-prod}, the node $i^*$ can achieve the best online loss when it contributes all the time in both cases.
\begin{table}[!ht]
\setlength{\tabcolsep}{4pt}
\caption{Online loss (standard error for $5$ runs, not applicable for for NONSTOP, since we fix the randomness including using same weight initialization for the model $\theta_0$ and the random seed for mini-batch and node selection during training) for the node $i^*$ under different types of prediction and stopping patterns in the \textbf{negative loss} utility function setting.
}
\label{table:pred-stopping-iteration-neg-loss}
\begin{center}
\begin{tabular}{c|cc|cc}
 \multicolumn{1}{c}{}& \multicolumn{2}{c}{Label Noise}&\multicolumn{2}{c}{Feature Noise}\\
 \toprule 

 & MNIST & \multicolumn{1}{c|}{CIFAR-10} & MNIST & CIFAR-10 \\
 \midrule 
RANDOM&0.05429(5.2e-04)&0.02685(3.6e-05)&0.05477(4.8e-04)&0.02684(3.6e-05)\\
POISSON&0.05660(1.1e-03)&0.02684(4.3e-05)&0.05514(8.8e-04)&0.02681(4.4e-05)\\
\midrule 
NONSTOP&\multicolumn{1}{l}{\textbf{0.05247 (N.A.)}}&\multicolumn{1}{l}{\textbf{0.02674 (N.A.)}}&\multicolumn{1}{|l}{\textbf{0.05349 (N.A.)}}&\multicolumn{1}{l}{\textbf{0.02671 (N.A.)}}\\
\bottomrule 
\end{tabular}
\end{center}
\end{table}

\begin{table}[!ht]
\setlength{\tabcolsep}{4pt}
\caption{Online loss (standard error for $5$ runs, not applicable for NONSTOP, since we fix the randomness including using same weight initialization for the model $\theta_0$ and the random seed for mini-batch and node selection during training) for the node $i^*$ under different types of prediction and stopping patterns in the \textbf{inner product} utility function setting.
}
\label{table:pred-stopping-iteration-inner-prod}
\begin{center}
\begin{tabular}{c|cc|cc}
 \multicolumn{1}{c}{}& \multicolumn{2}{c}{Label Noise}&\multicolumn{2}{c}{Feature Noise}\\
 \toprule 

 & MNIST & \multicolumn{1}{c|}{CIFAR-10} & MNIST & CIFAR-10 \\
 \midrule 
RANDOM&0.05524(1.1e-04)&0.02696(4.1e-06)&0.05979(8.7e-05)&0.02743(7.1e-07)\\
POISSON&0.05554(1.1e-04)&0.02697(3.6e-06)&0.06003(1.2e-05)&0.02743(1.6e-07)\\
\midrule 
NONSTOP&\multicolumn{1}{l}{\textbf{0.05450 (N.A.)}}&\multicolumn{1}{l}{\textbf{0.02693 (N.A.)}}&\multicolumn{1}{|l}{\textbf{0.05898 (N.A.)}}&\multicolumn{1}{l}{\textbf{0.02743 (N.A.)}}\\
\bottomrule 
\end{tabular}
\end{center}
\end{table}

\subsection{Proof of \cref{prop:equal-convergence}} 

\begin{proof}[Proof of \cref{prop:equal-convergence}]
First, by definition of $C_i$, $\Gamma_{i_*} = \mathcal{O}(1/\epsilon) \implies C_{i_*} = \mathcal{O}(1/\epsilon).$
Next, by \cref{proposition:fairness}, we know that $\forall i \in[N]\ \ \Gamma_i \leq \Gamma_{i_*}$. So,
$\forall i\in [N]\ \ C_i \leq C_{i_*} = \mathcal{O}(1/\epsilon)$.
\end{proof}

\paragraph{Selection of the equalizing coefficient $\beta$.}
We write $\psi_{i} = \psi_{i, T_{\alpha}}$ to omit the notation of $T_{\alpha}$ for brevity.
\begin{lemma}[\bf Finding the Range for $\beta$]\label{lemma:finding-beta}
Suppose $\forall i, M_1 \leq  \psi_i \leq  M_2$ and we want to select a $\beta$ so that for some $r_2 \geq r_1 > 0$, $r_1 \leq \psi_i/(1/\Gamma_i ) \leq r_2$. Then a suitable range for $\beta$ to satisfy this can be found efficiently using the bisection method.
\end{lemma}

\begin{proof}[Proof of \cref{lemma:finding-beta}]

Using the monotonicity of \cref{equ:sampling-probs} we have 
\[ \frac{\exp(\psi_i / \beta)}{N \exp(M_2/ \beta) }  \leq \varrho_i \leq \frac{\exp(\psi_i / \beta)}{N \exp(M_1/ \beta) } \]
which can be simplified to 
\[ \frac{1}{N}\exp\left( \frac{\psi_i - M_2}{\beta} \right)  \leq \varrho_i \leq \frac{1}{N}\exp\left( \frac{\psi_i - M_1}{\beta} \right) \]
that leads to 
\[  \left(1 - \frac{1}{N}\exp\left( \frac{\psi_i - M_1}{\beta} \right)\right)^k \leq (1-\varrho_i)^k \leq  \left(1 - \frac{1}{N}\exp\left( \frac{\psi_i - M_2}{\beta} \right)\right)^k  .\]
Substituting this inequality into 
\[ \Gamma_i \triangleq \frac{ (1-\varrho_i)^k }{ [1- ((1-\varrho_i)^k)]^2 } \]
gives
\[ \displaystyle \frac{ \left(1 - \displaystyle\frac{1}{N}\exp\left( \displaystyle\frac{\psi_i - M_1}{\beta} \right)\right)^k}{ \left[1 - \left(1 - \displaystyle\frac{1}{N}\exp\left( \displaystyle\frac{\psi_i - M_1}{\beta} \right)\right)^k\right]^2} \quad \leq \quad \Gamma_i \quad \leq \quad\frac{ \left(1 - \displaystyle\frac{1}{N}\exp\left( \displaystyle\frac{\psi_i - M_2}{\beta} \right)\right)^k  }{ \left[1 - \left(1 - \displaystyle\frac{1}{N}\exp\left( \displaystyle\frac{\psi_i - M_2}{\beta} \right)\right)^k\right]^2 } \]
because $\Gamma_i$ is monotonically increasing w.r.t.~$(1-\varrho_i)^k$ for  $ 0 \leq (1-\varrho_i)^k \leq 1$ where we have used the fact that $\varrho_i$ is a probability.

From this expression, we can derive the following inequalities on $\Gamma_i / \psi_i$ by using $M_1 \leq \psi_i \leq M_2$ as follows:
\[  M_1 \times  \frac{ \left(1 - \displaystyle\frac{1}{N}\exp\left( \displaystyle\frac{\psi_i - M_1}{\beta} \right)\right)^k}{ \left[1 - \left(1 - \displaystyle\frac{1}{N}\exp\left( \displaystyle\frac{\psi_i - M_1}{\beta} \right)\right)^k\right]^2}  \leq   
\Gamma_i\psi_i  
\leq M_2 \times \frac{ \left(1 - \displaystyle\frac{1}{N}\exp\left( \displaystyle\frac{\psi_i - M_2}{\beta} \right)\right)^k  }{ \left[1 - \left(1 - \displaystyle\frac{1}{N}\exp\left( \displaystyle\frac{\psi_i - M_2}{\beta} \right)\right)^k\right]^2 }\ .\]

Observe the expression
\[  \frac{ \left(1 - \displaystyle\frac{1}{N}\exp\left( \displaystyle\frac{\psi_i - M_1}{\beta} \right)\right)^k}{ \left[1 - \left(1 - \displaystyle\frac{1}{N}\exp\left( \displaystyle\frac{\psi_i - M_1}{\beta} \right)\right)^k\right]^2}  \] is monotonically decreasing w.r.t.~$\psi_i$, so using $M_1 \leq \psi_i \leq M_2$ again gives
\[  M_1 \times  \frac{ \left(1 - \displaystyle\frac{1}{N}\exp\left( \displaystyle\frac{M_2 - M_1}{\beta} \right)\right)^k}{ \left[1 - \left(1 - \displaystyle\frac{1}{N}\exp\left( \displaystyle\frac{M_2 - M_1}{\beta} \right)\right)^k\right]^2}  \leq   
M_1 \times  \frac{ \left(1 - \displaystyle\frac{1}{N}\exp\left( \displaystyle\frac{\psi_i - M_1}{\beta} \right)\right)^k}{ \left[1 - \left(1 - \displaystyle\frac{1}{N}\exp\left( \displaystyle\frac{\psi_i - M_1}{\beta} \right)\right)^k\right]^2}  \leq   
\frac{\Gamma_i}{\psi_i} \]
on the LHS of $\Gamma_i/\psi_i$ and 
\[
\frac{\Gamma_i}{\psi_i}
\leq M_2 \times \frac{ \left(1 - \displaystyle\frac{1}{N}\exp\left( \displaystyle\frac{\psi_i - M_2}{\beta} \right)\right)^k  }{ \left[1 - \left(1 - \displaystyle\frac{1}{N}\exp\left( \displaystyle\frac{\psi_i - M_2}{\beta} \right)\right)^k\right]^2 }
\leq  M_2 \times 
\frac{ \left(1 - \displaystyle\frac{1}{N}\exp\left( \displaystyle\frac{M_1 - M_2}{\beta} \right)\right)^k  }{ \left[1 - \left(1 - \displaystyle\frac{1}{N}\exp\left( \displaystyle\frac{M_1 - M_2}{\beta} \right)\right)^k\right]^2 }
\]
on the RHS of $\Gamma_i/\psi_i$.

Finally, set 
\[ M_1 \times  \frac{ \left(1 - \displaystyle\frac{1}{N}\exp\left( \displaystyle\frac{M_2 - M_1}{\beta} \right)\right)^k}{ \left[1 - \left(1 - \displaystyle\frac{1}{N}\exp\left( \displaystyle\frac{M_2 - M_1}{\beta} \right)\right)^k\right]^2} = r_1 \] 
and 
\[ M_2 \times 
\frac{ \left(1 - \displaystyle\frac{1}{N}\exp\left( \displaystyle\frac{M_1 - M_2}{\beta} \right)\right)^k  }{ \left[1 - \left(1 - \displaystyle\frac{1}{N}\exp\left( \displaystyle\frac{M_1 - M_2}{\beta} \right)\right)^k\right]^2 } = r_2 \] 
to solve for the range of $\beta$ using the bisection method since both expressions are monotonic in $\beta$.

\end{proof}

\paragraph{Difficulties and factors of preserving equality.}
While we restrict $\beta$ to be finite, the following limit illuminates the difficulties to preserve equality (in terms of how many iterations are required) due to $N$ and $k$: $\lim_{\beta\to \infty} \Gamma_i = (1 - 1/N)^k / [1 - (1- 1/N)^k]^2$ for all $i\in[N]$. To satisfy the sufficient condition in \cref{prop:equal-convergence}, we require $\epsilon \leq (1- 1/N)^{-k} + (1-1/N)^k - 2$ (ignoring the constant terms in $\mathcal{O}(1/\epsilon)$). We highlight that a smaller upper-bound on $\epsilon$ translates to more training iterations, which means it is harder to preserve equality since we need longer for the nodes with lower contributions to `catch up'. 
As $N$ represents a form of population (the number of nodes) and $k$ effectively represents a resource bottleneck (the number of nodes to synchronize in each $t$), if $N$ increases/$k$ decreases, the upper-bound on $\epsilon$ decreases (i.e., harder to preserve equality). 

\textbf{Finding a $\beta$ for a given $\Gamma^{*}_{i_{*}}$.} Recall from \cref{prop:equal-convergence}, $i_{*} \coloneqq \argmin_{i}\psi_{i,T_{\alpha}}$, the probability of selecting $i_{*}$ is $\varrho_{i_{*}} = \exp(\psi_{i_{*},T_{\alpha}}/\beta)/\sum_{i \in N}\exp(\psi_{i,T_{\alpha}}/\beta)$. Since ${\partial \varrho_{i_{*}}}/{\partial \beta} > 0$, $\varrho_{i_{*}}$ increases as $\beta$ increases from $0$ to $\infty$. Therefore,  $\Gamma_{i_{*}}$ decreases as $\beta$ increases. Since $\lim_{\beta\to \infty} \Gamma_{i_{*}} = (1 - 1/N)^k / [1 - (1- 1/N)^k]^2$, for any given $\Gamma^{*}_{i_{*}} > (1 - 1/N)^k / [1 - (1- 1/N)^k]^2$, we can find a $\beta$ (by using binary search, etc.) s.t.~$\Gamma_{i_{*}} = \Gamma^{*}_{i_{*}}$. This result states that for any arbitrarily small $\psi_{i_{*},T_{\alpha}}$, as long as a given $\Gamma^{*}_{i_{*}} < (1 - 1/N)^k / [1 - (1- 1/N)^k]^2$, we can always find a $\beta$ to make the worst expected staleness of nodes be equal to $\Gamma^{*}_{i_{*}}$. That is, we can always find $\beta$ to avoid the nodes having arbitrarily bad worst expected staleness.


\end{document}